\documentclass{ecai} 
\usepackage{microtype}
\usepackage{graphicx}
\usepackage{multirow}
\usepackage{balance}
\usepackage{booktabs} % fancy tables
\usepackage{tikz} % for drawing stuff
\usepackage{xcolor}
\usepackage{readarray} % for \getargsC{}
\usepackage{float}
\usepackage{dblfloatfix}
\usepackage{paralist}
\usepackage{amsmath, amssymb, amsthm}
\usepackage{listings}
\usepackage{url}
\definecolor{codegreen}{rgb}{0,0.6,0}
\definecolor{codegray}{rgb}{0.5,0.5,0.5}
\definecolor{codepurple}{rgb}{0.58,0,0.82}
\definecolor{backcolour}{rgb}{0.95,0.95,0.92}

\definecolor{darkred}{rgb}{0.64,0,0}
\definecolor{darkcyan}{rgb}{0,0.55,0.55}
\newcommand{\rowcolor}[1]{\textcolor{darkred}{#1}}
\newcommand{\columncolor}[1]{\textcolor{darkcyan}{#1}}

\newcommand{\shieldname}[1]{$\mathcal{T}_{#1}$}

\lstdefinestyle{mystyle}{
    language=Prolog,
    backgroundcolor=\color{backcolour},   
    commentstyle=\color{codegreen},
    keywordstyle=\color{magenta},
    numberstyle=\tiny\color{codegray},
    stringstyle=\color{codepurple},
    basicstyle=\ttfamily\footnotesize,
    breakatwhitespace=false,         
    breaklines=true,                 
    captionpos=b,                    
    keepspaces=true,                 
    numbers=left,                    
    numbersep=5pt,                  
    showspaces=false,                
    showstringspaces=false,
    showtabs=false,                  
    tabsize=2
}
\newcommand{\smallnfgame}[1]{%
\getargsC{#1}
\begin{tikzpicture}[scale=0.35]

\node (RT) at (-2,1) [label=left:\rowcolor{\argi}] {};
\node (RB) at (-2,-1) [label=left:\rowcolor{\argii}] {};
\node (CL) at (-1,2) [label=above:\columncolor{\argiii}] {};
\node (CR) at (1,2) [label=above:\columncolor{\argiv}] {};

\node (RTL) at (-1.4,0.6) {\rowcolor{\Large\argv}}; % top/left row player payoff etc.
\node (RTR) at (0.6,0.6) {\rowcolor{\Large\argvi}};
\node (RBL) at (-1.4,-1.4) {\rowcolor{\Large\argvii}};
\node (RBR) at (0.6,-1.4) {\rowcolor{\Large\argviii}};

\node (CTL) at (-0.6,1.4) {\columncolor{\Large\argix}};
\node (CTR) at (1.4,1.4) {\columncolor{\Large\argx}};
\node (CBL) at (-0.6,-0.6) {\columncolor{\Large\argxi}};
\node (CBR) at (1.4,-0.6) {\columncolor{\Large\argxii}};

\draw[-,very thick] (-2,-2) to (2,-2);
\draw[-,very thick] (-2,0) to (2,0);
\draw[-,very thick] (-2,2) to (2,2);
\draw[-,very thick] (-2,-2) to (-2,2);
\draw[-,very thick] (0,-2) to (0,2);
\draw[-,very thick] (2,-2) to (2,2);
\draw[-,very thin] (-2,2) to (0,0);
\draw[-,very thin] (0,0) to (2,-2);
\draw[-,very thin] (-2,0) to (0,-2);
\draw[-,very thin] (0,2) to (2,0);
\end{tikzpicture}
}

\theoremstyle{plain}
\newtheorem{theorem}{Theorem}[section]
\newtheorem{proposition}[theorem]{Proposition}
\newtheorem{lemma}[theorem]{Lemma}

\theoremstyle{definition}
\newtheorem{definition}[theorem]{Definition}

\theoremstyle{remark}
\newtheorem{remark}[theorem]{Remark}

\newcommand{\BibTeX}{B\kern-.05em{\sc i\kern-.025em b}\kern-.08em\TeX}

\begin{document}

%%%%%%%%%%%%%%%%%%%%%%%%%%%%%%%%%%%%%%%%%%%%%%%%%%%%%%%%%%%%%%%%%%%%%%%%

\begin{frontmatter}

%%% Use this command to specify your submission number.
%%% In doubleblind mode, it will be printed on the first page.

\paperid{4833} 

%%% Use this command to specify the title of your paper.

% \title{Safe, Explainable Multi-Agent Reinforcement Learning with Probabilistic Logic Shields}

\title{Analyzing Probabilistic Logic Shields for \\Multi-Agent Reinforcement Learning}

%%% Use this combinations of commands to specify all authors of your 
%%% paper. Use \fnms{} and \snm{} to indicate everyone's first names 
%%% and surname. This will help the publisher with indexing the 
%%% proceedings. Please use a reasonable approximation in case your 
%%% name does not neatly split into "first names" and "surname".
%%% Specifying your ORCID digital identifier is optional. 
%%% Use the \thanks{} command to indicate one or more corresponding 
%%% authors and their email address(es). If so desired, you can specify
%%% author contributions using the \footnote{} command.

% \author[A]{\fnms{Satchit}~\snm{Chatterji}\orcid{0009-0003-8648-1158}\thanks{Corresponding Author. Email: s.chatterji@uva.nl}}
% \author[A]{\fnms{Erman}~\snm{Acar}\orcid{0009-0002-9656-7249}}
\author[A]{\fnms{Satchit}~\snm{Chatterji}\thanks{Corresponding Author. Email: s.chatterji@uva.nl}}
\author[A]{\fnms{Erman}~\snm{Acar}}

% \address[A]{University of Amsterdam}
\address[A]{IvI \& ILLC, University of Amsterdam}

%%% Use this environment to include an abstract of your paper.

\begin{abstract}

% version 1
% Safe reinforcement learning (RL) is crucial for real-world applications, and multi-agent interactions introduce additional safety challenges. We propose \textbf{Shielded Multi-Agent Reinforcement Learning (SMARL)}, which extends single-agent Probabilistic Logic Shields (PLS) to decentralized, multi-agent settings. Our approach introduces Probabilistic Logic Temporal Difference Learning (PLTD) for shielded independent Q-learning and a probabilistic logic policy gradient-based shielded independent PPO. Results show improved safety, cooperation, and alignment with normative behaviors across multiple $n$-player game-theoretic environments, including asymmetric settings where a strict subset of agents is shielded. Thus, this method is shown to be effective as a robust equilibrium selection mechanism. These findings highlight SMARL’s potential to improve safety and cooperation in diverse multi-agent settings.

% version 2
Safe reinforcement learning (RL) is crucial for real-world applications, and multi-agent interactions introduce additional safety challenges. While Probabilistic Logic Shields (PLS) has been a powerful proposal to enforce safety in single-agent RL, their generalizability to multi-agent settings remains unexplored. In this paper, we address this gap by conducting extensive analyses of PLS within decentralized, multi-agent environments, and in doing so, propose \textbf{Shielded Multi-Agent Reinforcement Learning (SMARL)} as a general framework for steering MARL towards norm-compliant outcomes. Our key contributions are: (1) a novel Probabilistic Logic Temporal Difference (PLTD) update for shielded, independent Q-learning, which incorporates probabilistic constraints directly into the value update process; (2) a probabilistic logic policy gradient method for shielded PPO with formal safety guarantees for MARL; and (3) comprehensive evaluation across symmetric and asymmetrically shielded $n$-player game-theoretic benchmarks, demonstrating fewer constraint violations and significantly better cooperation under normative constraints. These results position SMARL as an effective mechanism for equilibrium selection, paving the way toward safer, socially aligned multi-agent systems.

%version 3

% Safe reinforcement learning (RL) is crucial for real-world applications, and multi-agent interactions introduce additional safety challenges. While Probabilistic Logic Shields (PLS) has been a powerful proposal to enforce safety in single-agent RL, their generalization to multi-agent settings remains unexplored. In order to fill this gap, in this article, we carry out an extensive analysis of 
% of PLS in multiagent settings across different  game-theoretic scenarios and environments. In doing so, we
% introduce \textbf{Shielded Multi-Agent Reinforcement Learning (SMARL)}, a framework that extends PLS to decentralized, multi-agent setting. In particular,  (1) we introduce a novel Probabilistic Logic Temporal Difference (PLTD) update for shielded, independent Q-learning; (2) we propose a probabilistic logic policy gradient method for shielded PPO with formal safety guarantees for MARL; and (3)  we demonstrate that fewer alignment violations and significantly better cooperation under normative constraints. These results suggest that PLS and  SMARL in general can serve as an effective mechanism for robust equilibrium selection, paving the way toward safer, socially aligned multi-agent systems.

\end{abstract}

\end{frontmatter}

%%%%%%%%%%%%%%%%%%%%%%%%%%%%%%%%%%%%%%%%%%%%%%%%%%%%%%%%%%%%%%%%%%%%%%%%

% \begin{table}[h]
% \caption{Locations of selected conference editions.}
% \centering
% \begin{tabular}{ll@{\hspace{8mm}}ll} 
% \toprule
% AISB-1980 & Amsterdam & ECAI-1990 & Stockholm \\
% ECAI-2000 & Berlin & ECAI-2010 & Lisbon \\
% ECAI-2020 & \multicolumn{3}{l}{Santiago de Compostela (online)} \\
% \bottomrule
% \end{tabular}
% \end{table}

%%%%%%%%%%%%%%%%%%%%%%%%%%%%%%%%%%%%%%%%%%%%%%%%%%%%%%%%%%%%%%%%%%%%%%%%

\section{Introduction}

Recent years have witnessed significant progress in multi-agent reinforcement learning (MARL), with sophisticated algorithms tackling increasingly complex problems in various domains including autonomous vehicles~\citep{shalev2016safe}, distributed robotics~\citep{brambilla2013swarm}, algorithmic trading~\citep{ganesh2019reinforcement}, energy grid management~\citep{van2023multi} and healthcare~\citep{riek2017healthcare}. The ultimate success of RL in this diverse collection of domains and the deployment in the real-world, however, demands overcoming a difficult key challenge: \textit{safety}.

This has naturally resulted in the research direction of \textit{Safe RL} which aims to learn optimal policies that are, by some measure, `safe'. \citet{gu2024review} present an up-to-date overview of the field. A number of proposals focus on the application of formal methods \citep{hunt2021verifiably,den2022planning}, within which the notion of \textit{shielding} is used, a technique that is inspired by formal verification through temporal logic specifications to avoid unsafe actions during the agent's learning process~\citep{bloem2015shield, alshiekh2018safe, jansen2020safe, carr2023safe}. 

One recent proposal that uses shielding to represent safety constraints is \textit{probabilistic logic shields} \citep[PLS,][]{yang2023safe} whose semantics are based on probabilistic logic (PL) programming \citep{de2007problog}. PLS constrains an agent's policy to comply with formal specifications \textit{probabilistically}. The specification of constraints, (the \textit{shield}), is defined within the exploration and learning pipeline. 
% 22/08/25 PLS comes with various advantages over other techniques: ($i$) the shields are applied at the \textit{policy level} instead of individual actions capturing the probabilistic evaluation of safety (in expectation) in contrast to deterministic shields that are based on hard-rejection \cite{hunt2021verifiably,jansen2020safe}; ($ii$) PLS is less demanding in terms of requiring knowledge of the underlying MDP compared to previous methods \citep[e.g.,][]{hunt2021verifiably, jansen2020safe, carr2023safe}; ($iii$) being end-to-end differentiable, it integrates logical semantics with deep learning architectures seamlessly; and ($iv$) it provides safety guarantees within single-agent settings.
PLS offers policy-level evaluation of safety in contrast to action-level hard-rejection shields \citep[e.g.,][]{hunt2021verifiably, jansen2020safe, carr2023safe}, differentiability, modest requirements on knowing the MDP, and safety guarantees within single-agent RL.

\begin{figure}[t]
    \centering
    \includegraphics[width=0.8\linewidth]{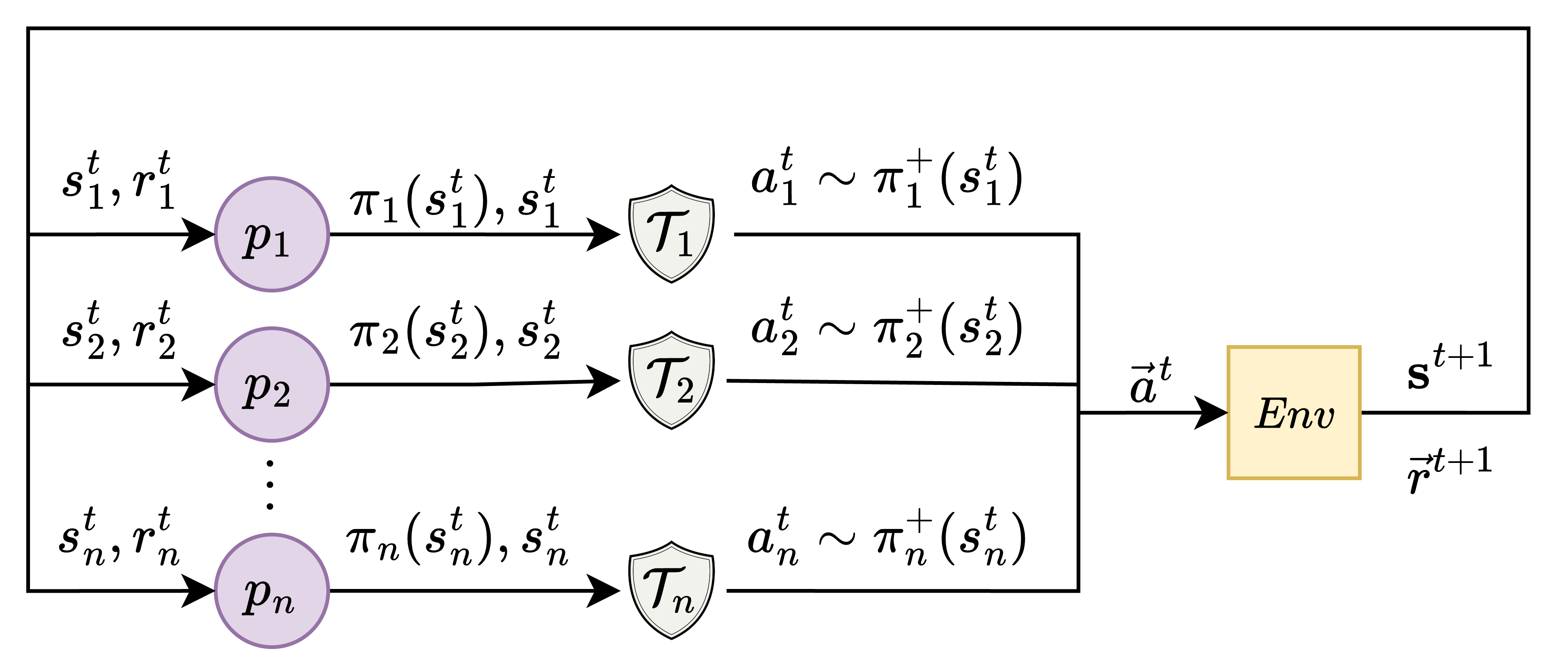}
    \caption{Shielded MARL (SMARL) interaction schematic. At time $t$, each agent $p_i$ passes their respective policy $\pi_i^t(s_i^t)$ and a safety-specific state $s^t_i$ to their shield $\mathcal{T}_i$, resulting in a safe policy $\pi_i^+(s_i^t)$ from which an action $a^t_i$ is sampled and returned to the environment.\vspace{-10pt}}
    \label{fig:smarl_parallel}
\end{figure}

However, safety is inherently a \textit{multi-agent} concept, as real-world environments often involve multiple agents interacting simultaneously, leading to hard-to-control complex systems. While there exist a few approaches tackling safety in multi-agent settings \citep{gu2023safe,elsayed2021safe,melcer2022shield, subramanian2024neuro}, to the best of our knowledge, PLS has not been adopted, extended or analyzed in the context of MARL. In this paper, we address this gap with the following contributions:
\begin{compactenum}
      \item We present a framework that extends shielding in RL to multi-agent settings, named \textit{Shielded MARL} (SMARL). We provide a theoretical guarantee that SMARL with PLS produces safer joint policies than unshielded counterparts. Using decentralized techniques, we introduce (a) \textit{Probabilistic Logic Temporal Difference Learning} (PLTD), with convergence guarantees, allowing for shielded independent Q-learning (SIQL), and (b) shielded independent PPO (SIPPO) using probabilistic logic policy gradients;
      \item We show that PLS can be used as an equilibrium selection mechanism, a key challenge in MARL, under various game-theoretic settings, such as coordination, spatio-temporal coordination, social dilemmas, cooperation under uncertainty, and mixed-motive problems. We provide strong empirical evidence across various $n$-player environments including an extensive-form game (Centipede), a stochastic game (Extended Public Goods Game), a simultaneous game (Stag-Hunt),  and its grid-world extension (Markov Stag-Hunt). Moreover, we investigate the impact of smaller (weak, less specific) and larger (strong, more specific) shields;
      % \item We provide a safety guarantee that SMARL produces safer joint policies than their unshielded counterparts;
      \item We investigate asymmetric shielding, i.e.\ the ability of shielded agents to influence unshielded peers. Results show that partial shielding can significantly enhance safety, highlighting SMARL’s effectiveness in both cooperative and non-cooperative settings.
\end{compactenum}

%\begin{compactenum}
%      \item We introduce \textit{Shielded MARL} (SMARL) by extending PLS to MARL -- in particular, we introduce \textit{Probabilistic Logic Temporal Difference Learning} (PLTD) to enable shielded independent Q-learning (SIQL), with a proof of convergence in limited settings, and introduce shielded independent PPO (SIPPO) using probabilistic logic policy gradients;
 %     \item We show that PLS can be used as a equilibrium selection mechanism in various $n$-player game-theoretic environments including a simultaneous game (Stag-Hunt), an extensive-form game (Centipede), a stochastic game (Extended Public Goods Game), and a grid-world (Markov Stag-Hunt). Moreover, we investigate the impact of smaller (weak) and larger (strong) shields;
 %     \item We provide a safety guarantee that SMARL produces safer joint policies than their unshielded counterparts;
%      \item We investigate the influence of shielded agents over unshielded ones. Results suggest that the agent equipped with the shield can guide the other agent's behavior significantly, highlighting SMARL's ability to enhance safety and cooperation.
%\end{compactenum}

% We believe a dedicated analysis of SMARL in two-player games is warranted, as these environments serve as fundamental, tractable models for understanding larger multi-agent populations -- an approach widely adopted in game theory and MARL research. In many multi-agent systems ($n\geq2$), interactions between opposing or independent agent groups often resemble two-player dynamics, making our scenarios and analyses broadly applicable.

\begin{figure*}[]
    \centering
    \begin{minipage}[b]{0.24\textwidth}
        \centering
        \includegraphics[width=0.8\linewidth]{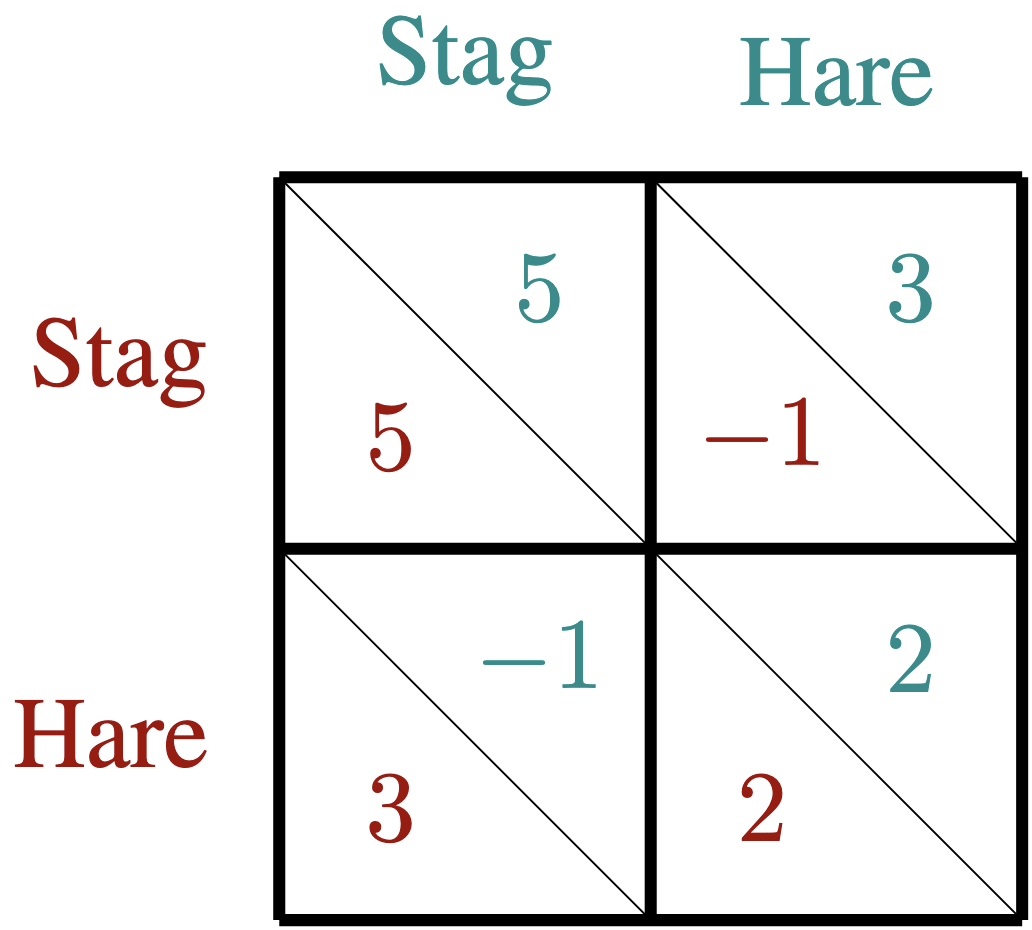}
        \vspace{0.3em}
        
        \textbf{(a)} Payoff matrix for \textit{Stag-Hunt}
        \label{fig:nfg_generic}
    \end{minipage}
    \hspace{0.1cm}
    \begin{minipage}[b]{0.24\textwidth}
        \centering
        \includegraphics[width=0.75\linewidth]{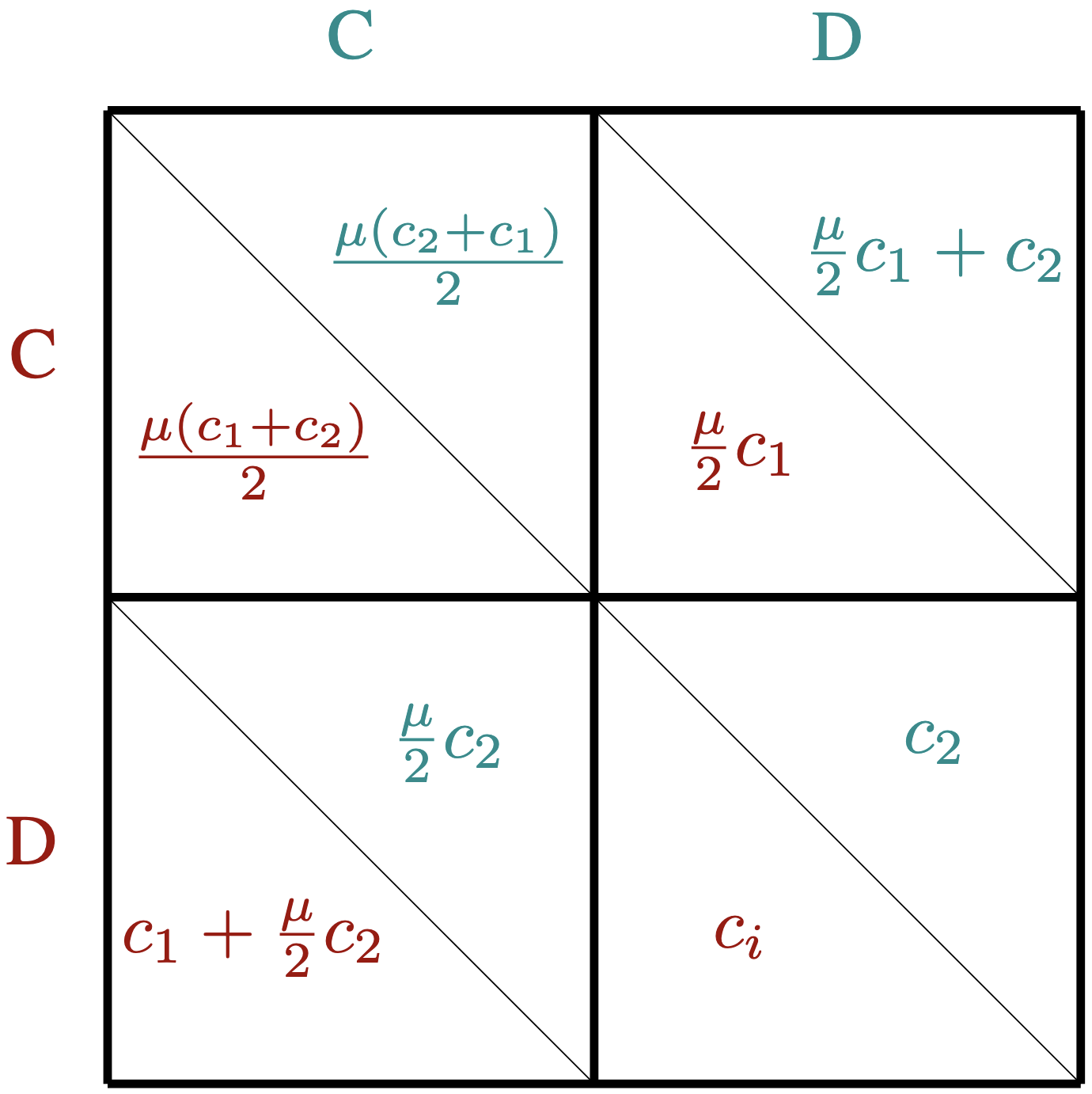}
        \vspace{0.3em}
        
        \textbf{(b)} Expected EPGG utilities
        \label{fig:nfg_epgg}
    \end{minipage}
    \hspace{0.1cm}
    \begin{minipage}[b]{0.24\textwidth}
        \centering
        \includegraphics[width=0.8\linewidth]{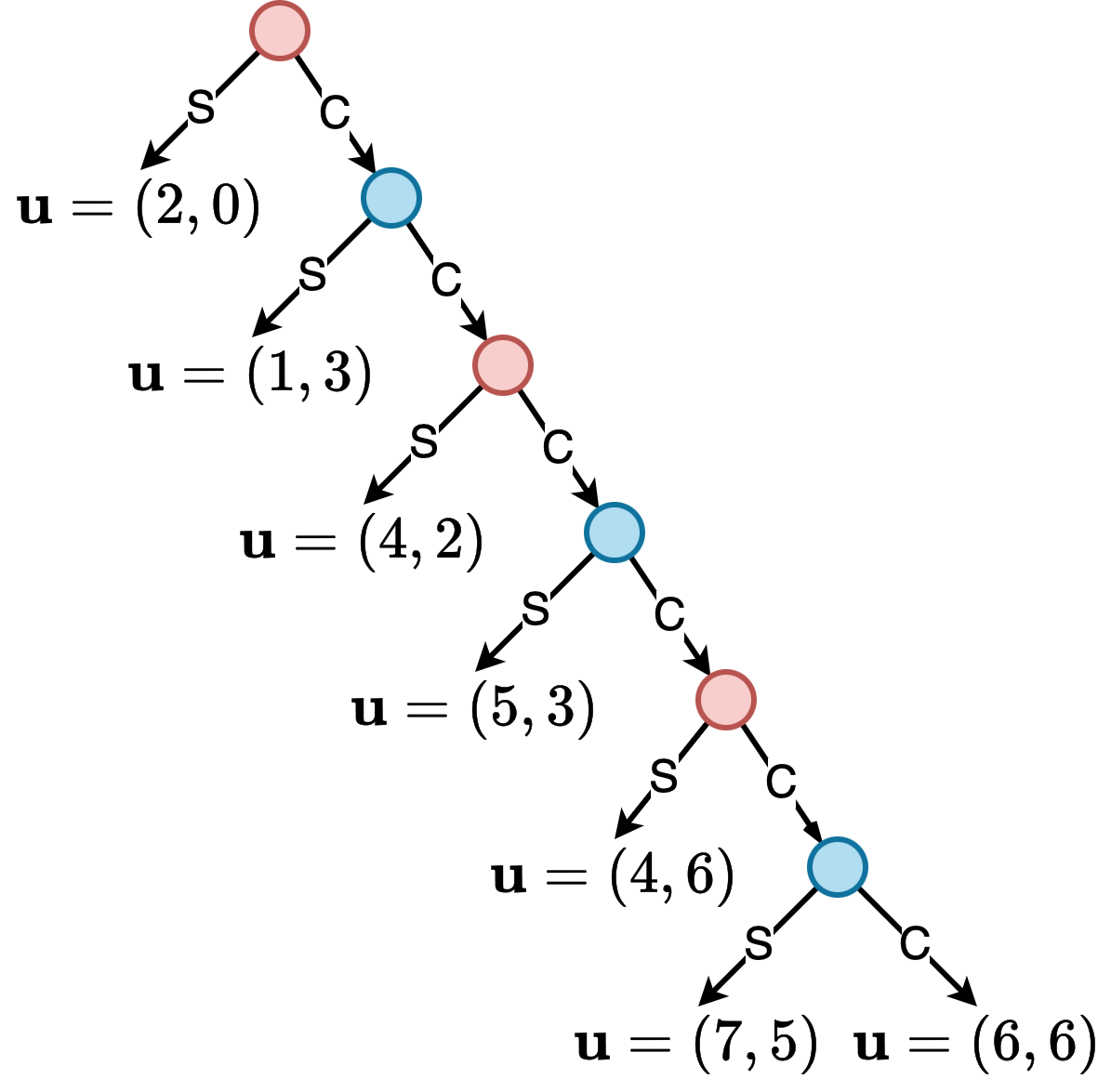}
        \vspace{0.3em}
        
        \textbf{(c)} \textit{Centipede} instance
        \label{fig:centipede_instance}
    \end{minipage}
    \hspace{0.1cm}
    \begin{minipage}[b]{0.23\textwidth}
        \centering
        \includegraphics[width=0.735\linewidth]{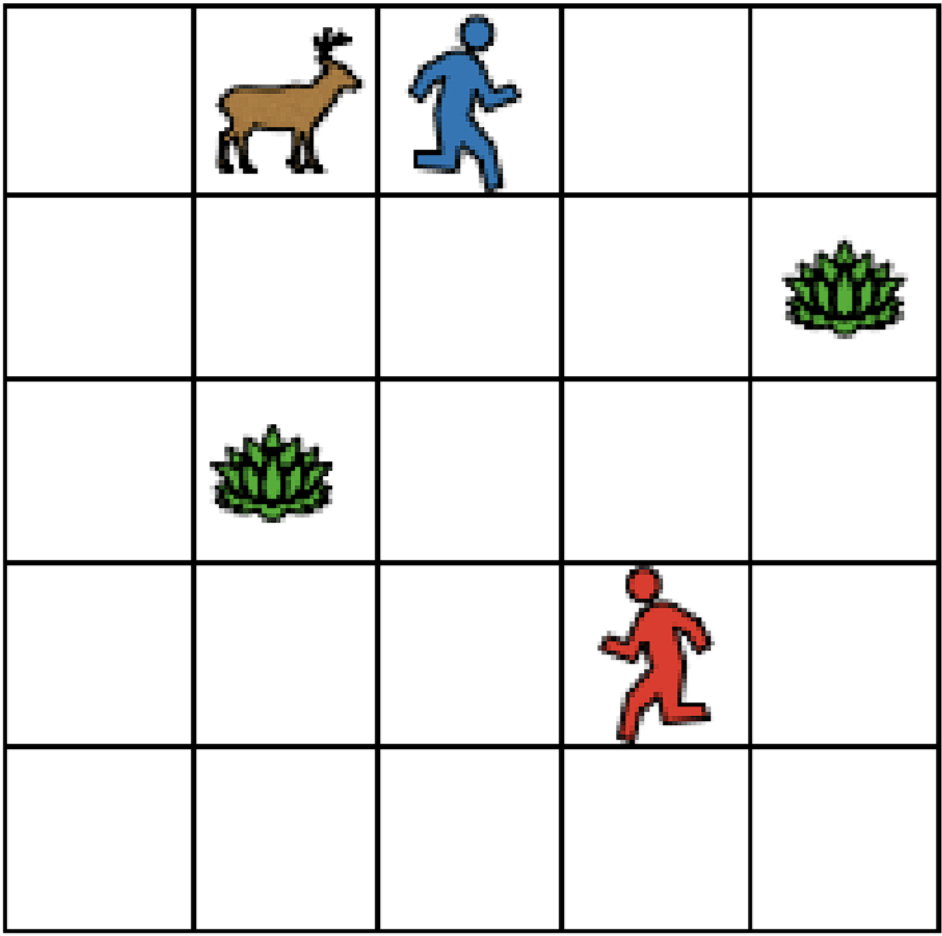}
        \vspace{0.3em}
        
        \textbf{(d)} \textit{Markov Stag-Hunt} \citep{peysakhovich2018prosocial}
        \label{fig:msh_env}
    \end{minipage}
    
    \caption{Representations of the games that are experimented within this paper. For details please refer to Section~\ref{sec:prelims:gt}.}
    \label{fig:nfg_utilities}
\end{figure*}

%%%%%%%%%%%%%%%%%%%%%%%%%%%%%%%%%%%%%%%%%%%%%%%%%%%%%%%%%%%%%%%%%%%%%%%%
%\subsection{Related Work}

\section{Preliminaries}
% We start by specifying what we mean by \textit{RL safety}, introduce MARL, provide background knowledge on the game-theoretic environments we use, and finally a quick introduction to PLS. 

\subsection{Safety Definition}\label{sec:safety}
We expand the interpretation of `safety' in RL beyond that of works like \citet{elsayed2021safe} and \citet{yang2023safe} (i.e.\ informally, \textit{the agent should not perform actions that may lead to something `harmful'}), and more towards an alternative conception provided in \citet[][pg.\ 3]{gu2024review}, namely, \textit{they act, reason, and generalize obeying human desire}. We formalize this in the vein of \citet{alshiekh2018safe}:

\begin{definition}\label{def:safety}
    \textbf{Formally-Constrained Safe RL} is the process of learning an optimal policy $\pi^*$ while maximally satisfying a set of formal specifications $\mathcal{T}$ during learning and deployment.
\end{definition}

$\mathcal{T}$ may represent any set of constraints, (e.g.\ arbitrary temporal logic formulae). These may not \textit{necessarily} be `harmful', but may instead ascribe normative or desirable behavior, such as social norms or equilibrium conditions.
In our case, $\mathcal{T}$ is a ProbLog program.

\subsection{Multi-Agent Reinforcement Learning} \label{sec:marl}
Multi-Agent Reinforcement Learning (MARL) generalizes RL to environments with multiple interacting agents, making the learning problem non-stationary since each agent’s dynamics depend on the evolving policies of others (\citep[][]{marl-book} provides a detailed introduction). Formally, MARL is often modeled as a \textit{stochastic game} in which each agent $i$ selects actions according to a local policy $\pi_i$, and inducing a joint policy $\vec{\pi} = \langle \pi_1, \dots, \pi_n \rangle$ determining the next state and rewards. A common baseline is \textit{Independent Q-Learning} \citep[IQL,][]{tan1993multi}, where each agent maintains its own Q-function and updates it independently using local observations and rewards. Similarly, \textit{Independent Proximal Policy Optimization} \citep[IPPO,][]{de2020independent} adapts PPO to the multi-agent case by training each one with a separate policy-gradient update. In practice, MARL algorithms may employ \textit{parameter sharing}, where agents use a common neural architecture for policies and/or critics, differing only through their inputs (e.g., observations). This has been shown to improve sample efficiency and facilitates coordination, but may reduce policy diversity and lead to homogenized behaviors \citep{marl-book, gupta2017cooperative}. 
% Together, IQL, IPPO, and their parameter-sharing variants serve as standard baselines against which safety-enhanced methods such as shielded MARL can be compared.

\subsection{Game-Theoretic Environments} \label{sec:prelims:gt}

A {normal-form game} (NFG) is a tuple $\langle N,\mathbf{A},\mathbf{u}\rangle$ where
 $N=\{1,...,n\}$ is a finite set of agents (or players), $\mathbf{A}=\mathcal{A}_1\times\cdots\times \mathcal{A}_n$ is a finite set of \emph{action profiles} $\mathbf{a}=(a_1,...a_n)$ with $\mathcal{A}_i$ being the set of player $i$'s actions, and $\mathbf{u}=(u_1,...,u_n)$ is a profile of utility functions $u_i:\mathbf{A}\rightarrow\mathbb{R}$. A strategy profile $\mathbf{s}=(s_1,...s_n)$ is called a \emph{Nash equilibrium} if for each player $k\in N$, the strategy $s_k$ is a best response to the strategies of all the other players $s_{i\in N\setminus\{k\}}$ \citep{leyton2022essentials, maschler2020game}.%.

% $\blacktriangleright$ \textbf{Prisoner's Dilemma} is a 2 players normal-form game in which each player that can take i.e.,  $\mathcal{A}_1 = \mathcal{A}_2$ = \{C, D\} where  (\textit{C}) and  (\textit{D}) correspond to \textit{cooperate} and \textit{defect}, respectively. 
% Any payoff matrix with the following structure 
% in Figure~\ref{fig:nfg_utilities} (for two players), is considered an instance of Prisoner's Dilemma \citep{leyton2022essentials}: $d=e>a=b>g=h>c=f$ (See Figure~\ref{fig:nfg_utilities} (a)). This payoff matrix has just one Nash equilibrium in which both players defect. 

$\blacktriangleright$ \textbf{Stag-Hunt}\label{sec:foundations:gt:nfg:sh} is a two-player NFG in which each player has the actions $\mathcal{A}_1  = \mathcal{A}_2 = \{Stag, Hare\}$. 
% Any payoff matrix with the following structure of payoffs is considered an instance of the Stag-Hunt: $a=b>d=e\geq g=h> c=f$. 
If they coordinate and both play \textit{Stag}, each will get a large reward. However, if only one plays \textit{Stag}, it gets a penalty (low/negative utility). Alternatively, either agent may unilaterally play \textit{Hare} to receive a small positive utility regardless of the actions of the other agent. An example Stag-Hunt \textit{payoff matrix} (a representation of the agents' utility functions) is seen in Figure~\ref{fig:nfg_generic}a.
%Both agents can also opt to go for the hare, in which case they both will get a small reward regardless of the action of the other player. This game has two pure-strategy Nash equilibria: either both players play $Stag$ (also called the \emph{cooperative equilibrium}), or $Hare$ (the \emph{non-cooperative} equilibrium). Additionally, there is a mixed strategy equilibrium that depends on the exact values in the utility matrix, and can be solved analytically.

%-- for the instance of the game in Figure~\ref{fig:nfg_utilities}, the mixed strategy profile is $((0.6, 0.4),(0.6, 0.4))$ with an expected utility of 2 for each player (see Appendix~\ref{app:stag_hunt_eq} for the computation).
$\blacktriangleright$ \textbf{Extended Public Goods Game} (EPGG) is a mixed-motive simultaneous game that represents the \textit{cooperation vs. competition} social dilemma \citep{orzanacar}. An EPGG is a tuple $\langle N, c, \mathbf{A}, f, u \rangle$, where $N = \{1, \ldots, n\}$ is the set of players, $c_i \in \mathbb{R}_{\geq 0}$ is the amount of coin each player $i$ is endowed with and collected in $c = (c_1, \ldots,c_n)$, and $f \in \mathbb{R}_{\geq 0}$ is the multiplication factor for the lump sum endowment (hence the name `extended', as opposed to the case $f \leq n$). Each player $i\in N$ decides whether to invest in the public good (\textit{cooperate}) or not (\textit{defect}), i.e., $\mathcal{A}_i = \{C, D\}$. The resulting quantity is then evenly distributed among all agents. The utility function for $i$ is defined as $u_i: \mathbf{A} \times \mathbb{R}_{\geq 0} \times \mathbb{R}_{\geq 0}^{n} 
\rightarrow \mathbb{R}$, with:
%\begin{equation}
%\label{eq:utility}
 $ r_i(\mathbf{a}, f, c) = \frac{1}{n} \sum_{j=1}^{n} c_j I(a_j) \cdot f + c_i (1 - I(a_i))$
%\end{equation}
where $\mathbf{a}$ is the action profile, $a_j$ is the $j-$th entry of $\mathbf{a}$, $I(a_j)$ is the indicator function returning 1 if the action of the agent $j$ is cooperative and 0 otherwise, and $c_j$ denotes the $j-$th entry of $c$. Here, EPGG is formulated as a partially observable stochastic game in which $f:=f_t\sim\mathcal{N} ( \mu , \sigma)$, sampled every time step $t$, where $\mathcal{N} ( \mu , \sigma)$ is a normal distribution with mean $\mu$ and variance $\sigma$. Depending on the value of $\mu$, the game is expected to be non-cooperative ($\mu <1$), cooperative ($\mu > n$) or a mix of both (i.e., mixed-motive for $1<\mu<n$). The \textit{expected} Nash equilibrium is determined by $\mu=\mathbb{E}_{f_t\sim\mathcal{N}(\mu,\sigma)}[f_t]$, and can be estimated empirically ($\hat\mu$) by taking the mean of several observations of $f_t$. A generic two-player EPGG payoff matrix is shown in Figure~\ref{fig:nfg_epgg}b.

$\blacktriangleright$ \textbf{Centipede} is a two-player extensive-form game (i.e., the game has several states) in which two agents take turns deciding whether to continue the game (increasing the potential rewards for both) or defect (ending the game and collecting a short-term reward). The game structure incentivizes short-term defection, as each agent risks being defected upon by their partner. Formally, the game begins with a `pot' $p_0$ (a set amount of utility) where each player has two actions, $\mathcal{A}_i=\{Continue,Stop\}$. After each round $t$, the pot increases, e.g., linearly (such as $p_{t+1}\leftarrow p_t+1$) or exponentially (such as $p_{t+1}\leftarrow p_t^2$). If a player plays $Stop$, then it receives $p_t/2+1$ while the other player receives $p_t/2-1$. If both players play $Continue$, they split the pot equally after a certain amount of rounds $t_{max}$, each receiving $p_{t_{max}}/2$. An instance of the game ($t_{max}=6$, $n=2$, $p_0\leftarrow2$, $p_{t+1}=p_t+2$) is given in Figure~\ref{fig:centipede_instance}c.
%An instance of a centipede game with $t_{max}=6$ turns and $n=2$ players is shown in Figure~\ref{fig:nfg_utilities}(c).
%The optimal strategy for the first  player is to stop. This is called the \textit{subgame perfect Nash equilibrium}, one analogue of the Nash equilibrium for extensive-form games (see \cite{maschler2020game} or \cite{leyton2022essentials} for more details). 

$\blacktriangleright$ \textbf{Markov Stag-Hunt} is a grid-world environment inspired by \citet{peysakhovich2018prosocial}, where agents move through a grid and must decide whether to hunt a stag (with risk of penalty if hunting alone) or harvest a plant, with an underlying reward structure similar to the \textit{Stag-Hunt} game. Figure~\ref{fig:msh_env}d visualizes this environment.

\subsection{Probabilistic Logic Shielding}

We summarize \citet{yang2023safe}'s approach known as \textit{Probabilistic Logic Shielding} (PLS), a method for incorporating probabilistic safety in RL, and refer the reader to the original work for a detailed treatment. Unlike earlier rejection-based shields that fully block unsafe actions and limit exploration \citep{alshiekh2018safe}, PLS reduces the probability of unsafe actions proportional to their risk, allowing them to occasionally occur. Under soft constraints, this lets agents learn from unsafe actions.

Assume a model $P$ such that $P(\texttt{safe} | s,a)$ represents the likelihood that taking action $a$ in state $s$ is safe.
% $P$ may not necessarily be a full representation of the underlying MDP, but must encode safety information. 
The safety of $\pi$ is the sum of the disjoint probabilities of the safety of taking each action:
\begin{small}
\begin{equation}
    P_\pi(\texttt{safe}\mid s)=\sum_{a\in\mathcal{A}}P(\texttt{safe}\mid s,a)\cdot\pi(a\mid s)
\end{equation}
\end{small}
By marginalizing the base policy over the safety model, we obtain a reweighted policy that favors safer actions -- this is called \textit{probabilistic shielding}. Formally, %\label{def:pshielding}
given a base policy $\pi$ and a probabilistic safety model $P$, the shielded policy $\pi^+$ is constructed as:
\begin{small}
\begin{equation} \label{eq:pshielding}
    \pi^+(a\mid s)=P_\pi(a\mid s,\mathtt{safe})=\frac{P(\mathtt{safe}\mid s,a)}{P(\mathtt {safe}\mid s)}\pi(a\mid s)
\end{equation}
\end{small}
Succinctly, PLS re-normalizes the action distribution to make unsafe actions less likely. \citet{yang2023safe} chose to implement the model $P$ in ProbLog \citep{de2007problog} for a number of reasons, including that it is easily differentiable and allows easy modeling and planning. This means our safety constraints are to be specified \textit{symbolically} through PL predicates and relations. Within this paper, a \textit{shield} is thus synonymous with a ProbLog program $\mathcal{T}$ (further details in Section~\ref{sec:shield_creation}), which
% that receives two sets of inputs: the first is the action distribution from a policy $\pi(s)$, and the second is a set of probabilistic fact valuations encoding safety-related information that may come from the agent or the environment -- this may differ from the state information used to compute the \textit{base} policy $\pi$.
% \footnote{According to shield classification of \citet{alshiekh2018safe}, PLS can be considered a \textit{preemptive} shield -- it changes the action distribution before the agent takes an action.}. 
% The appropriate ProbLog program $\mathcal{T}$ 
induces a probabilistic measure $\mathbf{P}_\mathcal{T}$. We can (conditionally) query this program to give us various information including the action safety (for each action $a$) as $P(\texttt{safe}|s,a)=\mathbf{P}_\mathcal{T}(\texttt{safe}|a)$, the safety of the whole policy under $s$ as $P_\pi(\texttt{safe}|s)=\mathbf{P}_\mathcal{T}(\texttt{safe})$ and the safe/shielded policy $\pi^+$ as $P_\pi(a|s,\texttt{safe})=\mathbf{P}_\mathcal{T}(a|\texttt{safe})$.

\section{Probabilistic Logic Shields for Multi-Agent Reinforcement Learning}

We present a general SMARL framework in Section~\ref{sec:methods:smarl}, applying (probabilistic) shielding to multi-agent algorithms. Next, in Section~\ref{sec:methods:pltd}, we introduce PLTD, which incorporates logical constraints directly into the TD-learning process. Finally, Section~\ref{sec:shield_creation} outlines the encoding of safety constraints as ProbLog programs.

\subsection{General Framework Description}
\label{sec:methods:smarl}
 A schematic diagram of a parallel SMARL setup is in Figure~\ref{fig:smarl_parallel}. Each agent $p_i$ (indexed with $i\in\{1,...,n\}$) observes a possibly subjective state $s_i^t$ at time $t$, and sends both the computed policy $\pi_i(s^t_i)$ and relevant safety-related information about $s_i^t$ to its shield $\mathcal{T}_i$. Each $\mathcal{T}_i$ is a ProbLog program that is used to compute a safe policy $\pi_i^+(s_i^t)$ according to Eq.~\ref{eq:pshielding} from which an action $a_i^t$ is sampled. Finally, the joint action profile $\vec{a}^t$ is sent to the environment $Env$. After updating its state, $Env$ sends new observations $\mathbf{s}^{t+1}$ and rewards $\vec{\mathbf{r}}^{t+1}$ to the agents, and the cycle continues with $t\leftarrow t+1$ until a terminal state is reached. The \textit{agent-environment-cycle interaction} scheme, used often in turn-based games \citep{terry2021pettingzoo} can likewise be extended naturally to the SMARL framework by having each agent interact with their own shield before sampling an action.

Though SMARL is agnostic to the underlying MARL algorithm, we experiment with two: IQL and IPPO (introduced in Section~\ref{sec:marl}). IQL with parameter sharing will be called \textit{Parameter Sharing IQL (PSQL)}, and IPPO will be called either \textit{Critic Sharing IPPO (CSPPO)} or \textit{Actor-Critic Sharing IPPO (ACSPPO)}, depending on whether the actors and/or critics are shared. In \textbf{Probabilistic Logic SMARL (PL-SMARL)}, at least one of the RL agents uses PLS, optionally sharing policies and parameter updates. PL-SMARL algorithms will be denoted \textit{`Shielded $X$' (S$X$)}; e.g., \textit{`Shielded IQL', (SIQL)} when $X=\text{IQL}$, or \textit{`Shielded CSPPO', (SCSPPO)} when $X=\text{CSPPO}$.

\subsubsection{Safety Guarantees}

Following Definition~\ref{def:safety}, we define relative safety for SMARL:
\begin{definition}[Relative SMARL Safety]\label{def:marlsafety}
    A joint policy $\vec\pi^+=\langle\pi_1^+,\pi_2^+,\ldots,\pi_n^+\rangle$ is defined to be \textbf{at least as safe} as another joint policy $\vec\pi=\langle\pi_1,\pi_2,\ldots,\pi_n\rangle$ if and only if for all agents $i\in N$ with respective shields $\mathcal{T}_i$, it is the case that $P_{\pi^+_i}(\texttt{safe} | s) \geq P_{\pi_i}(\texttt{safe} | s)$ for all reachable states $s\in S$; and is \textbf{strictly safer} if \textit{additionally} for at least one $j\in N$ and $s\in S$, $P_{\pi^+_j}(\texttt{safe} | s) > P_{\pi_j}(\texttt{safe} | s)$.
\end{definition}
\noindent Based on Definition~\ref{def:marlsafety}, we attain the following proposition:
\begin{proposition} \label{proposition:safety}
    A PL-SMARL algorithm with agents $i\in N$ with respective policies $\pi_i^+$ and shields $\mathcal{T}_i$ has a joint shielded policy $\vec\pi^+$ \textbf{at least as safe} as their base joint policy $\vec\pi$, and \textbf{strictly safer} whenever any individual base policy $\pi_j$ violates $\mathcal{T}_j$.
\end{proposition}

A proof is provided in the supplementary materials. This is similar to the first case of \citet{elsayed2021safe}'s correctness proof of the safety of factored shields. However, unlike factored shields, PLS does not guarantee \textit{per-step} safety; rather, it ensures that the 
shielded algorithm's safety is higher than that of the base algorithm \textit{in expectation}.
% \textit{expected} safety of the shielded policy is higher than that of the base policy.

\subsection{PL-Shielded Temporal Difference Learning (PLTD)}\label{sec:methods:pltd}
\citet{yang2023safe} define their shielded algorithm, probabilistic logic policy gradients (PLPG), as an extension of policy gradient methods, specifically, PPO \citep{schulman2017proximal}. Thus, to use it with DQN \citep{mnih2013playing}, SARSA \citep{zhao2016deep} (or other TD-learning methods), formalize the use of a ProbLog shield in conjunction with TD-learning. An important consideration for extending PLS to TD-learning is whether the algorithm is \textit{on-policy} or \textit{off-policy}. \citet{yang2023safe} note that for off-policy algorithms to converge, the exploration and learned policies must cover the same state-action space -- this assumption must be handled with care in the design of exploration strategies. We propose both on- and off-policy algorithms that incorporate safety constraints into Q-learning.

\subsubsection{Objective Function}
PLPG includes a \textit{shielded policy gradient} $[\nabla_\theta\log P_{\pi^+}(\mathtt{safe}|s)]$ and a \textit{safety gradient penalty} $[-\nabla_\theta\log P_{\pi^+}(\mathtt{safe}|s)]$. Since TD-methods do not use policy gradients, we must rely on a modified version of the latter to introduce safety constraints into the objective function.

Let $\mathcal{D}$ be the distribution of $d=\langle s^t,a^t,r^t,s^{t+1},a^{t+1}\rangle$ tuples extracted from a history buffer of the agent interacting with the MDP and $s^t$, $a^t$, $r^t$ be the state, taken action, and reward at time $t$. Let the Q-value approximator $Q_\theta(s,a)$ be parameterized by some parameters $\theta$. The off-policy loss (Q-learning based) and on-policy loss (SARSA-based) are augmented using the aforementioned safety penalty term. We refer to these methods as \textit{Probabilistic Logic TD Learning} (PLTD).

\begin{definition}[Probabilistic Logic TD Learning] \label{def:pltd}The PLTD minimization objective is:
\begin{small}
\begin{equation}
\begin{aligned}\label{eq:pltd}
    \mathcal{L}^{Q^+}(\theta) = \mathbb{E}_{d\sim \mathcal{D}}
    \Big[ \Big( r^t + \gamma \mathcal{X}
     - Q_\theta(s^t, a^t) \Big)^2 - \alpha S_P
    \Big]
\end{aligned}
\end{equation}
\end{small}
\noindent where $\mathcal{X}=\max_{a'} Q_\theta(s^{t+1}, a')$ for off-policy, and $\mathcal{X}=Q_\theta(s^{t+1}, a^{t+1})$ for on-policy DQN; $S_P=\log P_{\pi^+}(\mathtt{safe}|s)$ is the safety penalty; and $\alpha\in\mathbb{R}_{\geq 0}$ is the safety coefficient, or the weight of the safety penalty.
\end{definition}
The quantity $[-\log P_{\pi^+}(\mathtt{safe}|s)]$ is interpreted as \textit{the probability that $\pi^+$ satisfies the safety constraints} \citep{yang2023safe}. Under perfect sensor information, PLTD reduces to vanilla TD-learning within the reachable set of safe states. We thus get the following convergence property (following \citet{tsitsiklis1996analysis}):
\begin{proposition}\label{prop:pltdconvergence}
    PLTD, i.e.\ Eq.~\ref{eq:pltd} converges to an optimal safe policy given perfect safety information in all states in tabular and linear function approximation settings.
\end{proposition}
We prove this in the supplementary. For neural Q-networks, global convergence is currently unsolved, though we can adopt common DQN stabilization tricks (summarized in \citep{hessel2018rainbow}). While we observe empirical stability of PLTD, a complete theoretical analysis is left for future work.
\vspace{-0.5\baselineskip}
\subsection{Shield Construction using ProbLog}\label{sec:shield_creation}
A shield in PLS is defined using a ProbLog program \citep{de2007problog}, denoted $\mathcal{T}$. Each consists of three components: ($i$) an annotated disjunction $\Pi_s$ of predicates whose values are set to a base action distribution $\pi$, ($ii$) a set of constraint-related inputs describing the current state $\mathbf{H}_s$, and ($iii$) a set of sentences $\mathcal{KB}$ (for \emph{knowledge base}) that specify the agent's constraints, including a definition for a predicate `\texttt{safe\_next}' (the predicted safety of the next state under $\pi$). Again, `safety' broadly refers to any probabilistic constraint satisfaction goal (ref.\ Sec~\ref{sec:safety}).

An exemplary shield (Shield~\ref{shield:prelim_mixed}) for constraining agent behavior towards the mixed Nash equilibrium in Stag-Hunt games is described as such: The probabilistic valuations of the `\texttt{action($\cdot$)}' predicates are inferred from $\pi$. The `\texttt{sensor($\cdot$)}' predicates are set to the normalized absolute difference between the precomputed mixed Nash strategy and the mean of the agent's historical actions. The predicate `\texttt{safe\_next}' is inferred via ProbLog semantics \cite{cozman2017semantics} and determines the safety of the next state under $\pi$ -- safer states are those where the stated normalized absolute differences are low. For brevity, we describe subsequent shields in-line rather than displaying its full ProbLog source. For details on how we designed the shields for each experiment, we kindly direct the reader to the supplementary.
\begin{lstlisting}[language=Prolog, caption={Shield (\shieldname{mixed}) for mixed \textit{Stag-Hunt} equilibrum.}, label={shield:prelim_mixed}]
% actions
action(0)::action(stag);
action(1)::action(hare).
% sensors
sensor_value(0)::sensor(stag_diff).
sensor_value(1)::sensor(hare_diff).
% safety constraints
unsafe_next :- action(stag), 
               sensor(stag_diff).
unsafe_next :- action(hare), 
               sensor(hare_diff).
safe_next :- \+unsafe_next.
\end{lstlisting}
\vspace{-10pt}
\subsubsection{Deterministic Shields}
Using PLS, agents may be made to be deterministic which can be useful when training agents in environments requiring hard constraints. This can be formulated as the following (proof in the supplementary):
% This can be formulated as the following 
% (proven in the supplementary)
\begin{proposition} \label{proposition:deterministic}
    For any game there exists a shield $\mathcal{T}$ such that the resulting shielded agent deterministically selects a single action $a\in\mathcal{A}$ for each state $s\in S$ if the shielded policy $\pi^+$ computes all other actions $a'\in\mathcal{A}\setminus \{a\}$ as perfectly unsafe.
\end{proposition}

\section{Results}

We organize our experiments around key multi-agent phenomena relevant to several open MARL challenges. Rather than focusing on games individually, each section centers on a particular strategic challenge (e.g., equilibrium selection, social dilemmas) and uses one or more games as illustrative testbeds. For the sake of clarity, details are moved to the supplementary material, such as implementation details, shield programs for each experiment, a validation of PLTD in a single-agent experiment, all hyperparameters and training curves. All displayed results are averaged over 5 random seeds to account for variability in training. All reported errors correspond to standard deviations across runs and respective episodes/steps. Our code is open-sourced at \url{https://github.com/satchitchatterji/ShieldedMARL}.

\subsection{Coordination and Equilibrium Selection (Stag-Hunt)}

We use \textit{Stag-Hunt} to test the ability of PLS to guide agents towards normative behaviors in games with multiple Nash equilibria. IPPO serves as the baseline, while SIPPO agents are equipped with a shield constraining the agents towards either the pure cooperative (\shieldname{pure}) or mixed Nash equilibrium (\shieldname{mixed}) --  play $Stag/Hare$ 60\%/40\% of the time with an expected utility of 2.6. \shieldname{pure} simply tells the agent to not play $Hare$, making it a deterministic shield according to Proposition~\ref{proposition:deterministic}. \shieldname{mixed}, used as an example in Section~\ref{sec:shield_creation}, defines safety as the absolute difference between the precomputed mixed Nash strategy and the mean of the agent's historical actions.

\begin{table}[h]
\caption{Results for \textit{Stag-Hunt} with PPO-based agents over the last 50 training episodes. $\bar{r}$ is the mean return and \textit{cooperation} is defined as ${t_{max}^{-1}}\sum_{t=0}^{t_{max}} \mathbf{P}_{\mathcal{T}_{pure}}(\texttt{safe}|s^t)$.}

% \label{tab:sh}
% \centering
% \begin{tabular}{l@{\hspace{6mm}}l@{\hspace{6mm}}l@{\hspace{6mm}}l}
% \toprule
% \textbf{Algorithm} & $\bar{r}$ (train) & $\bar{r}$ (eval) & cooperation \\
% \midrule
% IPPO & 1.99$\pm$0.03 & 1.99$\pm$0.02 & 0.01$\pm$0.01 \\
% SIPPO ($\mathcal{T}_{pure}$) & 4.00$\pm$0.00 & 4.00$\pm$0.00 & 1.00$\pm$0.00 \\
% SIPPO ($\mathcal{T}_{mixed}$) & 2.57$\pm$0.48 & 2.63$\pm$0.43 & 0.58$\pm$0.08 \\
% \bottomrule
% \end{tabular}
% \end{table}

\label{tab:sh}
\centering
\begin{tabular}{l@{\hspace{6mm}}l@{\hspace{6mm}}l@{\hspace{6mm}}l}
\toprule
{Algorithm} & $\bar{r}$ (train) & $\bar{r}$ (eval) & cooperation \\
\midrule
IPPO & 1.99$\pm$0.03 & 1.99$\pm$0.02 & 0.01$\pm$0.01 \\
SIPPO ($\mathcal{T}_{pure}$) & 5.00$\pm$0.00 & 5.00$\pm$0.00 & 1.00$\pm$0.00 \\
SIPPO ($\mathcal{T}_{mixed}$) & 2.57$\pm$0.48 & 2.63$\pm$0.43 & 0.58$\pm$0.08 \\
\bottomrule
\end{tabular}
\end{table}

Table~\ref{tab:sh} shows the results this setting. The IPPO agents converge to the non-cooperative Nash equilibrium, consistently playing \textit{Hare} with a reward of 2 per step. In contrast, agents using $\mathcal{T}_{pure}$ quickly and reliably play \textit{Stag}, earning a higher reward of 5. The unshielded agents never converge to the mixed Nash equilibrium, even though PPO is capable of learning stochastic policies -- small deviations push these agents toward playing pure strategies, often resulting in the non-cooperative equilibrium. However, \textit{shielded} PPO agents using $\mathcal{T}_{mixed}$ successfully adopt the mixed strategy, though with high variability due to this attracting nature of the pure equilibria.

\subsection{Temporal Coordination (Centipede)}

We use \textit{Centipede} to test how well SMARL guides agents to cooperate to achieve large collective long-term rewards, even in scenarios where the dominant strategy suggests defection. A shield was constructed (\shieldname{continue}) to encourage agents to play the game by constraining against early defection. Table~\ref{tab:cent} shows results for pairs of MARL agents learning to play the \textit{Centipede} game. IPPO agents occasionally learn to play \textit{Continue} through all $t_{max}=50$ steps within 500 episodes. In contrast, we see the SIPPO agents playing the full game from the start of training.

\begin{table}
\caption{Results for \textit{Centipede} with PPO- and DQN-based agents over the last 50 training episodes. \textbf{Safety} is defined as $\frac{1}{T}\sum_{t=0}^T \mathbf{P}_{\mathcal{T}_{cent}}(\texttt{safe}|s^t)$ and $R_{\text{ep}} = \sum_{t=0}^{{t_{max}}} r^t$.}
\label{tab:cent}
\centering
\begin{tabular}{l@{\hspace{5mm}}l@{\hspace{5mm}}l@{\hspace{5mm}}l}
\toprule
{Algorithm} & $R_{\text{ep}}$ (train) & $R_{\text{ep}}$ (eval) & Safety \\
\midrule
IPPO & 42.35$\pm$36.62 & 42.83$\pm$35.81 & 0.83$\pm$0.23 \\
SIPPO & 100.50$\pm$0.00 & 100.50$\pm$0.00 & 1.00$\pm$0.00 \\
\midrule
IQL ($\epsilon$-greedy) & 34.62$\pm$46.59 & 34.70$\pm$46.53 & 0.68$\pm$0.23 \\
SIQL ($\epsilon$-greedy) & 100.50$\pm$0.00 & 100.50$\pm$0.00 & 1.00$\pm$0.00 \\
\midrule
IQL (softmax) & 1.73$\pm$1.01 & 30.10$\pm$38.61 & 0.73$\pm$0.21 \\
SIQL (softmax) & 100.50$\pm$0.00 & 100.50$\pm$0.00 & 1.00$\pm$0.00 \\
\bottomrule
\end{tabular}
\end{table}

\begin{figure*}[]
    \centering
\includegraphics[width=\linewidth]{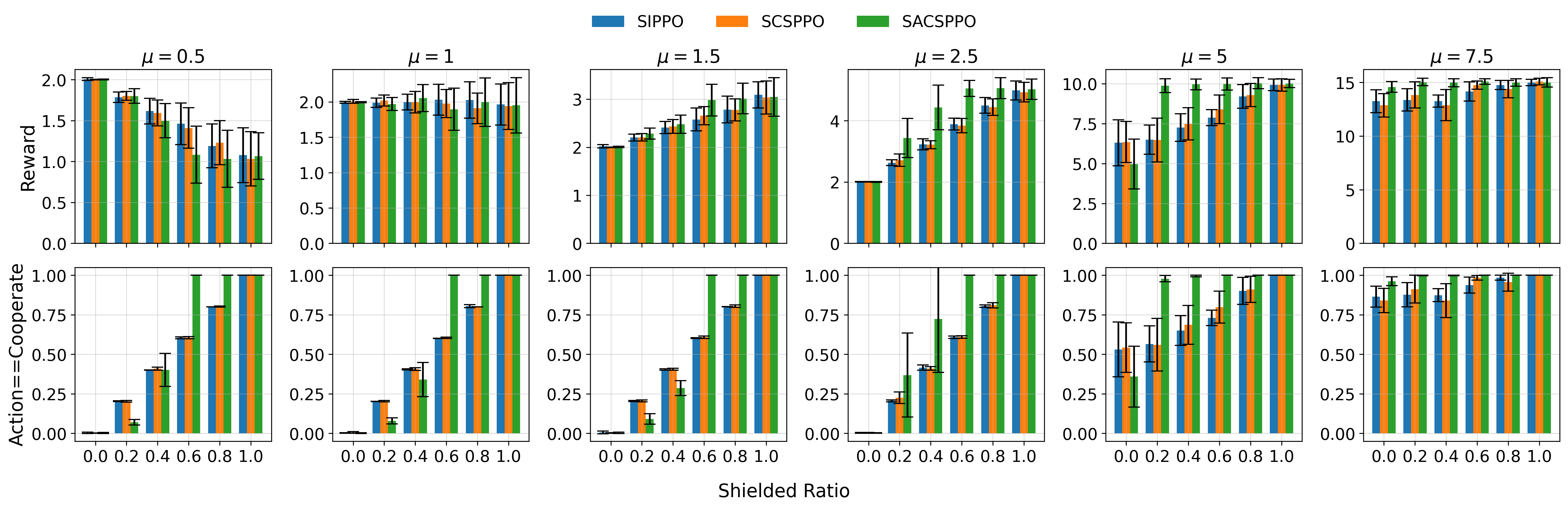}\vspace{-20pt}
    \caption{Results for the final 10 episodes of the 5-player \textit{EPGG} for SIPPO, SCSPPO, and SACSPPO agents with varying $\mu$ and ratios of shielded agents.}
    \label{fig:epgg_many}
\end{figure*}

For DQN-based agents, $\epsilon$-greedy and softmax exploration policies were tested. Unshielded $\epsilon$-greedy agents fail to progress far, as early on in exploration, $\epsilon\approx1$ ($\pi$ is uniform), and thus a 25\% probability of the agents going to the second stage of the game. The probability of reaching some early stage $\ell$ is $\approx0.25^{\ell-1}$. This discourages exploration, as stopping early provides a small but consistent reward, aligning with the Nash/subgame-perfect equilibrium. Unshielded softmax agents perform slightly better, but with high variance -- some agent pairs cooperate while most exit early. This highlights the influence of the choice of exploration strategy, and how using shield can guide desirable behavior -- all SIQL variations learn to play \textit{Continue} consistently.

\begin{figure}[H]
    \centering
    % \begin{minipage}[b]{0.3\textwidth}
    %     \centering
    %     \includegraphics[width=\textwidth]{Experiments/images/epgg/bar_safety_training_epgg.png}
    %     \vspace{0.3em}

    %     \textbf{(a)} Mean training rewards.
    %     \label{fig:epgg_bar_ppo:subfig1}
    % \end{minipage}
    % \hfill
    % \begin{minipage}[b]{0.3\textwidth}
    %     \centering
    %     \includegraphics[width=\textwidth]{Experiments/images/epgg/bar_safety_evaluation_epgg.png}
    %     \vspace{0.3em}

    %     \textbf{(b)} Mean evaluation rewards.
    %     \label{fig:epgg_bar_ppo:subfig2}
    % \end{minipage}
    % \hfill
    % \begin{minipage}[b]{0.3\textwidth}
    %     \centering
    %     \includegraphics[width=\textwidth]{Experiments/images/epgg/bar_safety_safety_epgg.png}
    %     \vspace{0.3em}

    %     \textbf{(c)} Mean safety/cooperation
    %     \label{fig:epgg_bar_ppo:subfig3}
    % \end{minipage}
    \includegraphics[width=\linewidth]{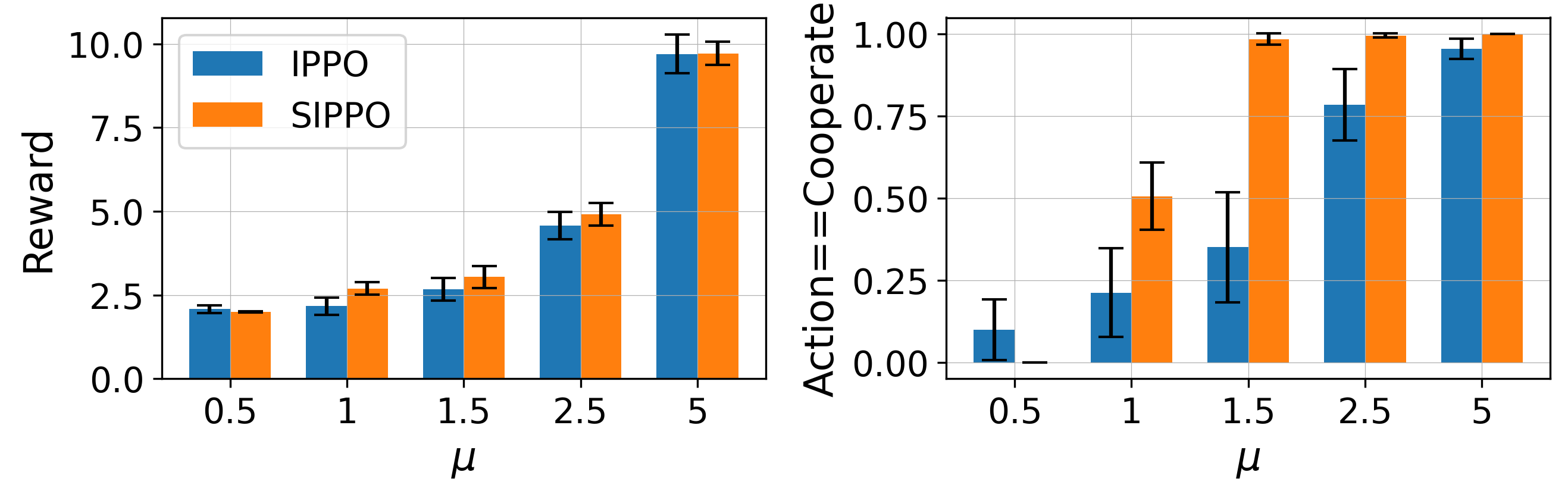}\vspace{-20pt}
    \caption[Results for 2-player \textit{EPGG}.]{Results for the final 10 episodes for the two-player \textit{Extended Public Goods Game} for PLPG-based agents.}
            \label{fig:epgg_bar_ppo}
\end{figure}

\vspace{-20pt}
\subsection{Social Dilemmas in Stochastic, Mixed-Motive Environments (Extended Public Goods Game)}

We use \textit{EPGG} to investigate how PLS can promote cooperation in environments with stochastic social incentives. Our experiments address two central questions: ($i$) \textit{can agents be made to learn strategies that reflect long-term expectations rather than short-term payoffs?}, and ($ii$) \textit{how does the effectiveness of shielding scale when only a subset of agents is shielded, and under different degrees of parameter sharing?} In the first experiment, we assess how shielding can steer agents toward the \textit{expected} Nash equilibrium of the EPGG. In the second, we evaluate how the degree of parameter sharing affects emergent cooperation.

\subsubsection{Learning Long-Term Normative Behavior}

We begin with a 2-player setting to isolate the effect of shielding on equilibrium selection under payoff uncertainty. This controlled setup enables a clear comparison between shielded agents learning long-term behavior versus unshielded agents reacting to stochastic payoffs. Limiting the population size removes confounding factors such as social influence and allows for direct game-theoretic analysis, providing a clean baseline before scaling to larger populations.

To test whether PLS can guide agents toward long-term cooperation, we designed a shield (\shieldname{EPGG}) that directs agents to follow the expected Nash equilibrium, based on an empirical estimate $\hat{\mu}$ of the stochastic payoff multiplier $f_t \sim \mathcal{N}(\mu, 1)$. This contrasts with the unshielded baseline, where agents maximize the instantaneous payoff.

Figure~\ref{fig:epgg_bar_ppo} reports results for $\mu \in \{0.5, 1, 1.5, 2.5, 5.0\}$. Unshielded IPPO agents reliably converge to the instantaneous equilibrium -- defecting when $f_t<1$ and cooperating when $f_t>2$. In contrast, SIPPO agents quickly converge to the desired behavior, leading to more prosocial behavior under uncertainty. We observe that SIPPO agents exhibit lower reward variance and more stable cooperation compared to the baseline, especially in mixed-motive cases ($1 < \mu < 2$). For $\mu = 0.5$, the shielded agents correctly learn to defect, but earn slightly less due to foregoing occasional cooperative gains from high $f_t$. Overall, shielded agents match or exceed the baseline in reward and demonstrate more consistent cooperation.
% \vspace{-20pt}
\subsubsection{Partial Shielding and Parameter Sharing}

To focus on how shield coverage and parameter sharing influence group behavior, we use a simplified shield that always selects the cooperative action. Unlike the previous experiment, this avoids computing expected equilibria and isolates the effect of shielding structure. Importantly, even though the shield is deterministic, shielded agents still receive gradient-based learning signals. This allows safety information to influence shared parameters, enabling unshielded agents to adopt cooperative behavior indirectly.

We run EPGG with $n=5$ agents, where a subset $k\in\{0,\dots,5\}$ of agents are shielded with a deterministic policy that always cooperates (reported as a \textit{shielded ratio}=$k/n$). We compare SIPPO, SCSPPO, and SACSPPO to test how different levels of parameter sharing modulate the influence of the shielded agents on the unshielded population.

Figure~\ref{fig:epgg_many} shows grouped bar charts for $\mu \in \{0.5, 1, 1.5, 2.5, 5, 7.5\}$. Results show a clear trend: increasing the number of shielded agents leads to higher overall cooperation and greater episodic reward in the mixed-motive range ($1 < \mu < 5$). Furthermore, this effect is amplified by parameter sharing: SACSPPO consistently achieves the same or higher cooperation than SCSPPO and SIPPO at the same shielded ratio. In fully cooperative settings ($\mu=7.5$), all configurations converge to high cooperation, but SACSPPO still exhibits the lowest variance. In competitive settings ($\mu=0.5$), cooperation remains low for the baseline algorithms, as expected, increasing with the number of shielded agents -- this results in an increasingly lower mean reward with the benefit of the agents instead cooperating. Notably, in the competitive regime, SACSPPO exhibits the highest level of cooperation among the methods, indicating that shared learning with shielded agents can drive prosocial behavior in unshielded agents -- even when it reduces their individual rewards.

These results demonstrate that probabilistic shields can influence unshielded agents via shared parameters, serving as an effective conduit for prosocial behavior in multi-agent social dilemmas.

% \vspace{-20pt}
\subsection{Spatio-Temporal Coordination with Imperfect Information (Markov Stag-Hunt)}\label{sec:exp:msh}
We use the \textit{Markov Stag-Hunt} environment \citep{peysakhovich2018prosocial} to investigate how PLS-SMARL enables coordination in sequential, stochastic, and partially observable multi-agent settings. This environment extends the normal-form \textit{Stag-Hunt} to a spatial domain where agents must learn when and how to cooperate. Three experiments are conducted: ($i$) we study the effect of shield constraint strength on agent behavior, ($ii$) we evaluate how cooperation scales with population size, (iii) finally we explore how cooperation emerges when only a subset of agents is shielded, highlighting the benefits of shielded parameter sharing.

\subsubsection{Variation in Shield Strength} \label{exp:shield_stength_variation}

\begin{table*}[]
\vspace{-10pt}\caption{Results for 2-player \textit{Markov Stag-Hunt} with shielded and unshielded PPO-based agents. Columns show: total plants harvested, stags killed, penalties from solo hunts, episodic reward ($R_{\text{ep}} = \sum_{t=0}^{t_{max}} r^t$), and mean episodic safety (${t_{max}^{-1}}\sum_{t=0}^{t_{max}} \mathbf{P}_{\mathcal{T}_{(\cdot)}}(\texttt{safe}|s^t)$).}
\label{tab:msh_full_results}
\centering
\begin{tabular}{l@{\hspace{4mm}}l@{\hspace{4mm}}l@{\hspace{4mm}}l@{\hspace{4mm}}l@{\hspace{4mm}}l@{\hspace{4mm}}l}
\toprule
{Algorithm} & $\sum$plants & $\sum$stags & $\sum$penalties & $R_{\text{ep}}$ (train) & $R_{\text{ep}}$ (eval) & Mean safety \\
\midrule
IPPO & 18.74$\pm$4.98 & 0.09$\pm$0.11 & 5.50$\pm$2.04 & 13.71$\pm$4.97 & 16.07$\pm$5.80 & -- \\
SIPPO ($\mathcal{T}_{weak}$) & 18.73$\pm$4.26 & 12.80$\pm$7.41 & 14.25$\pm$3.57 & 68.48$\pm$35.19 & 79.80$\pm$38.33 & 0.96$\pm$0.01 \\
SIPPO ($\mathcal{T}_{strong}$) & 11.33$\pm$4.08 & 77.32$\pm$8.82 & 11.07$\pm$3.25 & 386.86$\pm$46.50 & 393.07$\pm$65.27 & 0.95$\pm$0.01 \\
\bottomrule
\end{tabular}
\end{table*}

To isolate the effect of constraint strength on emergent cooperation, we compare two shields: $\mathcal{T}_{weak}$ and $\mathcal{T}_{strong}$. These were constructed such that $\mathcal{T}_{weak}\subset\mathcal{T}_{strong}$ in terms of their safety constraints. The former constrains agents to cooperate only when both are near a stag, while the latter more strictly promotes cooperation regardless of global position. Both shields are applied in a 2-agent setting using the SIPPO algorithm, with unshielded IPPO as a baseline.

% Two shields, $\mathcal{T}_{strong}$ and $\mathcal{T}_{weak}$, were constructed such that $\mathcal{T}_{strong}\subset \mathcal{T}_{weak}$ in terms of their safety constraints. Broadly, $\mathcal{T}_{weak}$ constrains the agents' behaviors to strongly cooperate when they both are near the stag and are unconstrained otherwise, and $\mathcal{T}_{strong}$ influences with towards cooperation regardless of where they are placed on the grid. 

In this experiment, analyzing just reward or safety does not fully capture cooperative behavior, so the key metrics displayed in Table~\ref{tab:msh_full_results} are: total plant harvests per episode, successful cooperative stag hunts, and unsuccessful attempts to hunt a stag alone. Additionally, we show the total return for each episode, and mean safety for shielded agents. 

As expected, IPPO agents avoid cooperation (thereby avoiding risk) and steadily harvest plants for a small, reliable reward. SIPPO agents using $\mathcal{T}_{weak}$ show moderate improvement, achieving occasional stag hunts while maintaining similar plant collection. Their increased cooperation also results in more penalties, as agents are willing to risk solo hunts. Under $\mathcal{T}_{strong}$, cooperation improves dramatically: agents prioritize stag hunting and harvest fewer plants, yet earn more than five times the episodic reward compared to SIPPO with $\mathcal{T}_{weak}$. This suggests that well-designed stronger shields not only induce safer behavior, but may lead to better coordination and long-term returns.
% $\blacktriangleright$ \textbf{\textit{Rewards:}} IPPO agents follow the same behavioral pattern as in the NFG version of the \textit{Stag-Hunt}, preferring to harvest plants for consistent, low-risk rewards. They avoid cooperating to not risk failed solo stag hunts, favoring harvesting plants. 
% % As noted by \citet{peysakhovich2018prosocial}, adjusting the reward structure may encourage more social behavior. 
% SIPPO pairs gather more reward by learning to be more cooperative.

% $\blacktriangleright$ \textbf{\textit{Behavior:}} SIPPO agents with $\mathcal{T}_{weak}$ show significant improvement in cooperation, achieving an average of around 13 successful stag hunts per episode by 500 episodes, while maintaining IPPO-level plant harvests. However, their willingness to take risks increases, reflected in an uptick in the number of stag penalties over time as the \textit{potential} rewards outweigh risk-aversion. Agents generally act very safely with respect to their own shields. SIPPO agents under $\mathcal{T}_{strong}$ demonstrate even greater cooperation, focusing more on stag hunts and less on plant harvesting, despite only slightly lower stag penalties as the $\mathcal{T}_{weak}$ agents. This implies that they learn better cooperative strategies. Their rewards are significantly higher (by more than a factor of 5) than those of agents using $\mathcal{T}_{weak}$, which are already much greater than the unshielded agents' rewards.

\subsubsection{Scaling Cooperation with Population Size}\label{exp:scaling_msh}

We next examine how cooperation and safety scale with group size when all agents are shielded using $\mathcal{T}_{strong}$. We vary the number of agents $n \in \{2, 3, 4, 5\}$ and compare three unshielded baselines (IPPO, CSPPO, ACSPPO) and three shielded variants (SIPPO, SCSPPO, SACSPPO). Metrics include total plants harvested, stags hunted, solo hunt penalties, mean episodic safety, and per-agent reward.

Figure~\ref{fig:msh_many_agent} shows that unshielded agents (IPPO, CSPPO, ACSPPO) consistently avoid cooperation, harvesting plants for low-risk, low-reward outcomes. Notably, ACSPPO -- despite its fully shared architecture -- performs worst overall, with near-zero safety and minimal rewards, underscoring that parameter sharing alone is insufficient to induce prosocial behavior.

In contrast, all shielded variants show increasing cooperation as population size grows. SACSPPO agents, in particular, demonstrate the strongest cooperative intent: they achieve the highest number of stag hunts, while also incurring more penalties, reflecting increased willingness to take coordination risks. Despite these penalties, SACSPPO yields the highest rewards and very high adherence to the shield's definition of safety constraints. While per-agent rewards tend to decrease with larger populations, this is expected given that key environment parameters -- such as the number of stags and plants, grid size, and episode length -- remain fixed across all settings.

These results demonstrate that PLS scales robustly with population size. Notably, with full parameter sharing, SACSPPO leverages shared gradients and consistent safety signals to generalize cooperation more effectively across agents, even in complex multi-agent settings.
\begin{figure}
    \centering
    \includegraphics[width=\linewidth]{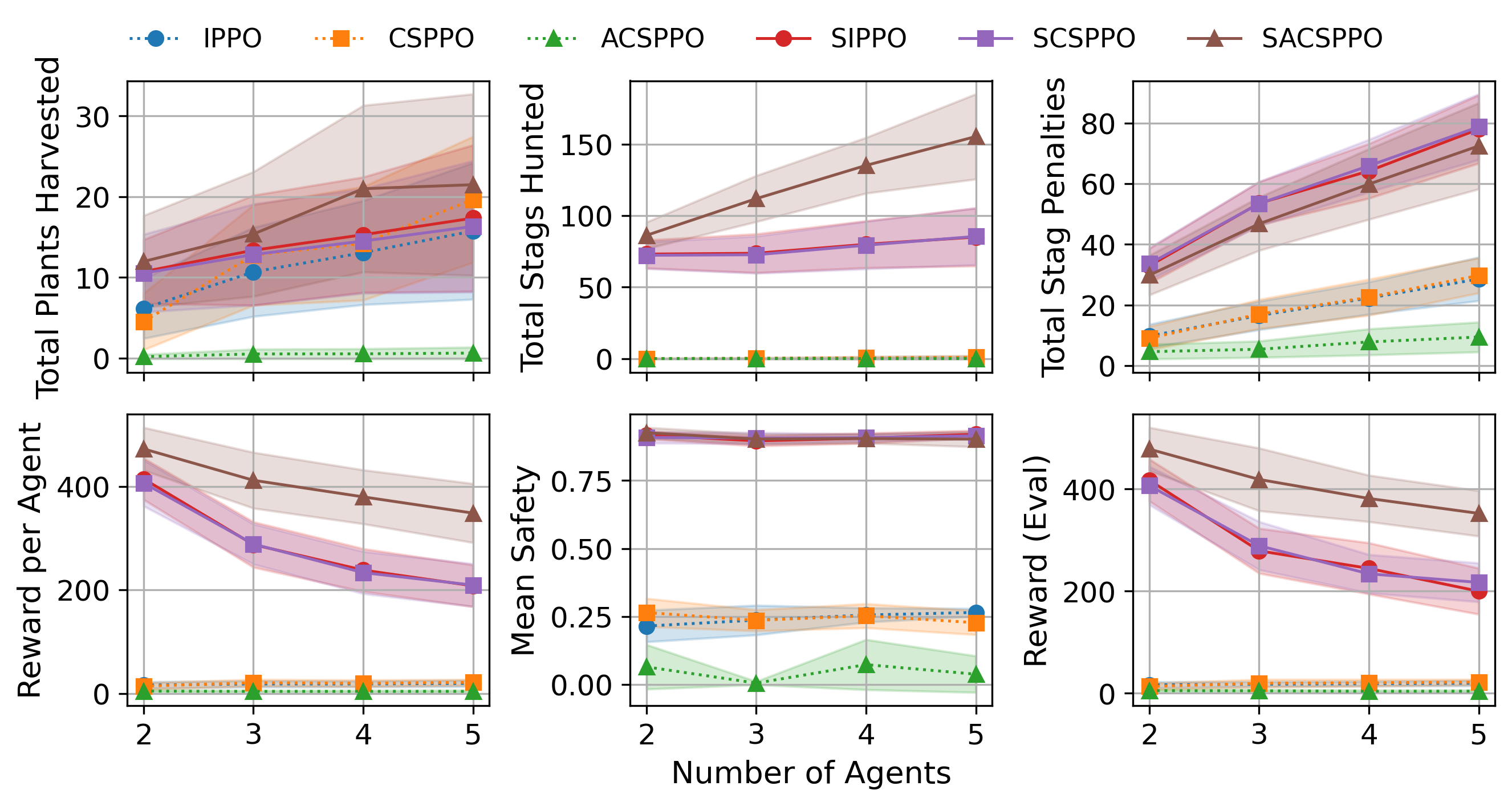}\vspace{-20pt}
    \caption{Results for the final 50 episodes of \textit{Markov Stag-Hunt} with all agents shielded with $\mathcal{T}_{strong}$, for $n \in \{2,3,4,5\}$.}
    \label{fig:msh_many_agent}
\end{figure}
\begin{figure}
    \centering
\includegraphics[width=\linewidth]{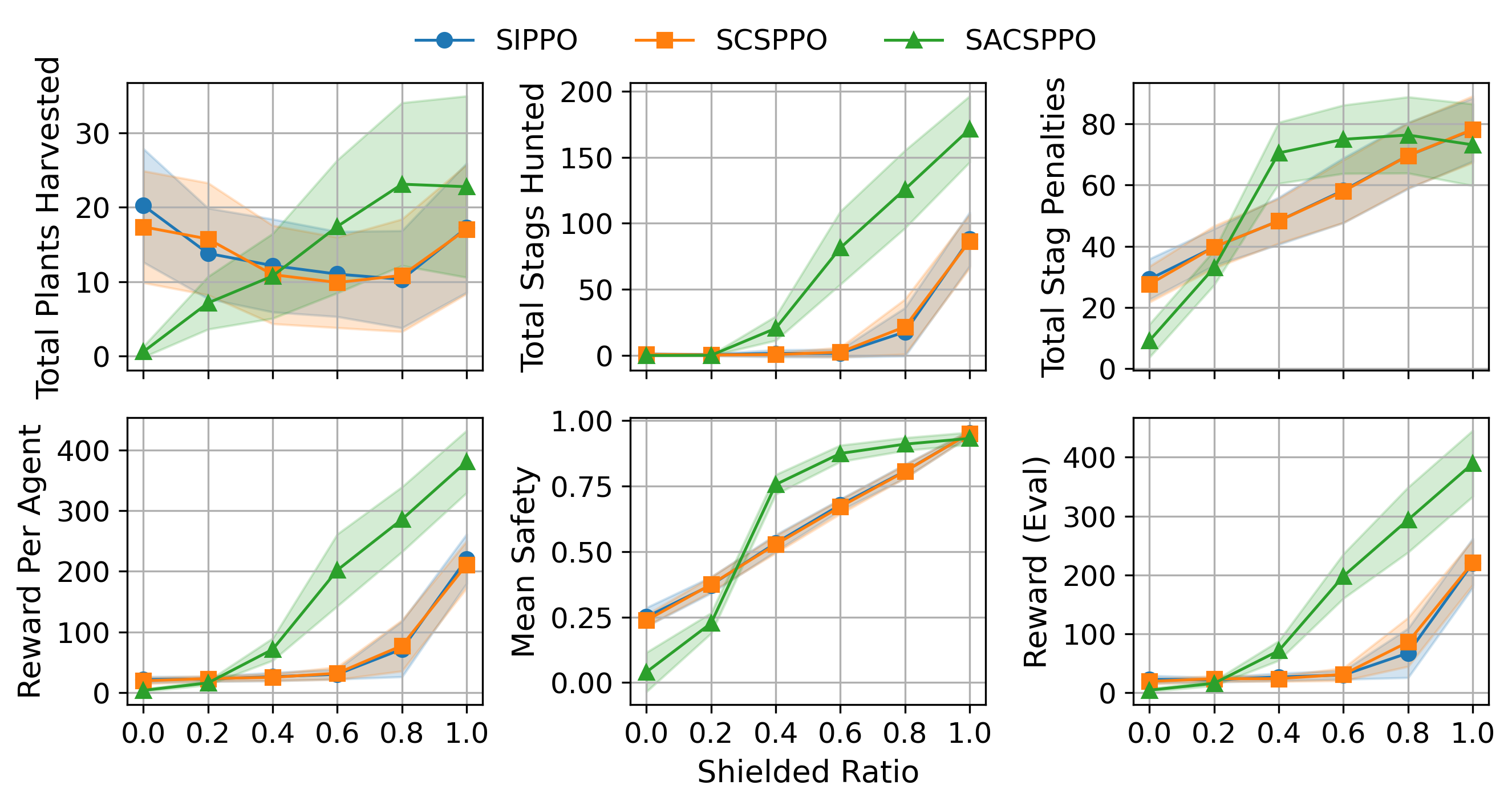}\vspace{-20pt}
    \caption{Results for the final 50 episodes of the \textit{Markov Stag-Hunt} with $5$ agents, varying the number of shielded agents using the strong shield $\mathcal{T}_{strong}$.}
    \label{fig:msh_part_shielding}
\end{figure}
\subsubsection{Partial Shielding and Parameter Sharing} \label{exp:partial_shield}

Finally, we investigate cooperation in partially shielded populations. Fixing the population size at $n=5$, we vary the number of shielded agents $k \in \{0, 1, \dots, 5\}$ and compare SIPPO, SCSPPO, and SACSPPO. All shielded agents use the strong shield $\mathcal{T}_{strong}$, and unshielded agents remain unconstrained. This experiment tests how shielded behavior propagates through shared parameters in populations where only a subset of agents receive explicit safety constraints.
\newline\indent~Figure~\ref{fig:msh_part_shielding} shows that cooperation, as measured by total stags hunted, increases monotonically with the number of shielded agents across all methods. SACSPPO shows the strongest effect: even with just two shielded agents, it significantly outperforms the other methods in cooperative behavior and safety adherence. As more agents are shielded, SACSPPO reaches a high level of cooperation, but also accumulates the most stag penalties -- suggesting increased risk-taking and cooperative intent, even when coordination occasionally fails.

Reward and evaluation performance trend similarly. While all methods improve with greater shield coverage, SACSPPO consistently yields the highest returns, particularly at higher shielded ratios. This indicates that the safety-driven cooperative strategies learned via PLS are not only prosocial but also effective (reward-maximizing). In contrast, SIPPO and SCSPPO agents remain more conservative: they hunt fewer stags, incur fewer penalties, and achieve lower overall rewards.

Plant harvesting does not uniformly decrease with increasing shielded ratio across all methods. While SIPPO and SCSPPO show a modest decline -- suggesting a shift from self-interested strategies to coordinated risk-taking -- SACSPPO exhibits an unexpected increase in plant harvesting as more agents are shielded. Taken together with SACSPPO’s high shield adherence, this may indicate that stronger parameter sharing drives homogenization that improves reward maximization, encouraging agents to opportunistically exploit lower-risk strategies. Alternatively, the behavior could reflect an emergent division of labor, where agents already near the stag focus on hunting while others stabilize returns via plant harvesting. Distinguishing between these hypotheses would require further analysis of agent roles and policy diversity. Importantly, mean safety steadily increases with shield coverage in all methods, but SACSPPO achieves the highest safety levels -- highlighting its superior ability to align unshielded agents with shielded behavior through shared learning.

Together, these results demonstrate that PLS can induce prosociality even under partial shielding, and that shared parameters are a key conduit for propagating safety-aligned behavior to unshielded agents.

\section{Related Work}

In safe (MA)RL, several recent frameworks address safety during learning and execution \citep{gu2024review}. \citet{gu2023safe}, propose \textbf{constrained Markov games} and \textbf{MACPO}, but without formal language semantics. \citet{subramanian2024neuro} integrate \textbf{PL neural networks} as the function approximator within agents to affect interpretable top-down control; however, in our method, the agents may use any function approximator, maximizing potential learning flexibility. \citet{waga2022dynamic} use \textbf{Mealy machines} to approximate a world model to ensure safety, leading to some inherent symbolic interpretability. Other related concepts include \textbf{dynamic shielding}, for instance, \citet{elsayed2021safe} who introduce two frameworks that dynamically enforce safety constraints -- our own method is similar to their \textbf{factored shields}. 
\citet{banerjee2024dynamic}, \citet{zhang2019mamps} and \citet{xiao2023model} are recent works that use the notion of (dynamic) \textbf{model predictive shielding}. Though these show promising results, these do not benefit from the interpretability of formal language semantics.
\citet{melcer2022shield} explore \textbf{decentralized shielding} (similar to factored shields), i.e.\ allowing each agent to have its own shield, reducing computational overhead and enhancing scalability. \citep{melcer2022shield} and \citep{elsayed2021safe} use \textbf{LTL as formal language}, though with a high knowledge requirement of the underlying MDP -- this is in contrast to our use of PLS which requires only safety-related parts of the states. Finally, some methods, such as \citet{wang2024safe}, enforce safety through \textbf{natural language constraints}; in contrast, our approach builds on PLS to provide stronger formal guarantees.

\section{Discussion and Conclusion}
In this article, we introduced Probabilistic Logic Shielded MARL, a novel set of DQN- and PPO-based methods that extend PLS to MARL. Several classes of games were analyzed and shown to benefit from SMARL in terms of safety and cooperation, and theoretical guarantees of these methods were presented.

Despite the benefits implied by our results, a few limitations and extensions must be reflected upon:
% \begin{compactenum}
    $\blacktriangleright$ First, current methods for solving PL programs become computationally expensive as the state and action spaces grow \citep{fierens2015inference}. However, as with factored shields \citep{elsayed2021safe}, our method scales only linearly with the number of agents. %We reiterate the importance of studying two-player games as they are foundational in understanding larger population dynamics. 
    % Thus, future work will consider PLS being applied to larger populations. 
    $\blacktriangleright$  The use of \textit{a-priori} constraints in PLS provides \textit{predefined} safety knowledge, which is straightforward for simple environments but becomes complex in dynamic, real-world systems. This approach limits scalability, relying on expert-defined safety measures. Most existing dynamic shielding methods do not use formal language semantics, and thus a balance must be struck between autonomous shield updating and formal verification.
    $\blacktriangleright$  The complexity of the ProbLog program in PLS impacts performance and behavior, as seen in the Markov Stag-Hunt experiments where the stronger shield promoted more cooperation but at the cost of increased complexity w.r.t.\ hand-designed rules. It may be interesting to examine how varying shield complexity affects computational load and safety outcomes. 
    $\blacktriangleright$  Partial shielding was explored where shielded agents update their parameters with safety gradients, while unshielded agents focus on maximizing rewards, optionally sharing parameters. A weighted shared parameter update may balance safety and reward optimization dynamically. 
    $\blacktriangleright$  Finally, while PL-SMARL aims to enhance safety, its implementation must carefully consider ethical and societal impacts, as hand-designed constraints may introduce biases, under-specified definitions, or system vulnerabilities that current research has yet to fully address.

%%%%%%%%%%%%%%%%%%%%%%%%%%%%%%%%%%%%%%%%%%%%%%%%%%%%%%%%%%%%%%%%%%%%%%%%

%%% Use this environment to include acknowledgements (optional).
%%% This will be omitted in doubleblind mode.

\begin{ack}
This publication is part of the project `Hybrid Intelligence: augmenting human intellect' (\url{https://hybrid-intelligence-centre.nl},) with project number 024.004.022 of the research programme ‘Gravitation’ which is (partly) financed by the Dutch Research Council (NWO).
\end{ack}

%%%%%%%%%%%%%%%%%%%%%%%%%%%%%%%%%%%%%%%%%%%%%%%%%%%%%%%%%%%%%%%%%%%%%%%%

%%% Use this command to include your bibliography file.

\bibliography{mybibfile}

\begin{thebibliography}{48}
\providecommand{\natexlab}[1]{#1}
\providecommand{\url}[1]{\texttt{#1}}
\expandafter\ifx\csname urlstyle\endcsname\relax
  \providecommand{\doi}[1]{doi: #1}\else
  \providecommand{\doi}{doi: \begingroup \urlstyle{rm}\Url}\fi

\bibitem[Albrecht et~al.(2023)Albrecht, Christianos, and Sch{\"a}fer]{marl-book}
S.~V. Albrecht, F.~Christianos, and L.~Sch{\"a}fer.
\newblock {Multi-agent Reinforcement Learning: Foundations and Modern Approaches}.
\newblock \emph{{Massachusetts Institute of Technology: Cambridge, MA, USA}}, 2023.

\bibitem[Alshiekh et~al.(2018)Alshiekh, Bloem, Ehlers, K{\"o}nighofer, Niekum, and Topcu]{alshiekh2018safe}
M.~Alshiekh, R.~Bloem, R.~Ehlers, B.~K{\"o}nighofer, S.~Niekum, and U.~Topcu.
\newblock {Safe Reinforcement Learning via Shielding}.
\newblock In \emph{{Proceedings of the AAAI Conference on Artificial Intelligence}}, volume~32, 2018.

\bibitem[Axelrod(1980)]{axelrod1980effective}
R.~Axelrod.
\newblock {Effective Choice in the Prisoner's Dilemma}.
\newblock \emph{{Journal of Conflict Resolution}}, 24\penalty0 (1):\penalty0 3--25, 1980.

\bibitem[Banerjee et~al.(2024)Banerjee, Rahmani, Biswas, and Dillig]{banerjee2024dynamic}
A.~Banerjee, K.~Rahmani, J.~Biswas, and I.~Dillig.
\newblock {Dynamic Model Predictive Shielding for Provably Safe Reinforcement Learning}.
\newblock \emph{{Advances in Neural Information Processing Systems}}, 37:\penalty0 100131--100159, 2024.

\bibitem[Bloem et~al.(2015)Bloem, K{\"o}nighofer, K{\"o}nighofer, and Wang]{bloem2015shield}
R.~Bloem, B.~K{\"o}nighofer, R.~K{\"o}nighofer, and C.~Wang.
\newblock {Shield Synthesis: Runtime Enforcement for Reactive Systems}.
\newblock In \emph{{International Conference on Tools and Algorithms for the Construction and Analysis of Systems}}, pages 533--548. Springer, 2015.

\bibitem[Brambilla et~al.(2013)Brambilla, Ferrante, Birattari, and Dorigo]{brambilla2013swarm}
M.~Brambilla, E.~Ferrante, M.~Birattari, and M.~Dorigo.
\newblock {Swarm Robotics: A Review from the Swarm Engineering Perspective}.
\newblock \emph{{Swarm Intelligence}}, 7:\penalty0 1--41, 2013.

\bibitem[Brockman et~al.(2016)Brockman, Cheung, Pettersson, Schneider, Schulman, Tang, and Zaremba]{brockman1606.01540}
G.~Brockman, V.~Cheung, L.~Pettersson, J.~Schneider, J.~Schulman, J.~Tang, and W.~Zaremba.
\newblock {OpenAI Gym}, 2016.

\bibitem[Carr et~al.(2023)Carr, Jansen, Junges, and Topcu]{carr2023safe}
S.~Carr, N.~Jansen, S.~Junges, and U.~Topcu.
\newblock {Safe Reinforcement Learning via Shielding Under Partial Observability}.
\newblock In \emph{Proceedings of the AAAI Conference on Artificial Intelligence}, volume~37, pages 14748--14756, 2023.

\bibitem[Cozman and Mau{\'a}(2017)]{cozman2017semantics}
F.~G. Cozman and D.~D. Mau{\'a}.
\newblock {On the Semantics and Complexity of Probabilistic Logic Programs}.
\newblock \emph{{Journal of Artificial Intelligence Research}}, 60:\penalty0 221--262, 2017.

\bibitem[De~Raedt et~al.(2007)De~Raedt, Kimmig, and Toivonen]{de2007problog}
L.~De~Raedt, A.~Kimmig, and H.~Toivonen.
\newblock {Problog: A Probabilistic Prolog and Its Application in Link Discovery}.
\newblock In \emph{{IJCAI}}, volume~7, pages 2462--2467, 2007.

\bibitem[De~Witt et~al.(2020)De~Witt, Gupta, Makoviichuk, Makoviychuk, Torr, Sun, and Whiteson]{de2020independent}
C.~S. De~Witt, T.~Gupta, D.~Makoviichuk, V.~Makoviychuk, P.~H. Torr, M.~Sun, and S.~Whiteson.
\newblock {Is Independent Learning All You Need in the Starcraft Multi-agent Challenge?}
\newblock \emph{arXiv preprint arXiv:2011.09533}, 2020.

\bibitem[Den~Hengst et~al.(2022)Den~Hengst, Fran{\c{c}}ois-Lavet, Hoogendoorn, and van Harmelen]{den2022planning}
F.~Den~Hengst, V.~Fran{\c{c}}ois-Lavet, M.~Hoogendoorn, and F.~van Harmelen.
\newblock {Planning for Potential: Efficient Safe Reinforcement Learning}.
\newblock \emph{{Machine Learning}}, 111\penalty0 (6):\penalty0 2255--2274, 2022.

\bibitem[ElSayed-Aly et~al.(2021)ElSayed-Aly, Bharadwaj, Amato, Ehlers, Topcu, and Feng]{elsayed2021safe}
I.~ElSayed-Aly, S.~Bharadwaj, C.~Amato, R.~Ehlers, U.~Topcu, and L.~Feng.
\newblock {Safe Multi-Agent Reinforcement Learning via Shielding}.
\newblock In \emph{{Proceedings of the 20th International Conference on Autonomous Agents and MultiAgent Systems}}, pages 483--491, 2021.

\bibitem[Fierens et~al.(2015)Fierens, Van~den Broeck, Renkens, Shterionov, Gutmann, Thon, Janssens, and De~Raedt]{fierens2015inference}
D.~Fierens, G.~Van~den Broeck, J.~Renkens, D.~Shterionov, B.~Gutmann, I.~Thon, G.~Janssens, and L.~De~Raedt.
\newblock {Inference and Learning in Probabilistic Logic Programs using Weighted Boolean Formulas}.
\newblock \emph{{Theory and Practice of Logic Programming}}, 15\penalty0 (3):\penalty0 358--401, 2015.

\bibitem[Ganesh et~al.(2019)Ganesh, Vadori, Xu, Zheng, Reddy, and Veloso]{ganesh2019reinforcement}
S.~Ganesh, N.~Vadori, M.~Xu, H.~Zheng, P.~Reddy, and M.~Veloso.
\newblock {Reinforcement Learning for Market Making in a Multi-Agent Dealer Market}.
\newblock \emph{{arXiv preprint arXiv:1911.05892}}, 2019.

\bibitem[Gu et~al.(2023)Gu, Kuba, Chen, Du, Yang, Knoll, and Yang]{gu2023safe}
S.~Gu, J.~G. Kuba, Y.~Chen, Y.~Du, L.~Yang, A.~Knoll, and Y.~Yang.
\newblock {Safe Multi-Agent Reinforcement Learning for Multi-Robot Control}.
\newblock \emph{{Artificial Intelligence}}, 319:\penalty0 103905, 2023.

\bibitem[Gu et~al.(2024)Gu, Yang, Du, Chen, Walter, Wang, and Knoll]{gu2024review}
S.~Gu, L.~Yang, Y.~Du, G.~Chen, F.~Walter, J.~Wang, and A.~Knoll.
\newblock {A Review of Safe Reinforcement Learning: Methods, Theories and Applications}.
\newblock \emph{IEEE Transactions on Pattern Analysis and Machine Intelligence}, 2024.

\bibitem[Gupta et~al.(2017)Gupta, Egorov, and Kochenderfer]{gupta2017cooperative}
J.~K. Gupta, M.~Egorov, and M.~Kochenderfer.
\newblock {Cooperative Multi-Agent Control using Deep Reinforcement Learning}.
\newblock In \emph{{International Conference on Autonomous Agents and Multiagent Systems}}, pages 66--83. Springer, 2017.

\bibitem[Hessel et~al.(2018)Hessel, Modayil, Van~Hasselt, Schaul, Ostrovski, Dabney, Horgan, Piot, Azar, and Silver]{hessel2018rainbow}
M.~Hessel, J.~Modayil, H.~Van~Hasselt, T.~Schaul, G.~Ostrovski, W.~Dabney, D.~Horgan, B.~Piot, M.~Azar, and D.~Silver.
\newblock {Rainbow: Combining Improvements in Deep Reinforcement Learning}.
\newblock In \emph{{Proceedings of the AAAI Conference on Artificial Intelligence}}, volume~32, 2018.

\bibitem[Hunt et~al.(2021)Hunt, Fulton, Magliacane, Hoang, Das, and Solar-Lezama]{hunt2021verifiably}
N.~Hunt, N.~Fulton, S.~Magliacane, T.~N. Hoang, S.~Das, and A.~Solar-Lezama.
\newblock {Verifiably Safe Exploration for End-to-End Reinforcement Learning}.
\newblock In \emph{{Proceedings of the 24th International Conference on Hybrid Systems: Computation and Control}}, pages 1--11, 2021.

\bibitem[Jansen et~al.(2020)Jansen, K{\"o}nighofer, Junges, Serban, and Bloem]{jansen2020safe}
N.~Jansen, B.~K{\"o}nighofer, S.~Junges, A.~Serban, and R.~Bloem.
\newblock {Safe Reinforcement Learning using Probabilistic Shields}.
\newblock In \emph{{31st International Conference on Concurrency Theory (CONCUR 2020)}}. Schloss-Dagstuhl-Leibniz Zentrum f{\"u}r Informatik, 2020.

\bibitem[jemaw(2019)]{jemaw2019}
jemaw.
\newblock {gym-safety}, July 2019.
\newblock URL \url{https://github.com/jemaw/gym-safety}.
\newblock commit hash: 468831a2bae112454a8954ab702198d0c69ff50f.

\bibitem[Kingma and Ba(2014)]{kingma2014adam}
D.~P. Kingma and J.~Ba.
\newblock {Adam: A Method for Stochastic Optimization}.
\newblock \emph{{arXiv preprint arXiv:1412.6980}}, 2014.

\bibitem[Leyton-Brown and Shoham(2022)]{leyton2022essentials}
K.~Leyton-Brown and Y.~Shoham.
\newblock \emph{{Essentials of Game Theory: A Concise Multidisciplinary Introduction}}.
\newblock {Springer Nature}, 2022.

\bibitem[Maschler et~al.(2020)Maschler, Zamir, and Solan]{maschler2020game}
M.~Maschler, S.~Zamir, and E.~Solan.
\newblock \emph{{Game Theory}}.
\newblock Cambridge University Press, 2020.

\bibitem[Melcer et~al.(2022)Melcer, Amato, and Tripakis]{melcer2022shield}
D.~Melcer, C.~Amato, and S.~Tripakis.
\newblock {Shield Decentralization for Safe Multi-Agent Reinforcement Learning}.
\newblock \emph{{Advances in Neural Information Processing Systems}}, 35:\penalty0 13367--13379, 2022.

\bibitem[Mnih et~al.(2013)Mnih, Kavukcuoglu, Silver, Graves, Antonoglou, Wierstra, and Riedmiller]{mnih2013playing}
V.~Mnih, K.~Kavukcuoglu, D.~Silver, A.~Graves, I.~Antonoglou, D.~Wierstra, and M.~Riedmiller.
\newblock {Playing Atari with Deep Reinforcement Learning}.
\newblock \emph{arXiv preprint arXiv:1312.5602}, 2013.

\bibitem[Orzan et~al.(2024)Orzan, Acar, Grossi, and R\u{a}dulescu]{orzanacar}
N.~Orzan, E.~Acar, D.~Grossi, and R.~R\u{a}dulescu.
\newblock {Emergent Cooperation under Uncertain Incentive Alignment}.
\newblock In \emph{Proceedings of the 23rd International Conference on Autonomous Agents and Multiagent Systems}, AAMAS '24, page 1521–1530, Richland, SC, 2024. International Foundation for Autonomous Agents and Multiagent Systems.
\newblock ISBN 9798400704864.

\bibitem[Paszke et~al.(2019)Paszke, Gross, Massa, Lerer, Bradbury, Chanan, Killeen, Lin, Gimelshein, Antiga, et~al.]{paszke2019pytorch}
A.~Paszke, S.~Gross, F.~Massa, A.~Lerer, J.~Bradbury, G.~Chanan, T.~Killeen, Z.~Lin, N.~Gimelshein, L.~Antiga, et~al.
\newblock {PyTorch: An Imperative Style, High-Performance Deep Learning Library}.
\newblock \emph{{Advances in Neural Information Processing Systems}}, 32, 2019.

\bibitem[Peysakhovich and Lerer(2018)]{peysakhovich2018prosocial}
A.~Peysakhovich and A.~Lerer.
\newblock {Prosocial Learning Agents Solve Generalized Stag Hunts Better than Selfish Ones}.
\newblock In \emph{{Proceedings of the 17th International Conference on Autonomous Agents and MultiAgent Systems}}, AAMAS '18, page 2043–2044. {International Foundation for Autonomous Agents and Multiagent Systems}, 2018.

\bibitem[Prakash(2021)]{prakash2021}
A.~Prakash.
\newblock {pz\_dilemma}, November 2021.
\newblock URL \url{https://github.com/arjun-prakash/pz_dilemma}.
\newblock commit hash: 7df435f83cf0917359ecc8847dd8b474df18fa5c.

\bibitem[Riek(2017)]{riek2017healthcare}
L.~D. Riek.
\newblock {Healthcare Robotics}.
\newblock \emph{{Communications of the ACM}}, 60\penalty0 (11):\penalty0 68--78, 2017.

\bibitem[Schulman et~al.(2017)Schulman, Wolski, Dhariwal, Radford, and Klimov]{schulman2017proximal}
J.~Schulman, F.~Wolski, P.~Dhariwal, A.~Radford, and O.~Klimov.
\newblock {Proximal Policy Optimization Algorithms}.
\newblock \emph{arXiv preprint arXiv:1707.06347}, 2017.

\bibitem[Shalev-Shwartz et~al.(2016)Shalev-Shwartz, Shammah, and Shashua]{shalev2016safe}
S.~Shalev-Shwartz, S.~Shammah, and A.~Shashua.
\newblock {Safe, Multi-Agent, Reinforcement Learning for Autonomous Driving}.
\newblock \emph{{arXiv preprint arXiv:1610.03295}}, 2016.

\bibitem[Subramanian et~al.(2024)Subramanian, Liu, Khan, Lenchner, Amarnath, Swaminathan, Riegel, and Gray]{subramanian2024neuro}
C.~Subramanian, M.~Liu, N.~Khan, J.~Lenchner, A.~Amarnath, S.~Swaminathan, R.~Riegel, and A.~Gray.
\newblock {A Neuro-Symbolic Approach to Multi-Agent RL for Interpretability and Probabilistic Decision Making}.
\newblock \emph{arXiv preprint arXiv:2402.13440}, 2024.

\bibitem[Sutton and Barto(2018)]{sutton2018reinforcement}
R.~S. Sutton and A.~G. Barto.
\newblock \emph{{Reinforcement Learning: An Introduction}}.
\newblock MIT Press, 2018.

\bibitem[Tan(1993)]{tan1993multi}
M.~Tan.
\newblock {Multi-agent Reinforcement Learning: Independent vs. Cooperative agents}.
\newblock In \emph{{Proceedings of the Tenth International Conference on Machine Learning}}, pages 330--337, 1993.

\bibitem[Terry et~al.(2021)Terry, Black, Grammel, Jayakumar, Hari, Sullivan, Santos, Dieffendahl, Horsch, Perez-Vicente, et~al.]{terry2021pettingzoo}
J.~Terry, B.~Black, N.~Grammel, M.~Jayakumar, A.~Hari, R.~Sullivan, L.~S. Santos, C.~Dieffendahl, C.~Horsch, R.~Perez-Vicente, et~al.
\newblock {PettingZoo: Gym for Multi-Agent Reinforcement Learning}.
\newblock \emph{{Advances in Neural Information Processing Systems}}, 34:\penalty0 15032--15043, 2021.

\bibitem[Tsitsiklis and Van~Roy(1996)]{tsitsiklis1996analysis}
J.~Tsitsiklis and B.~Van~Roy.
\newblock {Analysis of Temporal-Difference Learning with Function Approximation}.
\newblock \emph{{Advances in Neural Information Processing Systems}}, 9, 1996.

\bibitem[van~der Sar et~al.(2023)van~der Sar, Zocca, and Bhulai]{van2023multi}
E.~van~der Sar, A.~Zocca, and S.~Bhulai.
\newblock {Multi-Agent Reinforcement Learning for Power Grid Topology Optimization}.
\newblock \emph{{arXiv preprint arXiv:2310.02605}}, 2023.

\bibitem[Waga et~al.(2022)Waga, Castellano, Pruekprasert, Klikovits, Takisaka, and Hasuo]{waga2022dynamic}
M.~Waga, E.~Castellano, S.~Pruekprasert, S.~Klikovits, T.~Takisaka, and I.~Hasuo.
\newblock {Dynamic Shielding for Reinforcement Learning in Black-box Environments}.
\newblock In \emph{{International Symposium on Automated Technology for Verification and Analysis}}, pages 25--41. Springer, 2022.

\bibitem[Wang et~al.(2024)Wang, Fang, Tomilin, Fang, and Du]{wang2024safe}
Z.~Wang, M.~Fang, T.~Tomilin, F.~Fang, and Y.~Du.
\newblock {Safe Multi-Agent Reinforcement Learning with Natural Language Constraints}.
\newblock \emph{arXiv preprint arXiv:2405.20018}, 2024.

\bibitem[Watkins and Dayan(1992)]{watkins1992q}
C.~J. Watkins and P.~Dayan.
\newblock Q-learning.
\newblock \emph{{Machine Learning}}, 8:\penalty0 279--292, 1992.

\bibitem[Welford(1962)]{welford1962note}
B.~P. Welford.
\newblock {Note on a Method for Calculating Corrected Sums of Squares and Products}.
\newblock \emph{{Technometrics}}, 4\penalty0 (3):\penalty0 419--420, 1962.

\bibitem[Xiao et~al.(2023)Xiao, Lyu, and Dolan]{xiao2023model}
W.~Xiao, Y.~Lyu, and J.~Dolan.
\newblock {Model-based Dynamic Shielding for Safe and Efficient Multi-agent Reinforcement Learning}.
\newblock In \emph{{Proceedings of the 2023 International Conference on Autonomous Agents and Multiagent Systems}}, pages 1587--1596, 2023.

\bibitem[Yang et~al.(2023)Yang, Marra, and De~Raedt]{yang2023safe}
W.~Yang, G.~Marra, and L.~De~Raedt.
\newblock {Safe Reinforcement Learning via Probabilistic Logic Shields}.
\newblock In \emph{{International Joint Conference on Artificial Intelligence (IJCAI)}}, pages 5739--5749, 2023.

\bibitem[Zhang et~al.(2019)Zhang, Bastani, and Kumar]{zhang2019mamps}
W.~Zhang, O.~Bastani, and V.~Kumar.
\newblock {MAMPS: Safe Multi-agent Reinforcement Learning via Model Predictive Shielding}.
\newblock \emph{arXiv preprint arXiv:1910.12639}, 2019.

\bibitem[Zhao et~al.(2016)Zhao, Wang, Shao, and Zhu]{zhao2016deep}
D.~Zhao, H.~Wang, K.~Shao, and Y.~Zhu.
\newblock {Deep Reinforcement Learning with Experience Replay Based on SARSA}.
\newblock In \emph{{2016 IEEE Symposium Series on Computational Intelligence (SSCI)}}, pages 1--6. IEEE, 2016.

\end{thebibliography}

% \clearpage\newpage
\appendix

\section*{Supplementary Material for \textit{Analyzing Probabilistic Logic Shields for Multi-Agent Reinforcement Learning}}
The supplementary material is organized in the following way:
\begin{itemize} [Appendix]
    \item \ref{appsec:proofs}: Theoretical proofs of propositions.
    \item \ref{appsec:envs}: In-depth discussion of environments used, including implementation details, reward functions and settings. 
    \item \ref{appsec:shields}: Details on shields including their ProbLog source codes and process of construction
    \item \ref{appsec:cartsafe_results}: Single-agent experiment used to as an initial indication of PLTD's safety and efficacy.
    \item \ref{appsec:hyperparams}: Table and explanation of all hyperparameters used.
    \item \ref{appsec:training_curves}: Partial training/reward curves.
\end{itemize}
% \onecolumn % optional
\section{Proofs}\label{appsec:proofs}
\subsection{Proof of Proposition~3.2} \label{app:safety}
\begin{proposition}
    A PL-SMARL algorithm with agents $i\in N$ with respective policies $\pi_i^+$ and shields $\mathcal{T}_i$ has a joint shielded policy $\vec\pi^+$ \textbf{at least as safe} as their base joint policy $\vec\pi$, and \textbf{strictly safer} whenever any individual base policy $\pi_j$ violates its shield $\mathcal{T_j}$.
\end{proposition}
\begin{proof}
    Let $\Pi_i$ be the space of all policies for each agent $i\in N$. We then recall from \citet{yang2023safe}'s Proposition~1 that $P_{\pi^+_i}(\texttt{safe} | s) \geq P_{\pi_i}(\texttt{safe} | s)$, $\forall \pi_i \in \Pi_i \text{ and } \forall s \in S$. Thus, according to the safety specifications in $\mathcal{T}_i$, every $\pi_i^+$ is \textbf{at least as safe as} $\pi_i$, and thus $\vec{\pi}^+$ is \textbf{at least as safe as} $\vec{\pi}$. 
    
    If all base policies $\pi_i\in\vec{\pi}$ already satisfy their respective shields $\mathcal{T}_i$ in all reachable states (i.e., $P_{\pi_i}(safe | s) = 1$ for all $s$), then shielding leaves the policy unchanged: $\pi_i^+ = \pi_i$. Otherwise, if any $\pi_j\in\vec{\pi}$ does not abide by $\mathcal{T}_j$, by construction, the shielded policy will assign higher probability to safe actions in expectation, i.e.\ $P_{\pi^+_j}(\texttt{safe} | s) > P_{\pi_j}(\texttt{safe} | s)$, and the resulting joint policy $\vec{\pi}^+$ will be \textbf{strictly safer} than $\vec{\pi}$.
\end{proof}

%%%%%%%%%%%%%%%%%%%%%%%%%%%%%%%%%%%%%%%%%%%%%%%%%%%%%%%%%%%%%%%%%%%%%%%%%%%%%%%
\subsection{Proof of Proposition~3.4}
First, for completion, we prove a lemma that states that any shield with perfect information (binary information about the valuations of the \verb|sensor(.)| predicates) always yield $P_{\mathcal{T}}(\text{safe}\mid s,a)\in\{0,1\}$. Then we define a more specific form of Proposition~3.4 and show that PLTD converges under perfect sensor information and regular TD-learning assumptions.
\begin{lemma}[Perfect-information shield yields binary safety]
\label{lem:perfect_info}
Let $\mathcal{T}$ be a ProbLog shield that, at a given state $s$, receives action facts \texttt{action($a$)} whose probabilities are the policy $\pi(a\mid s)$, and sensor facts \texttt{sensor($\ell$)} that are \emph{deterministic} in~$s$ (i.e. we have perfect sensor information), i.e., $P_\mathcal{T}\!\bigl(\texttt{sensor}(\ell)=\mathrm{true}\bigr)\in\{0,1\}$ and   $\texttt{sensor}(\ell)=\mathrm{true}\iff
  \ell\ \text{actually holds in }s.$

\noindent Suppose $\mathcal{T}$’s knowledge base contains, for each ground
action~$a$, a clause of the schematic form
\begin{equation}
  \texttt{unsafe\_next} \;:\!-\;
      \texttt{action}(a),\,\texttt{sensor}(\ell_a).    
\end{equation}

\noindent where the literal $\ell_a$ may be one of:
\begin{enumerate}[i.]
  \setlength{\itemsep}{0pt}
  \setlength{\parskip}{0pt}
  \item a state-dependent `hazard' literal, or  
  \item an always-true dummy literal (i.e.\ \texttt{unsafe\_next :- action($a$).}).
\end{enumerate}
Together with the complementary rule $$\texttt{safe\_next} \;:\!-\; \lnot\texttt{unsafe\_next},$$the shield satisfies  
\begin{equation}
  P_{\mathcal{T}}\!
  \bigl(\texttt{safe\_next}\mid \texttt{action}(a)\bigr)
  \;\in\;\{0,1\}\qquad\text{for every }a.
\end{equation}

\noindent Define the \emph{state-wise safe-action set}
\begin{equation}
  A_{\mathrm{safe}}(s)\;=\;
  \bigl\{\,a\mid
        P_{\mathcal{T}}(\texttt{safe\_next}\mid\texttt{action}(a))=1
  \bigr\}.
\end{equation}
Then
\begin{equation}
  P_{\mathcal{T}}(\text{safe}\mid s,a)=
  \begin{cases}
    1,& a\in A_{\mathrm{safe}}(s);\\[4pt]
    0,& a\notin A_{\mathrm{safe}}(s).
  \end{cases}
\end{equation}
\end{lemma}

\begin{proof}
Condition the ground program on \texttt{action($a$)}. All remaining random facts are the deterministic \texttt{sensor($\ell$)} atoms, so the program is now fully deterministic. Hence the queried probability is $1$ if and only if the rule for \texttt{unsafe\_next} fails.

\smallskip\noindent
\textit{Case (a)}.
If $\ell_a$ is false in $s$ the rule body fails, so  \texttt{safe\_next} holds with certainty.

\smallskip\noindent
\textit{Case (b).}  
If $\ell_a$ is always true, the rule body succeeds, so \texttt{safe\_next} is false with certainty.

Thus the conditional probability is always $0$ or~$1$, and partitioning actions by the value $1$ yields the stated set $A_{\mathrm{safe}}(s)$.
\end{proof}
% ---------------------------------------------------------

\begin{proposition}[PLTD Convergence under Perfect Safety Information]
Let $\mathcal{M}= \langle S,A,P,r,\gamma\rangle$ be a finite MDP with $\gamma<1$ and let $\mathcal{T}$ be a PL shield with perfect sensor information. Assume the agent
\begin{enumerate}
  \item[i.] [GLIE] follows a time-varying policy $\pi_t$ such that (a) $\pi_t(a\mid s)=0$ for all $a\notin A_{\text{safe}}(s)$, and (b) $\pi_t\xrightarrow{t\to\infty}\pi^{+}$ while every safe pair $(s,a)$ is selected infinitely often (e.g.\ $\epsilon_t$-greedy on $A_{\text{safe}}(s)$ with $\epsilon_t\!\downarrow0$, $\sum_t \epsilon_t=\infty$).
  \item[ii.] stores a \emph{tabular} action–value function $Q_\theta$, and
  \item[iii.] updates $\theta$ by minimising the PLTD loss (Eq.~3.3) with a stepsize sequence $\{\eta_t\}$ satisfying $\sum_t\eta_t = \infty$ and $\sum_t\eta_t^{2} < \infty$ (regular TD-learning assumption, \citet{sutton2018reinforcement}).
\end{enumerate}

Then, with probability~$1$, the parameter sequence $\{\theta_t\}$ converges to a fixed point $Q^{\star}_{\mathrm{safe}}$ that is optimal over the set of safe policies, i.e.
\begin{small}
\begin{equation}
\begin{aligned}
  Q^{\star}_{\mathrm{safe}}(s,a)=
  &\max_{\pi:\,\pi(a'\!\mid s)=0\ \forall\,a'\notin A_{\mathrm{safe}}(s)}
\\[2pt]
  &\mathbb{E}_{\pi}\!\Bigl[\textstyle\sum_{k=0}^{\infty}
      \gamma^{k} r_{t+k}\,\bigm|\,s_t=s,\; a_t=a\Bigr].
\end{aligned}
\end{equation}
\end{small}
\end{proposition}

% \begin{proof}[Sketch]
% \textbf{1. Penalty vanishes.} By Lemma~\ref{lem:perfect_info}, a shield with deterministic (perfect) sensors produces binary safety probabilities. Hence we may treat $\mathcal T$ as a \emph{perfect-information shield}, i.e.\ $P(\text{safe}\mid s,a)=1$ if $a\in A_{\text{safe}}(s)$ and $0$ otherwise, where $A_{\text{safe}}(s)=\{a\mid P(\text{safe}\mid s,a)=1\}$. Under this shield $P_{\pi^{+}}(\text{safe}\mid s)=1$ in every visited state, so the safety-penalty term $\log P_{\pi^{+}}(\text{safe}\mid s)$ in Eq.~(3) equals~$0$. Hence the PLTD loss reduces to the ordinary squared TD error.

% \medskip
% \noindent
% \textbf{2. Behaviour/target match.} Because the agent both \emph{acts} with and \emph{updates toward} $\pi^{+}$, the distribution required for TD convergence (no off-policy mismatch) holds.

% \medskip
% \noindent
% \textbf{3. Standard TD theory.} With a finite MDP, tabular $Q_\theta$, bounded rewards, and Robbins–Monro stepsizes, classical results for Q-learning (off-policy) and SARSA (on-policy) \citep[e.g.][]{watkins1992q,sutton2018reinforcement} guarantee almost-sure convergence to the action-value function of the behaviour policy, here $\pi^{+}$.

% \medskip
% \noindent
% \textbf{4. Optimality over safe policies.} Because $\pi^{+}$ never selects unsafe actions, the resulting limit $Q^{\star}_{\mathrm{safe}}$ is optimal over the admissible (safe) policy set, completing the proof.
% \end{proof}

\begin{proof}
Fix a finite MDP $\mathcal{M}=\langle S,A,P,r,\gamma\rangle$ with
$\gamma\!<\!1$ and a ProbLog shield~$\mathcal{T}$ whose sensor atoms are
deterministic in the current state~$s$; by
Lemma~\ref{lem:perfect_info} the shield therefore induces \emph{binary}
state–action safety probabilities.  Throughout, let
$A_{\text{safe}}(s)\subseteq A$ denote the set of $s$-safe actions singled out
in the lemma.

\paragraph{1.  The PLTD loss collapses to the ordinary TD loss.}
For any state–action pair $(s,a)$,
\begin{equation}
P_{\mathcal{T}}\bigl(\text{safe}\mid s,a\bigr)
  \;=\;\mathbf{1}\!\{a\in A_{\text{safe}}(s)\}.
\end{equation}
Consequently $P_{\pi^+}(\text{safe}\mid s)=1$ for every $s$ that can be
encountered under the \emph{behaviour policy} $\pi^+$ (assumption~i).
Thus the safety penalty
$S_P=\log P_{\pi^+}(\text{safe}\mid s)$ in
Eq.~3.3 vanishes identically and the PLTD objective reduces to
\begin{equation}
  \mathcal{L}^{Q^+}(\theta)
   \;=\;\mathbb{E}_{d\sim\mathcal D}\!\bigl[\!
     \bigl(r^t+\gamma\mathcal X-Q_\theta(s^t,a^t)\bigr)^2\bigr],
\end{equation}
namely the usual squared one-step TD error (with
$\mathcal X$ equal to $\max_{a'}Q_\theta(s^{t+1},a')$ in the off-policy
variant or $Q_\theta(s^{t+1},a^{t+1})$ in the on-policy variant).

\paragraph{2.  Behaviour–target consistency.}
The agent \emph{acts} with~$\pi^+$ \emph{and} evaluates TD targets under
$\pi^+$, so there is no off-policy mismatch; the TD recursion is
on-policy in the SARSA setting and uses the same transition-probability
matrix in the Q-learning setting.

\paragraph{3.  Almost-sure convergence of the tabular TD iterates.}
Because $S$, $A$ and therefore the post-shield state–action space
$
  \{(s,a)\mid a\in A_{\text{safe}}(s)\}
$
are finite, rewards are bounded, and the stepsizes
$\{\eta_t\}$ satisfy $\sum_t\eta_t=\infty$ and
$\sum_t\eta_t^{2}<\infty$ (assumption~iii), the classical stochastic-approximation proofs of \citet{watkins1992q} for Q-learning
or \citet{sutton2018reinforcement} for SARSA apply \emph{verbatim}.  A
minor notational change is that the maximisation (in the off-policy
case) is taken only over safe actions, but the contraction argument is
identical.  Under the usual further condition that every safe
state–action pair is visited infinitely often (assumption i) the parameter sequence
$\{\theta_t\}$ converges with probability 1 to the unique fixed point
$Q^{\pi^+}$ of the Bellman operator corresponding to~$\pi^+$.

\paragraph{4.  Optimality of the limit among safe policies.}
By construction $\pi^+$ places zero probability on unsafe actions and is
\emph{greedy} with respect to $Q^{\pi^+}$ over each safe-action set
$A_{\text{safe}}(s)$:
\begin{equation}
  \pi^+(a\mid s)>0\;\Longrightarrow\;
   a\in\operatornamewithlimits{argmax}_{a'\in A_{\text{safe}}(s)}Q^{\pi^+}(s,a').
\end{equation}
Standard policy-evaluation/-improvement arguments (e.g.\ \citet{sutton2018reinforcement}) now imply that
$Q^{\pi^+}$ coincides with the \emph{optimal} action-value function over
the restricted policy class
$
  \Pi_{\text{safe}}\;=\;\{\pi\mid
    \pi(a'\mid s)=0\text{ whenever }a'\notin A_{\text{safe}}(s)\}.
$
We denote this limit by $Q^{\star}_{\text{safe}}$ and $\forall (s,a)$ obtain
\begin{equation}
  Q^{\star}_{\text{safe}}(s,a)
  \;=\;\max_{\pi\in\Pi_{\text{safe}}}
    \mathbb{E}_{\pi}\Bigl[\sum_{k=0}^{\infty}\gamma^{k}r_{t+k}
      \,\bigm|\,s_t=s,\;a_t=a\Bigr],
\end{equation}

\paragraph{5.  Putting it together.}
Steps~1–4 establish that
$Q_\theta$ converges to $Q^{\star}_{\text{safe}}$ as claimed.

\vspace{-\baselineskip}
\end{proof}

\begin{remark} As stated in the main paper, for tabular and linear function-approximation settings, Lemma \ref{lem:perfect_info} together with the projected TD results of \citet{tsitsiklis1996analysis} yields convergence of PLTD. For deep neural Q-networks, global convergence is currently not known; in that regime we can adopt the usual stabilisation tricks (target network, replay buffer, Double-DQN). Empirically, we can show that PLTD remains stable across tasks, but a complete theoretical treatment is left for future work.
\end{remark}

\subsection{Proof of Proposition~3.5} \label{app:propproof}

\begin{proposition}
    For any game there exists a shield $\mathcal{T}$ such that the resulting shielded agent deterministically selects a single action $a\in\mathcal{A}$ for each state $s\in S$ if the shielded policy $\pi^+$ computes all other actions $a'\in\mathcal{A}\setminus \{a\}$ as perfectly unsafe.
\end{proposition}
% \begin{lemma}
%     For a game with perfect safety information, an agent will have a pure strategy $a$ if the shielded policy $\pi^+$ computes all other actions $a'\in\mathcal{A}$ as having $\pi^+(a'\mid s)=0$.
% \end{lemma}
\begin{proof}
Assume a game $\mathcal{E}$ that has actions $\mathcal{A}$ for a given agent and $|\mathcal{A}|=n_\mathcal{A}$. This agent has a policy $\pi(s)=\langle a_0,\vec{a}_{-a_0}\rangle\in\mathbb{R}^{n_\mathcal{A}}_{\geq0}$, where $\vec{a}_{-a_0}$ is the tuple of action values except $a_0$, such that $\sum\pi(s)=1$. We recall from Definition ... that for any agent with a policy $\pi$, a safe policy $\pi^+$ is computed as:
\begin{equation*}
    \pi^+(a\mid s)=P_\pi(a\mid s,\mathtt{safe})=\frac{P(\mathtt{safe}\mid s,a)}{P(\mathtt {safe}|s)}\pi(a\mid s)
\end{equation*}

We now construct a shield $\mathcal{T}$ in ProbLog. Let \verb|action(a0)| be a probabilistic fact in $\mathcal{T}$ that represents $\pi(a_0|s)$ and with constraints $\mathcal{KB}$ as follows:
\begin{lstlisting}
action(0)::action(a0);
... % optional: other actions

% optional: sensor inputs

% safety specification
unsafe_next :- action(X), X\=a0.
safe_next :- \+unsafe_next.
\end{lstlisting}

Note that the shield is state independent. Using the semantics of PL, we see that $P(\texttt{safe}| s,a')=0$ for $a'\neq a_0$, i.e.\ there is no state in which the safety of any other action apart from $a_0$ at all safe. Now, we can compute $\pi^+(a_0| s)$ and determine that $P(\texttt{safe}|s)=1-\sum\vec{a}_{-a_0}=a_0=\pi(a_0|s)$, leaving us with $\pi^+(a_0|s)=P_\pi(\texttt{safe}| s, a_0)=1$. Thus, we get a safe deterministic policy (pure strategy) $\pi^+(s)=\langle 1, 0, ...\rangle$, sampling from which, the agent will always play $a_0$.\end{proof}

% You can have as much text here as you want. The main body must be at most $8$ pages long.
% For the final version, one more page can be added.
% If you want, you can use an appendix like this one.  

% The $\mathtt{\backslash onecolumn}$ command above can be kept in place if you prefer a one-column appendix, or can be removed if you prefer a two-column appendix.  Apart from this possible change, the style (font size, spacing, margins, page numbering, etc.) should be kept the same as the main body.
%%%%%%%%%%%%%%%%%%%%%%%%%%%%%%%%%%%%%%%%%%%%%%%%%%%%%%%%%%%%%%%%%%%%%%%%%%%%%%%
\section{Environment Specifics} \label{appsec:envs}

\begin{figure}[h]
    \centering
    \includegraphics[width=0.3\linewidth]{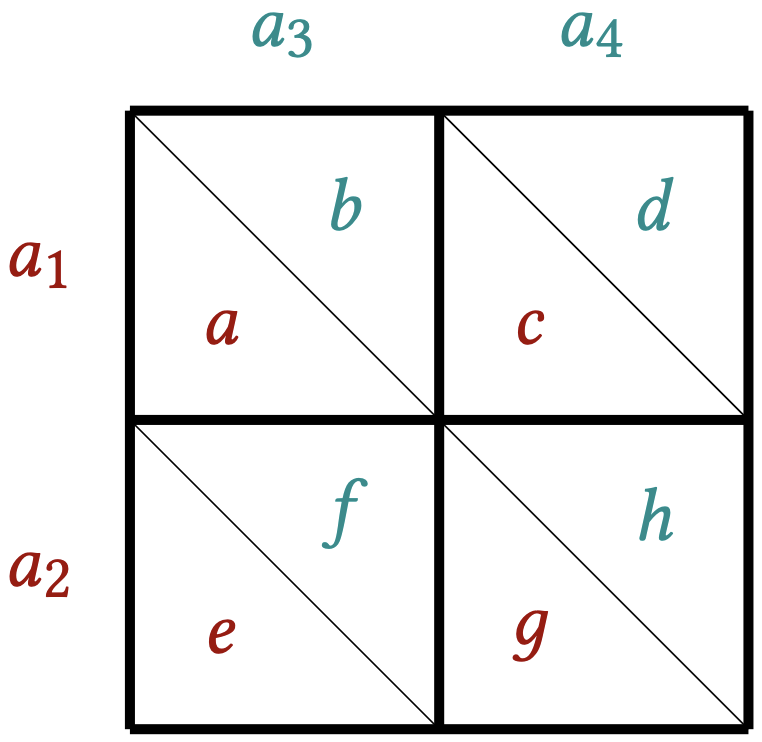}
    \caption{Generic payoff matrix}
    \label{fig:app:nfg_utilities}
\end{figure}

\subsection{Repeated NFGs (Stag-Hunt)}
In these experiments both agents choose an action simultaneously and receive a reward determined by the payoff matrix. With reference to the generic payoff matrix in \ref{fig:app:nfg_utilities}, any game with a payoff matrix with the following structure of payoffs is considered an instance of the \textit{Stag-Hunt}: $a=b>d=e\geq g=h> c=f$. Each episode is a series of 25 repeated games. Thus the agents must learn to maximize the discounted reward attained over the whole episode \citep{axelrod1980effective}. The three Nash equilibria for the \textit{Stag-Hunt} game instance as in Figure~2 (main paper) are $(Stag, Stag)$, $(Hare, Hare)$ and a mixed Nash equilibrium where each agent plays \textit{Stag} 60\% of the time and \textit{Hare} 40\% of the time.

\subsection{Centipede Game}

The centipede game is a turn-based game that has the following description: the game begins with a `pot' $p_0$ (a set amount of payoff). Each player has two actions, either to \textit{continue} or to \textit{stop}. After each round $t$, the pot increases, for example linearly ($p_{t+1}\leftarrow p_t+1$) or exponentially ($p_{t+1}\leftarrow p_t^2$). If any player decides to \textit{stop} at their turn, they receive a utility of $p_t/2+1$ and the other player gets a utility of $p_t/2-1$. If both players continue after a set number of rounds $t_{max}$, both players split the pot equally, receiving $p_{t_{max}}/2$. An instance of a centipede game with $t_{max}=6$ turns and $n=2$ players is shown in Figure~\ref{fig:cat_game}. The optimal strategy here is for the very first player to stop. This is called the \textit{subgame perfect Nash equilibrium}, one analogue of of the Nash equilibrium for extensive-form games.

\begin{figure}[h]
    \centering
    \includegraphics[width=1\linewidth]{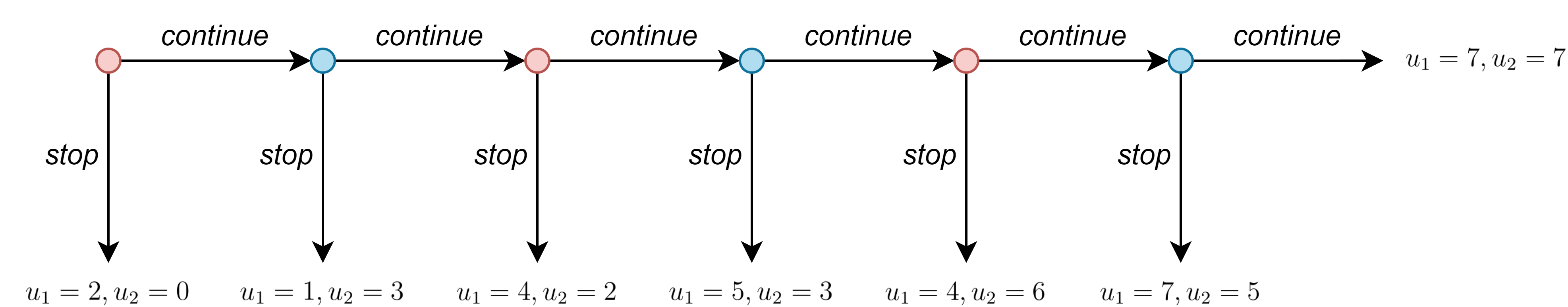}
    \caption[Example \textit{Centipede} Game Instance.]{\textit{Centipede} game for 2 players with $p_0=2$ and growth rule $p_{t+1}\leftarrow p_t+2$. Red nodes represent player 1's decision node and blue ones are player 2's. The leaves represent the utilities received by each player. The first move is at the root of the tree.}
    \label{fig:cat_game}
\end{figure}

\subsubsection{Environment description}
\citet{prakash2021} provides a good starting point with respect to implementation. However, in order to make the structure of the game homogeneous with the other environments experimented with\footnote{This is not strictly required but makes the implementation of the Python environment, the interaction with agents, and the shield more straightforward.}, a simple modification was made that turns the centipede game into a repeated simultaneous game with changing utilities each round.

This is done by combining the two agents' turns into a 2-player 2-action game, where the player who goes \rowcolor{first} is the \rowcolor{row} player and the \columncolor{second} player is the \columncolor{column} player. Let these actions be \textit{S} (for \textit{Stop}) and \textit{C} (for \textit{Continue}). If and only if both agents choose \textit{C}, a new game is created with utilities derived from the next pair of turns, with no reward being given to either player. This transformation is general and can be applied to any \textit{Centipede} game instance. Figure~\ref{fig:nfg_cat_game} shows an example of this -- a repeated NFG version of Figure~\ref{fig:cat_game}. It follows that we can also verify that the Nash solution (more specifically the subgame-perfect Nash equilibrium) is for the \rowcolor{first (row)} player to play \textit{S} right away regardless of what the \columncolor{second (column)} player plays (this is their dominant strategy).

\newcommand{\catnfgi}{\smallnfgame{S C S C $2$ $2$ $1$  $0$ $0$ $0$ $3$ $0$}}
\newcommand{\catnfgii}{\smallnfgame{S C S C $4$ $4$ $3$  $0$ $0$ $0$ $5$ $0$}}
\newcommand{\catnfgiii}{\smallnfgame{S C S C $4$ $4$ $6$ $7$ $0$ $0$ $5$ $7$}}
\begin{figure}
    \centering
    \includegraphics[width=0.9\linewidth]{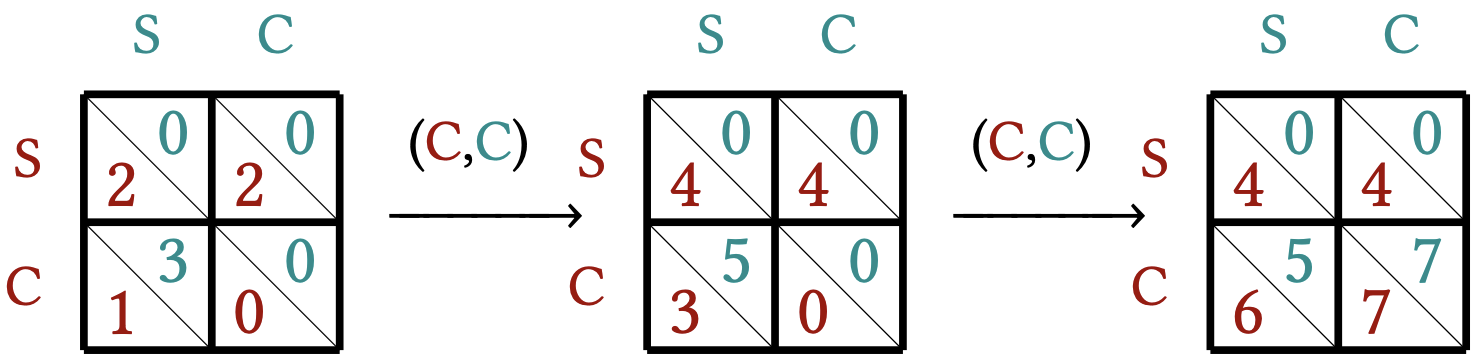}
    \caption[Repeated Simultaneous Game Form of \textit{Centipede}.]{Repeated simultaneous game form of \textit{Centipede} shown in Figure~\ref{fig:cat_game}.}
    \label{fig:nfg_cat_game}
\end{figure}

\subsubsection{Experimental conditions}
Two agents play this game. The one that has the first move is randomised at the start of each episode. The initial pot is set to $p_0=1$ and is updated for each successful \textit{Continue} action as $p_{t+1}\leftarrow p_{t}+2$. If an agent plays \textit{Stop}, they receive a reward of $p_t/2+1$ and the other receives $p_t/2-1$. The maximum number of steps are set to $t_{max}=50$, with the final rewards being equal for both agents at $p_{t_{max}}/2$.

\subsection{Two-Player Extended Public Goods Game}

The \textit{Extended Public Goods Game} is an environment where each player $i$ has an initial amount of money $c_i$ and can either choose to \textit{Cooperate} (put their money into the pot) or \textit{Defect} (keep the money). The money in the pot is multiplied by a value $f$, and the new total is distributed evenly to all players. The Nash equilibrium for a single game depends on the value of $f$, and thus the Nash strategy for each agent for a full game is not necessarily dependent on the expected value of $f$.

\begin{figure}[h]
    \centering
    \includegraphics[width=0.4\linewidth]{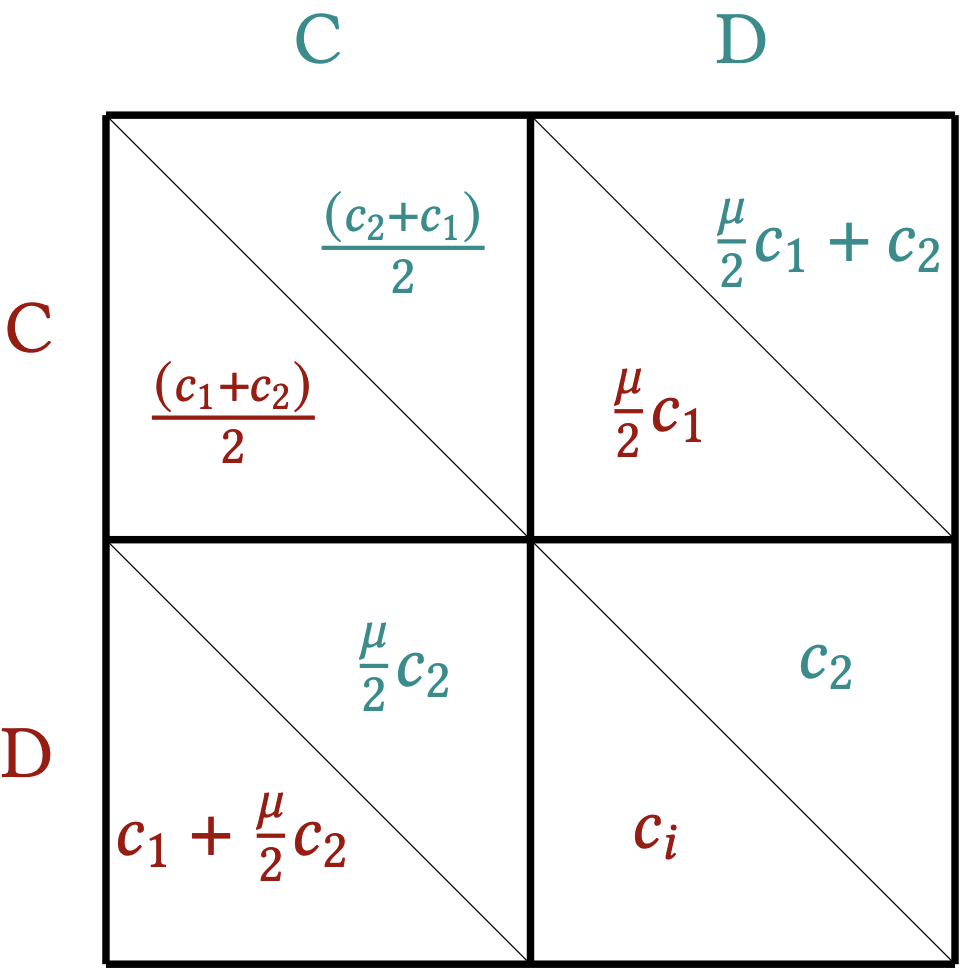}
    \hspace{0.05\linewidth}
    \includegraphics[width=0.4\linewidth]{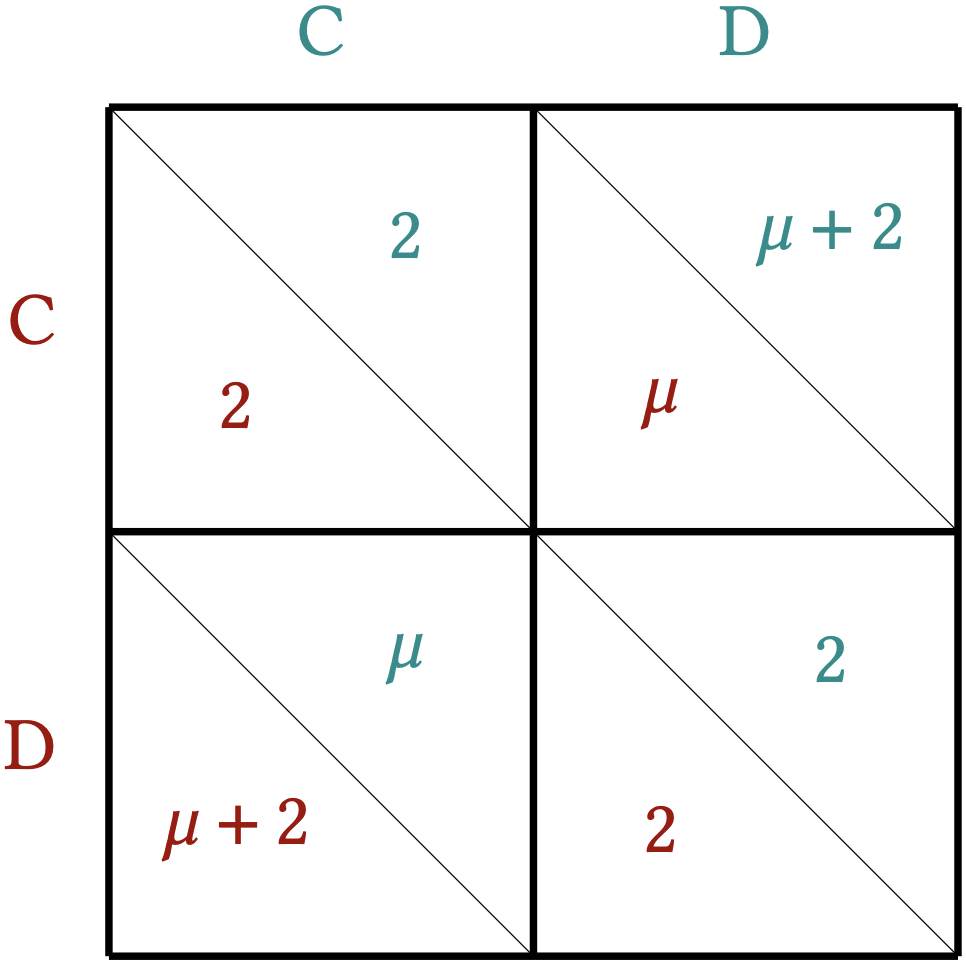}
    \caption[Expected Payoff Matrices for \textit{Extended Public Goods Games}.]{Expected payoff matrix for a two-player EPGG with general endowments $c_1$ and $c_2$ (left) and the payoff when $c_1=c_2=2$ (right). For both, $\mu=\mathbb{E}_{f\sim\mathcal{N}(\mu,\sigma)}[f]$.}
    \label{fig:expected_payoff_epgg}
\end{figure}

Assume a game with $n=2$ agents. At each $t$, the expected payoffs can be computed a-priori with respect to taking an action with the knowledge of $f$ being passed to the agents. Thus, perfectly rational agents maximize their expected reward by learning to cooperate conditioned on the instantaneous $f$ they receive. It might be interesting to try and coerce the agents to take actions based on the \textit{expected} value of $f$ instead of instantaneous ones -- indeed, if we know the distribution from which $f$ is sampled (e.g. $f\sim\mathcal{N}(\mu,\sigma)$ for some $\mu$ and $\sigma$) and we want the agents to converge to a singular action throughout the episode (i.e.\ not the time step-wise Nash equilibrium), then the expected payoff matrix can be derived as Figure~\ref{fig:expected_payoff_epgg}. Here we can see that the \textit{expected} Nash equilibrium over the population of $f$ is determined by $\mu=\mathbb{E}_{f\sim\mathcal{N}(\mu,\sigma)}[f]$ and can be estimated empirically by taking the mean of several observations of $f$.

\subsubsection{Experimental conditions}
Two agents play this game. The agents observe the action of the other agent taken at the previous time step and the current sample of the multiplication factor $f_t\sim \mathcal{N}(\mu, \sigma)$. It is expected for rational agents to defect when $f_t<1$ and cooperate when $f_t>2$ and have mixed incentives otherwise. The game is played for 500 episodes consisting of 25 games each. At each time step, a new $f_t$ is sampled. In separate experiments $\mu\in\{0.5,1.5,2.5,5.0\}$. The standard deviation of the distribution for $f_t$ is kept at $\sigma=1$ always. The initial endowment for each agent is $c_i=c_j=2$.

\subsection{Markov Stag-Hunt}
\citet{peysakhovich2018prosocial} describe a few grid environments that have the property that the optimal policies for agents reflect the structure of the \textit{Stag-Hunt} game -- principally, there are two strategies that the agents can converge to, either take on more risk by cooperating but attain higher rewards, or take on lower risk but get guaranteed small rewards. An optimal strategy may even be some mix of the two, depending on the details of the environment in question.  

One of these environments, which they term the \textit{Markov Stag-Hunt} environment has been adapted here to show that adding local probabilistic logic shields to MARL agents can help them converge to a global behavior that corresponds to either of the \textit{underlying} pure Nash equilibria of choice -- either the cooperative or the non-cooperative one. It should be stated that the original code for \textit{Markov Stag-Hunt} from \cite{peysakhovich2018prosocial} is not publicly available, and thus the environment created here is a custom approximation based on the sparse description in their paper. Thus, their results should not be directly compared ours.

\subsubsection{Environment description}
The adapted \textit{Markov Stag-Hunt} environment used within this paper can be described in general as such: $n$ agents are placed in a square grid world of size $g\times g$. The agents have the following action set at each time step to move through the grid with respect to a global direction system: \{\texttt{UP}, \texttt{DOWN}, \texttt{LEFT}, \texttt{RIGHT}, \texttt{NONE}\}. There are $n_{stag}$ stags placed in the environment, as well as $n_{plant}$ plants, each occupying a single position each on the lattice. The agents can have either local or global observations, with or without a `self' token, either as a 3D tensor or flattened one-hot encoded vector depending on the needs and configuration of the experiment.

An agent attempts to `hunt' a stag if is on the same grid position as the stag, and `harvests' a plant if it is on the same position of a plant. A stag is successfully hunted if there are $n_{hunt}\geq n_{hunt\_min}$ agents that hunt the same stag at the same time step and each receive a large positive reward $r_{stag}$, otherwise, the agents are unsuccessful and receive a penalty (negative reward) $r_{pen}$. An agent can harvest a plant alone, and will receive a small reward of $r_{plant}<r_{stag}$. Once a plant is harvested or a stag is attempted to be hunted, they reappear in another part of the grid. The stag also has an optional behaviour of moving towards the nearest agent at each time step with probability $p_{move}$. The grid exists for a certain amount of time steps $t_{max}$.

It is intuitive that if the agents learn to cooperate and attack the stag, they end up with a higher expected reward, with the risk of getting a penalty if the cooperation is inadequate. They may also learn an anti-cooperative strategy of only harvesting plants, which has a low reward but no penalty. This is a similar outcome to the NFG version of the Stag-Hunt, but in a far more complex environment. The first strategy is \textit{high risk - high reward}, and the second is \textit{low risk - low reward}.

\subsubsection{Experimental conditions}
In the experiments conducted, the grid was set to a size of $5\times5$ with $n=2$ agents. The number of stags and plants were respectively set to $n_{stag}=1$ and $n_{plant}=2$. In order to have a successful hunt, both agents need to hunt the stag at the same time ($n_{hunt\_min}=2$). The rewards were respectively set to $r_{stag}=10$, $r_{pen}=-2$ and $r_{plant}=2$. The stag was set not to move ($p_{move}=0$) as there is ample stochasticity in the environment already for the sake of initial experimentation (stemming from the random positions and reappearances of the plants and the stag and from the behavior of the agents themselves). Finally, the environment lasted $t_{max}=200$ steps.

The agents observe the grid as a concatenated (flattened) vector consisting of the set of one-hot encoded vectors for each position on the grid to describe what is in contained in them: either ($i$) nothing, ($ii$) a stag, ($iii$) a plant, ($iv$) the observing agent, ($v$) the other agent, or ($vi$) both agents. Thus, the full input is a vector of size is $5\times 5\times 6=150$.
\section{Details on Shield Construction}\label{appsec:shields}

% \subsection{Prisoner's Dilemma}

% Since in general, RL agents optimise their own reward (towards the Nash equilibrium), they would learn to defect. Thus, a shield can be built that forces the agent to cooperate regardless of what the other agent did in the previous time step:
% \begin{lstlisting}[language=Prolog, caption={Shield for \textit{Prisoner's Dilemma}}]
% % actions
% action(0)::action(cooperate);
% action(1)::action(defect).

% % sensors
% sensor_value(0)::sensor(other_cooperate).
% sensor_value(1)::sensor(other_defect).

% % safety constraints
% unsafe_next :- action(defect).
% safe_next :- \+unsafe_next.
% \end{lstlisting}

% Within the literature accounting the \textit{Prisoner's Dilemma}, such an agent is known as an \textit{always cooperate} agent . By computing the safety of an action under this shield we can monitor whether it was a cooperate or defect action due to the safety constraint (the cooperative action has a safety of 1 and defect has a safety of 0). Variations on this simple shield structure will be used often for other experiments as well in this thesis. 

% Also notice that the information that is passed into the sensors are not being used at all in this shield -- these are redundant and we can simplify the shield as follows:

% \begin{lstlisting}[language=Prolog, caption={Simplified shield (\shieldname{coop}) for \textit{Prisoner's Dilemma}}, label={shield:pd}]
% % actions
% action(0)::action(cooperate);
% action(1)::action(defect).

% % safety constraints
% unsafe_next :- action(defect).
% safe_next :- \+unsafe_next.
% \end{lstlisting}

\subsection{Stag-Hunt}

For the pure equilibria, {cooperation} would refer to collaborating with the other agent to hunt the stag and \texttt{defect} refers to an agent individually going for the hare instead.
\begin{lstlisting}[language=Prolog, caption={Shield (\shieldname{pure}) for pure Nash equilibrum for \textit{Stag-Hunt}}, label={shield:pure}]
% actions
action(0)::action(stag);
action(1)::action(hare).

% safety constraints
unsafe_next :- action(hare).
safe_next :- \+unsafe_next.
\end{lstlisting}

For the mixed equilibria, we would need a more sophisticated shield, given the current limitations of the existing implementation of \citet{yang2023safe}. One way of doing this is to add soft constraints to the actions based on how far they differ from the mixed equilibrium. As a proxy for the policy of the agent, a set of the latest $h$ historical actions of the agents can be collected -- this is termed a \textit{buffer}. A \textbf{mean policy} $\hat{\pi}_h$ at the current time $t$ can then be derived from these actions, iterating over all $a'\in\mathcal{A}$:
\begin{equation}
    \hat{\pi}_h(a'\mid s) = \frac{1}{h}\sum_{i\in\{1,...,h\}} \mathbb{I}\left(a^{t-i}=a'\right)
\end{equation}

Let an \textit{a-priori} normative policy be ${\pi}^*$, say, the mixed Nash strategy. If $\hat{\pi}_h$ is too distant from the expected mixed strategy with respect to some divergence measure, then we can re-normalize the input policy using Definition~2 to be more similar to ${\pi}^*$. The divergence measure should be bounded between $[0,1]$ in order to feed it into the ProbLog shield program. For \textit{Stag-Hunt}, let such an input be the absolute difference between $\hat{\pi}_h$ and ${\pi}^*$:
\begin{align}
    \texttt{sensor(stag\_diff)} &:= |{\pi}^*(\text{\textit{Stag}}\mid s)-\hat{\pi}_h(\text{\textit{Stag}}\mid s)|\\
    \texttt{sensor(hare\_diff)} &:= |{\pi}^*(\text{\textit{Hare}}\mid s)-\hat{\pi}_h(\text{\textit{Hare}}\mid s)|
\end{align}
% \newpage
We can then construct a shield as follows:

\begin{lstlisting}[language=Prolog, caption={Shield (\shieldname{mixed}) towards mixed Nash equilibrum for \textit{Stag-Hunt}}, label={shield:mixed}]
% actions
action(0)::action(stag);
action(1)::action(hare).

% sensors
sensor_value(0)::sensor(stag_diff).
sensor_value(1)::sensor(hare_diff).

% safety constraints
unsafe_next :- action(stag), sensor(stag_diff).
unsafe_next :- action(hare), sensor(hare_diff).
safe_next :- \+unsafe_next.
\end{lstlisting}

The constraints that define the value of the \texttt{unsafe\_next} predicate can be interpreted as \textit{it is unsafe to take action stag proportional to the value of sensor(stag\_diff)} (and analogously for playing \textit{Hare}). Thus, the higher the difference between the Nash strategy and the historical policy for each agent, the more the policy is normalized. The policy is not normalized at all when \texttt{sensor(stag\_diff)} and \texttt{sensor(hare\_diff)} are both zero. In this case, the overall policy is perfectly safe.

For the experiments discussed below, the buffer length is set to $h=50$. Additionally, under the payoff matrix in Figure~\ref{fig:app:nfg_utilities}, the \textit{a priori} mixed Nash equilibrium can be computed such that $\pi^*(\text{\textit{Stag}}\mid s)=0.6$ and $\pi^*(\text{\textit{Hare}}\mid s)=0.4$ for both agents. The evaluation safety is computed based on $\mathcal{T}_{pure}$ as it directly reflects the actions of the agents (with $\mathbf{P}_{\mathcal{T}_{pure}}(\text{\textit{Stag}}\mid s)=1$ and $\mathbf{P}_{\mathcal{T}_{pure}}(\text{\textit{Hare}}\mid s)=0$).

\subsection{Centipede Game}

Since the agents have two actions, we can use a shield to force the agents to always play \textit{C} at every time step (regardless of previous actions) until the game terminates and a reward is attained using the following shield.
\begin{lstlisting}[language=Prolog, caption={Simple shield \shieldname{continue} for \textit{Centipede}}, label={shield:centipede}]
% actions
action(0)::action(continue);
action(1)::action(stop).

% safety constraints
unsafe_next :- action(stop).
safe_next :- \+unsafe_next.
\end{lstlisting}
As a result, it is expected that shielded agents will always continue, as this is the only safe action in any round.

\subsection{Extended Public Goods Game}

Since, at each time step $t$, there is effectively a new simultaneous game with 2 actions, we can dictate the actions of the agents with a simple shield such as the one used for the \textit{Centipede Game}. Thus, it is trivial to create a shield that dictates the actions of the agents to have a certain policy such as \textit{always play cooperate}, regardless of inputs.

Let us instead create a shield that that uses a model $\hat{\mu}$ of the true $\mu$ estimated over training so far. Additionally, we can use the standard deviation of the multiplication factors so far, $\hat{\sigma}$ to model the uncertainty the agent has with respect to the current $f_t$ compared to $\hat{\mu}$. This can be done online without the need of a buffer using Welford's method \citep{welford1962note}. 
% by setting $\hat{\mu}_{new}\leftarrow\hat{\mu}_{old}+(f_t-\hat{\mu}_{old})/t$. 
We then compare $\hat{\mu}$ to the predetermined Nash outcomes and decide if it is high enough to warrant defecting -- if $\hat{\mu}\geq1$, \textit{cooperate}, else \textit{defect}. In Shield~\ref{shield:pgg}, this is set as the (Boolean) value of the fact \texttt{sensor(mu\_high)} -- its value is set to 1 if $\hat{\mu}\geq1$ else 0. The agents will err on the side of a cooperative strategy when $1<\hat{\mu}<2$, where the rational strategy would be mixed.

We then might want to include uncertainty to increase exploitation if it has high certainty about the environment (align the agent more strongly with the inductive bias induced via the shield), and contrariwise, increase exploration if there is high uncertainty (guide the agent more weakly). We must thus construct a smooth distance measure $d$ between $f_t$ and $\hat{\mu}$ bounded between $[0,1]$ such that $d(f_t,\hat{\mu})=1$ when $f_t=\hat\mu$ and $d(f_t,\hat{\mu})=0$ when $|f_t-\hat{\mu}|\rightarrow\infty$. One of many possibilities is described here. First, we compute the $z$-score of the current $f_t$ with respect to the estimated parameters so far:
\begin{equation}
    z(f_t\mid \hat{\mu}, \hat{\sigma}) = \frac{f_t-\hat{\mu}}{\hat{\sigma}}
\end{equation}
We then use the CDF of a standard normal distribution to translate this into a probability value $\Phi(z(f_t\mid \hat{\mu}, \hat{\sigma}))$. Subtracting 0.5 centers $\Phi(z(f_t\mid\hat{\mu}, \hat{\sigma}))$ around 0 for $z(f_t\mid \hat{\mu}, \hat{\sigma})=0$. Doubling the absolute value and subtracting from 1 inverts the scale to make it 1 at the mean and symmetrically 0 at extreme tails. Thus the final distance function is:
\begin{equation}
    d(f_t,\hat{\mu}) = 1-2\cdot\left|\Phi(z(f_t\mid \hat{\mu}, \hat{\sigma}))-0.5\right|
\end{equation}

\noindent In Shield~\ref{shield:pgg} below, the value of the fact \texttt{sensor(f\_certainty)} is thus set to $d(f_t,\hat{\mu})$. Thus using the inputs defined above, the shield is constructed as follows:
\begin{lstlisting}[language=Prolog, caption={Shield (\shieldname{EPGG}) for EPGG capturing uncertainty.}, label={shield:pgg}]
% actions
action(0)::action(cooperate);
action(1)::action(defect).

% sensors
sensor_value(0)::sensor(mu_high).
sensor_value(1)::sensor(f_certainty).

% safety constraints
unsafe_next :- \+action(cooperate), sensor(mu_high), sensor(f_certainty).
unsafe_next :- \+action(defect), \+sensor(mu_high), sensor(f_certainty).
safe_next :- \+unsafe_next.
\end{lstlisting}

\noindent The two constraints on lines 10 and 11 essentially say, \textit{it is unsafe if cooperate is not played and $\hat\mu>1$ with some certainty about $f_t$} (and analogously for \textit{defect}).

\subsection{Markov Stag-Hunt}

Let us focus on constructing a shield that might help the agents learn how to cooperate and hunt a stag together. Similar shields can be constructed for non-cooperative behavior too, however, it is much harder for the agents to converge to the cooperative behavior in the first place.

Since the environment is rather complex, there are a myriad of ways this goal may be achieved, and two are described here. The inputs to the shield are the policy and some sensors that describe relative positions of the stag and the other agent:

\begin{lstlisting}[language=Prolog, caption={Partial shield for \textit{Markov Stag-Hunt.}}]
% actions
action(0)::action(left);
action(1)::action(right);
action(2)::action(up);
action(3)::action(down);
action(4)::action(stay).

% sensors
sensor_value(0)::sensor(left).
sensor_value(1)::sensor(right).
sensor_value(2)::sensor(up).
sensor_value(3)::sensor(down).
sensor_value(4)::sensor(stag_near_self).
sensor_value(5)::sensor(stag_near_other).
\end{lstlisting}

The value of the sensor predicates \texttt{sensor(stag\_near\_self)} and \texttt{sensor(stag\_near\_other)} is 1 when the stag is adjacent to the acting agent and the other agent respectively, and 0 otherwise. The sensors with literals named after the cardinal directions are binary and represent the relative position of the stag with respect to the agent -- for example, \texttt{sensor(left)} is 1 when the stag is strictly to the left of the acting agent and 0 otherwise.

A first shield can be constructed that influences the actions of the agents such that they strongly tend to go towards the stag, wait for the other agent, and then hunt the stag together. Such a shield can be implemented as such:
\begin{lstlisting}[language=Prolog, caption={Strong shield ($\mathcal{T}_{strong}$) for \textit{Markov Stag-Hunt}.}, label={shield:msh_strong}]
% actions & sensors go here
...

% define movement towards the stag
go_towards_stag :- action(Dir), sensor(Dir).
% define a state where both agents are near the stag
stag_surrounded :- sensor(stag_near_self), 
                   sensor(stag_near_other).

% 1) it is unsafe to not go towards the stag if it not near
unsafe_next :- \+go_towards_stag, 
               \+sensor(stag_near_self).
% 2) it is unsafe to not wait, the stag is near without another agent
unsafe_next :-  \+action(stay), sensor(stag_near_self), 
                \+sensor(stag_near_other).
% 3) it is unsafe to not hunt, the stag is near and there is another agent
unsafe_next :- \+go_towards_stag, stag_surrounded.

% combine all unsafe conditions to get safety
safe_next :- \+unsafe_next.
\end{lstlisting}

This shield will be referred to as the \textit{strong shield} in this environment, since it significantly affects the policy of the agent regardless of its position on the grid.

However, in general, depending on the environment, we may not know how to construct a strongly shielded agent or we may wish to constrain the agent less in order for it to learn optimal policies in a more end-to-end fashion. Thus, in these cases, a shield can be constructed that only significantly effects the agent's policy in a small number of cases. For \textit{Markov Stag-Hunt} another shield created is the following:

\begin{lstlisting}[language=Prolog, caption={Weak shield ($\mathcal{T}_{weak}$) for \textit{Markov Stag-Hunt}.}, label={shield:msh_weak}]
% actions & sensors go here
...

% define movement towards the stag
go_towards_stag :- action(Dir), sensor(Dir).
% define a state where both agents are near the stag
stag_surrounded :- sensor(stag_near_self), 
                   sensor(stag_near_other).

% it is unsafe not to hunt stag and it is surrounded
unsafe_next :- \+go_towards_stag, stag_surrounded.
safe_next :- \+unsafe_next.

% if the stag is not nearby, any action is fine (safety=1)
safe_next :- \+sensor(stag_near_self).
\end{lstlisting}

This shield will be called the \textit{weak shield} as it affects the policy of an agent that is in the neighborhood of the stag. It only guides the agent to hunt the stag when another agent is nearby, and otherwise allows the policy to explore freely. 

The set of safety-related sentences $\mathcal{BK}$ for the weak shield can be argued to be a subset of that of the strong shield ($\mathcal{BK}_{weak}\subset\mathcal{BK}_{strong}$). An additional line (line 14 in the weak shield) has to be added to ensure that the behaviour of the agent over all the sensor space is well-defined, and to allow any policy to be safe in the complement of $\mathcal{BK}_{weak}$ with respect to $\mathcal{BK}_{strong}$ (i.e. the sentences $\mathcal{BK}_{strong}\setminus\mathcal{BK}_{weak}$).

\section{Single-Agent Experiment: CartSafe} 
\label{appsec:cartsafe_results}

\begin{figure*}[h!]
    \centering
    \begin{minipage}[b]{0.32\textwidth}
        \centering
        \includegraphics[width=\textwidth]{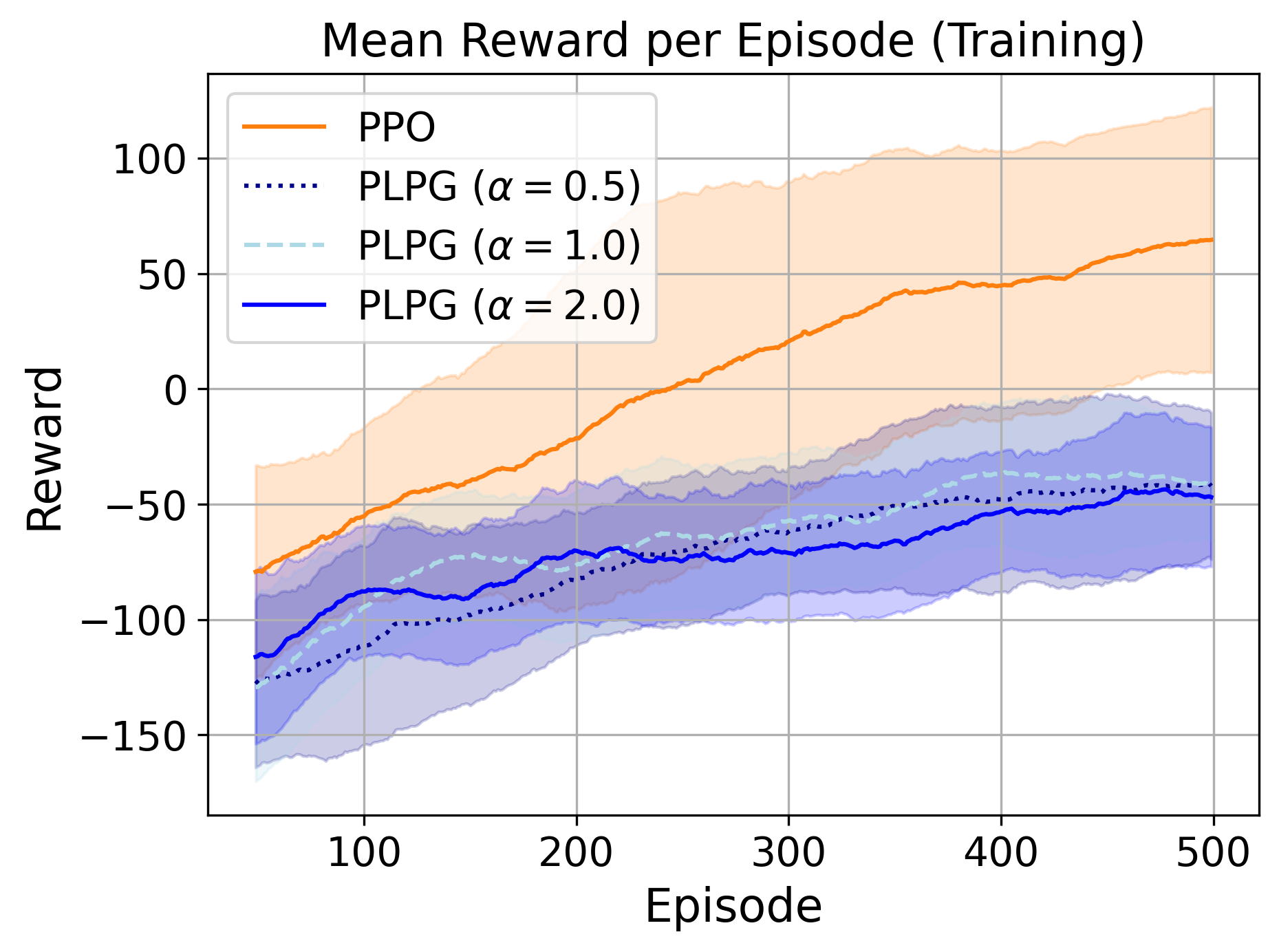}
        \vspace{0.3em}

        \textbf{(a)}
        \label{fig:subfig1}
    \end{minipage}
    \hfill
    \begin{minipage}[b]{0.32\textwidth}
        \centering
        \includegraphics[width=\textwidth]{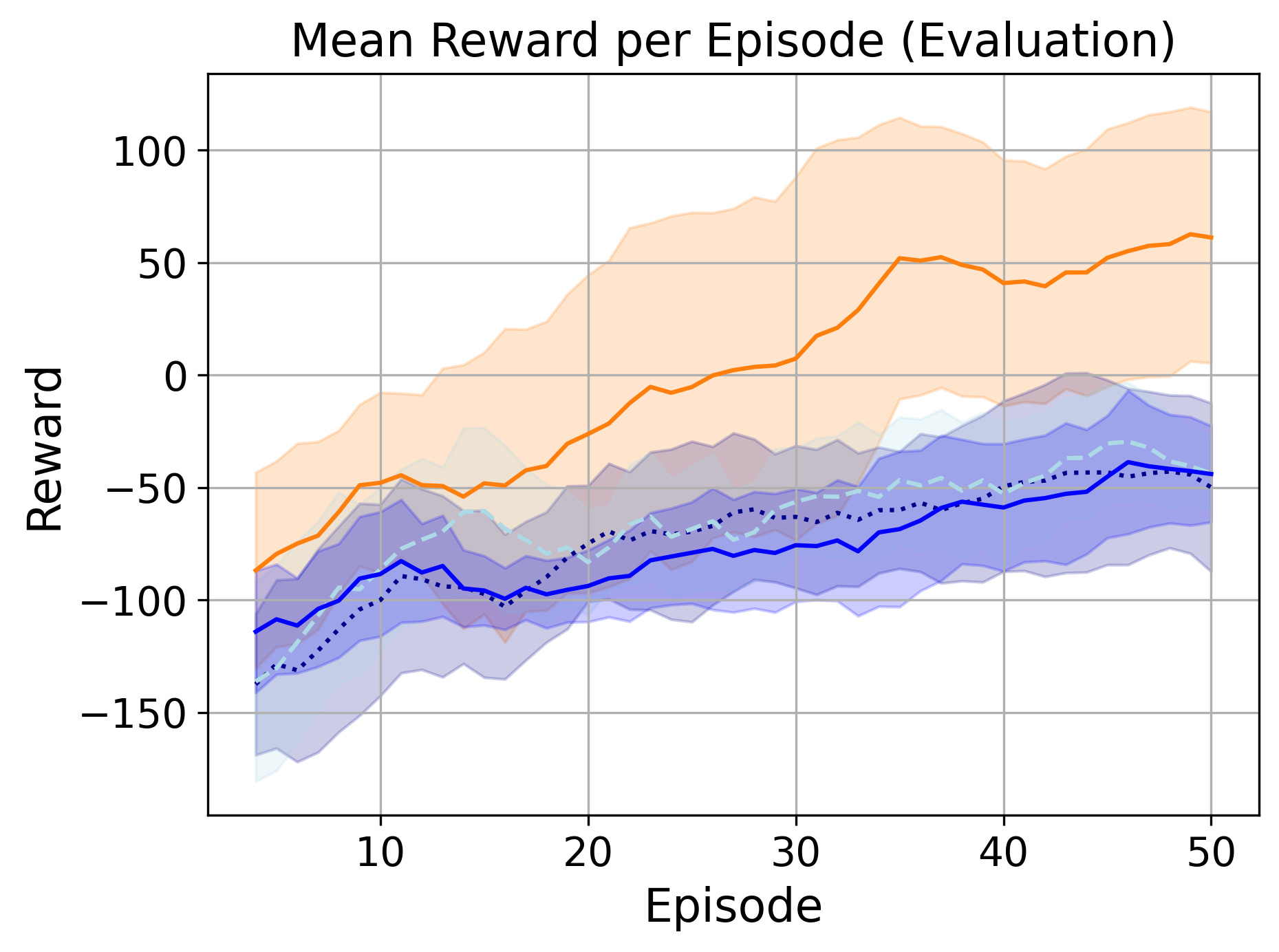}
        \vspace{0.3em}

        \textbf{(b)}
        \label{fig:subfig2}
    \end{minipage}
    \hfill
    \begin{minipage}[b]{0.32\textwidth}
        \centering
        \includegraphics[width=\textwidth]{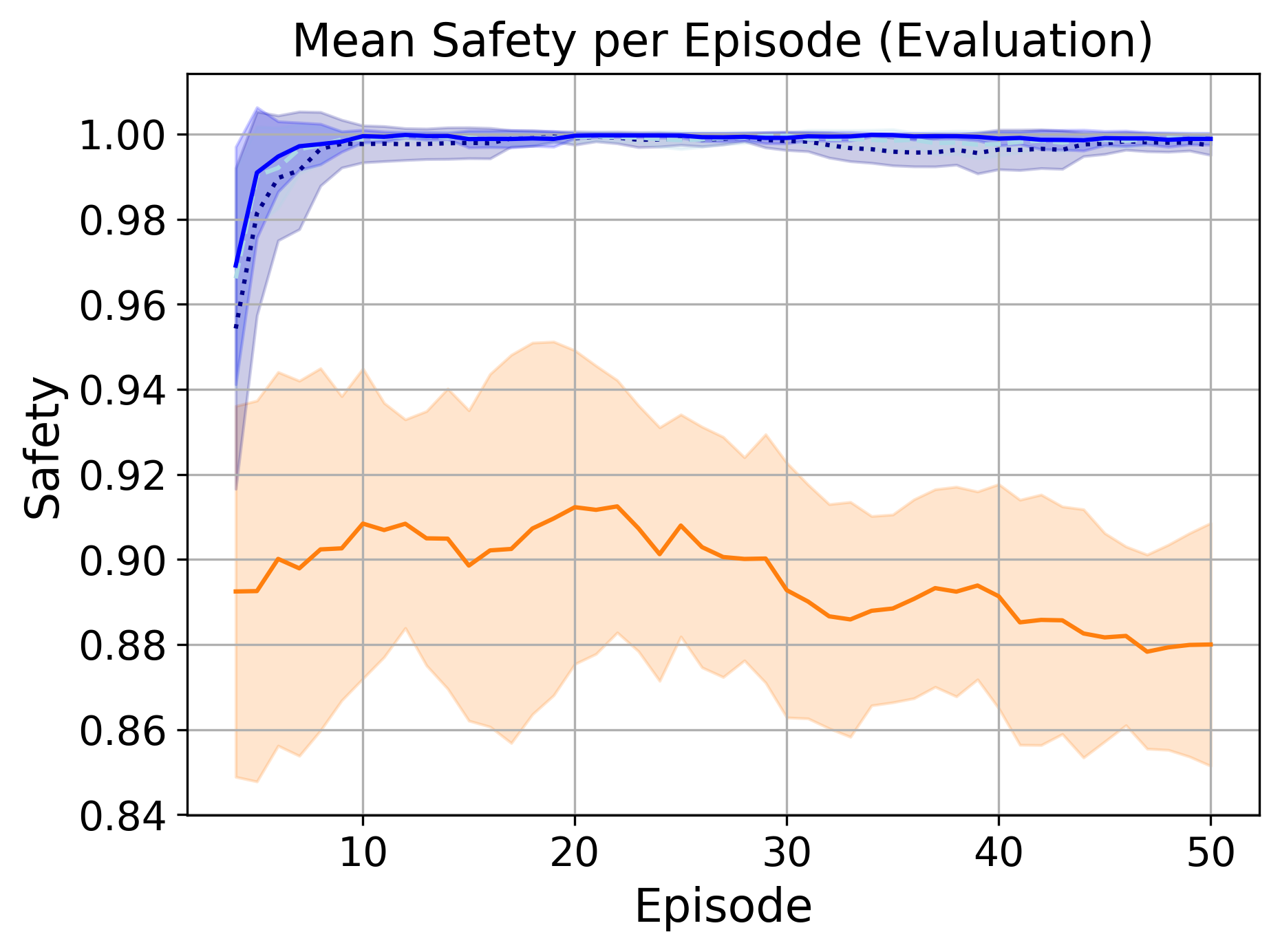}
        \vspace{0.3em}

        \textbf{(c)}
        \label{fig:subfig3}
    \end{minipage}

    \caption[Single-agent Results for \textit{CartSafe} (PLPG agents).]{Training, evaluation and safety results for \textit{CartSafe} for PPO-based agents. The lines represent the mean and the shadow the standard deviation over 5 seeds. Results are smoothed using a rolling average with window size 50.}
    \label{fig:cartsafe_ppo}
\end{figure*}
\begin{figure*}[h!]
    \centering

    % First row: Off-policy, epsilon-greedy
    \begin{minipage}[b]{0.32\textwidth}
        \centering
        \includegraphics[width=\textwidth]{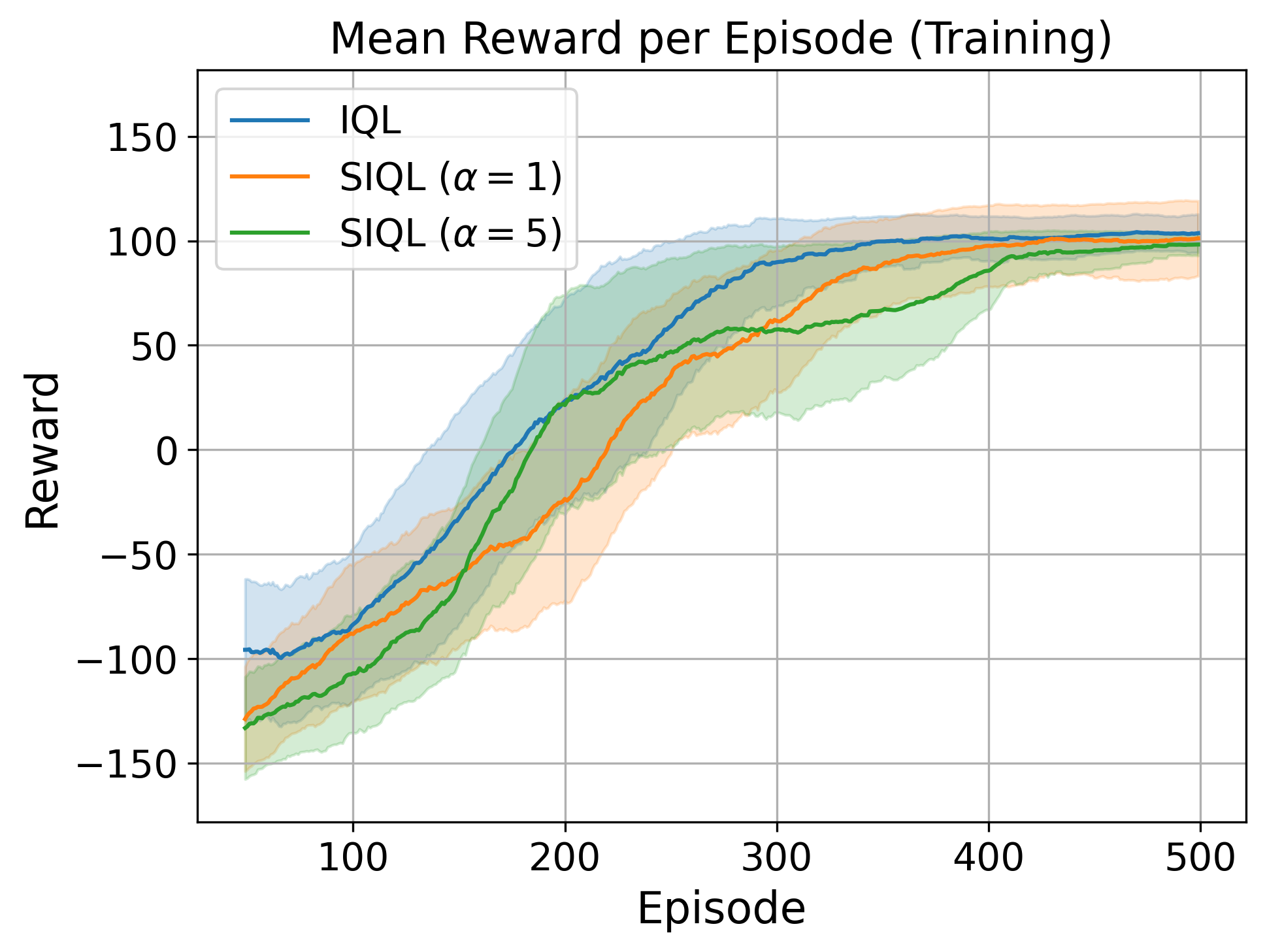}
        \vspace{0.3em}

        \textbf{(a)}
        \label{fig:cartsafe_dqn:sub1}
    \end{minipage}
    \hfill
    \begin{minipage}[b]{0.32\textwidth}
        \centering
        \includegraphics[width=\textwidth]{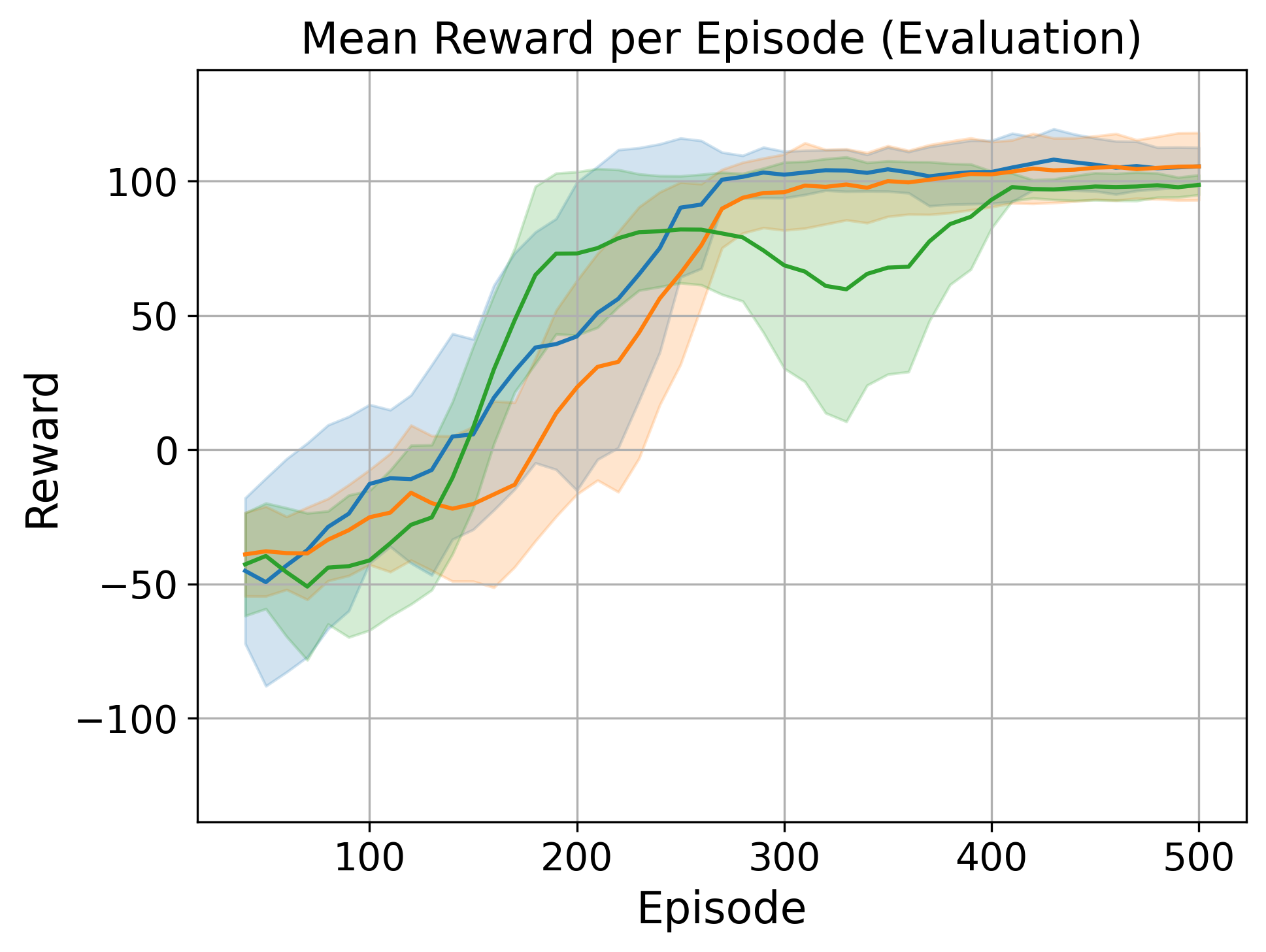}
        \vspace{0.3em}

        \textbf{(b)}
        \label{fig:cartsafe_dqn:sub2}
    \end{minipage}
    \hfill
    \begin{minipage}[b]{0.32\textwidth}
        \centering
        \includegraphics[width=\textwidth]{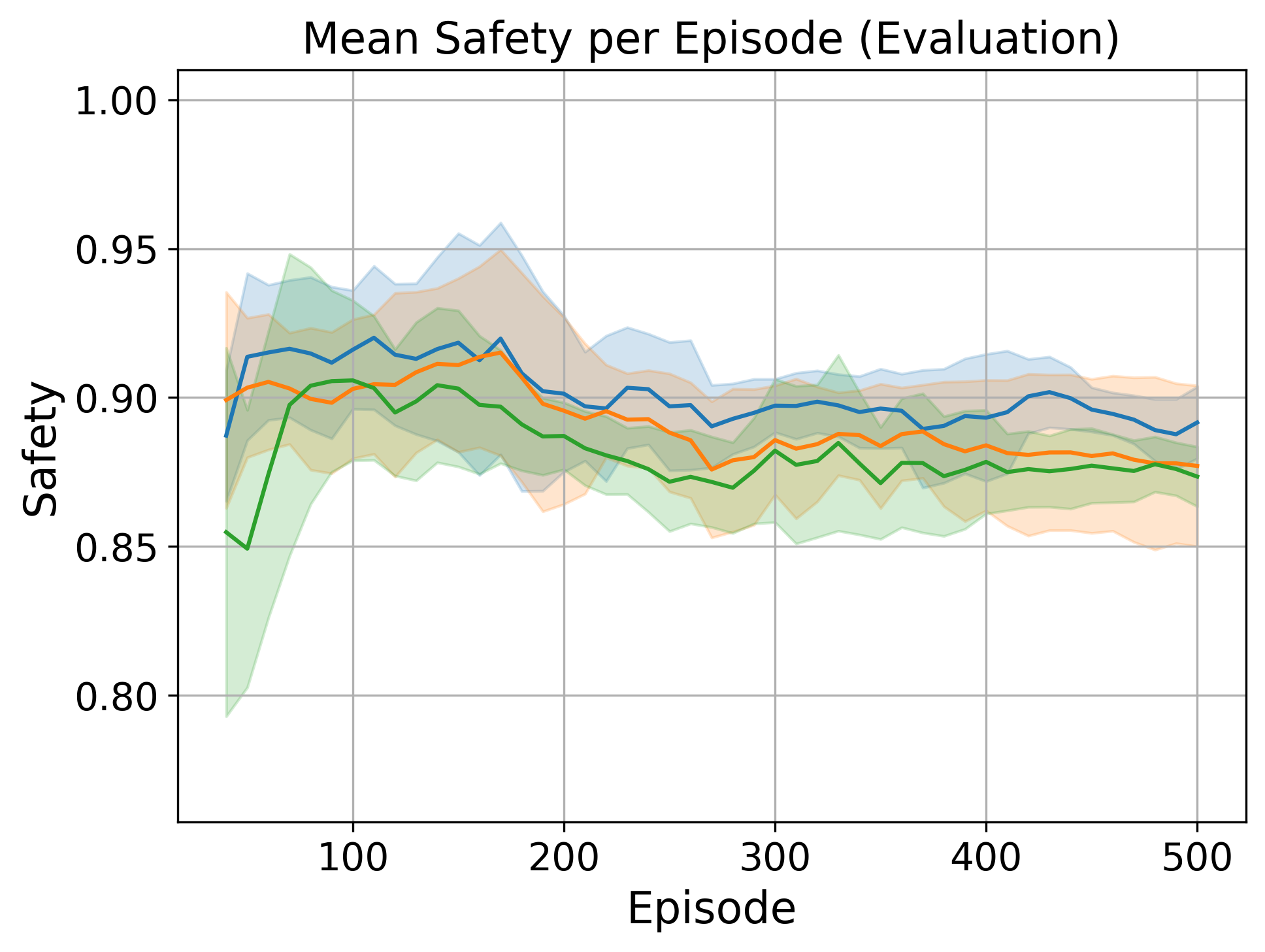}
        \vspace{0.3em}

        \textbf{(c)}
        \label{fig:cartsafe_dqn:sub3}
    \end{minipage}

    \vspace{10pt}

    % Second row: On-policy, epsilon-greedy
    \begin{minipage}[b]{0.32\textwidth}
        \centering
        \includegraphics[width=\textwidth]{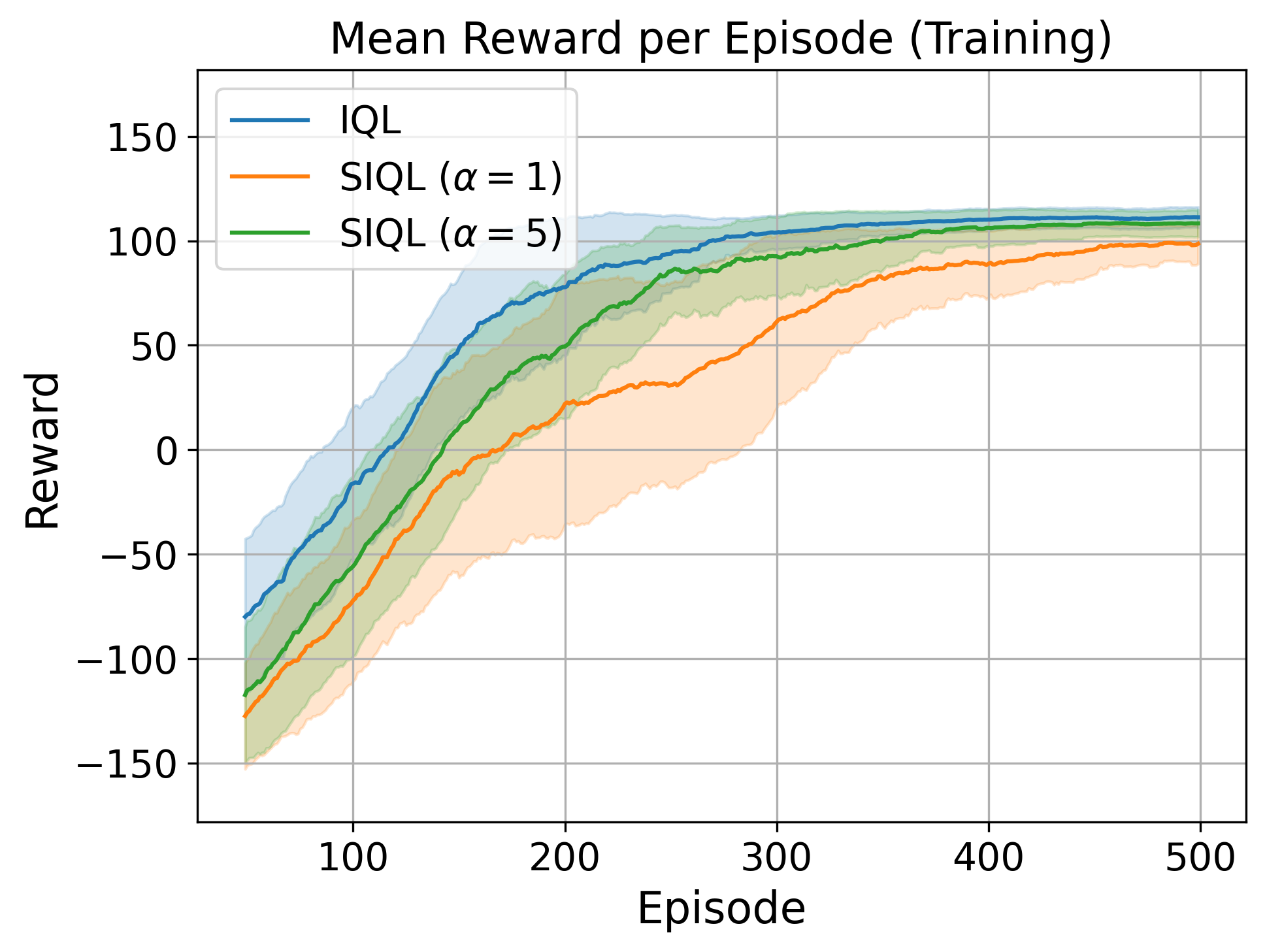}
        \vspace{0.3em}

        \textbf{(d)}
        \label{fig:cartsafe_dqn:sub4}
    \end{minipage}
    \hfill
    \begin{minipage}[b]{0.32\textwidth}
        \centering
        \includegraphics[width=\textwidth]{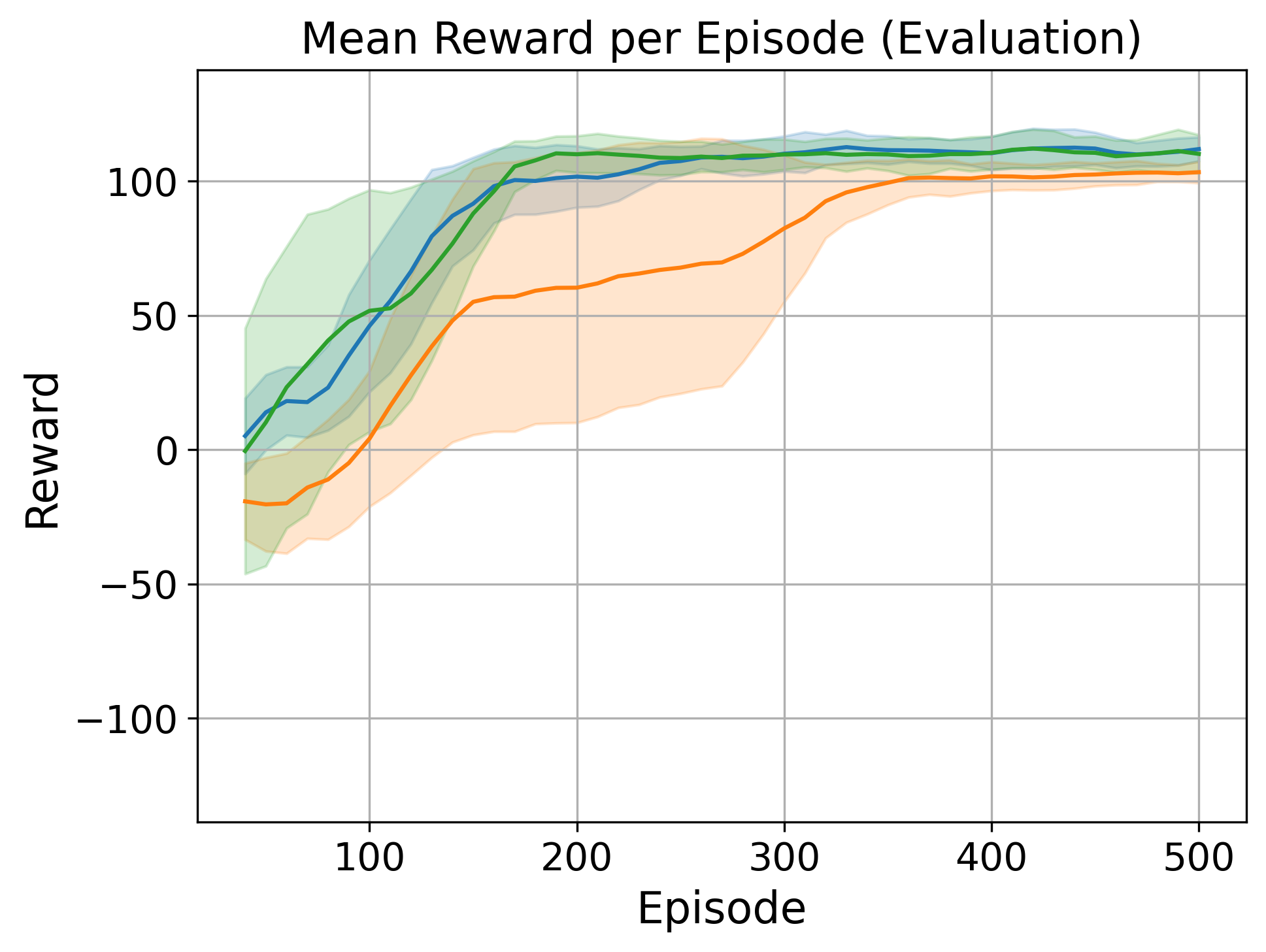}
        \vspace{0.3em}

        \textbf{(e)}
        \label{fig:cartsafe_dqn:sub5}
    \end{minipage}
    \hfill
    \begin{minipage}[b]{0.32\textwidth}
        \centering
        \includegraphics[width=\textwidth]{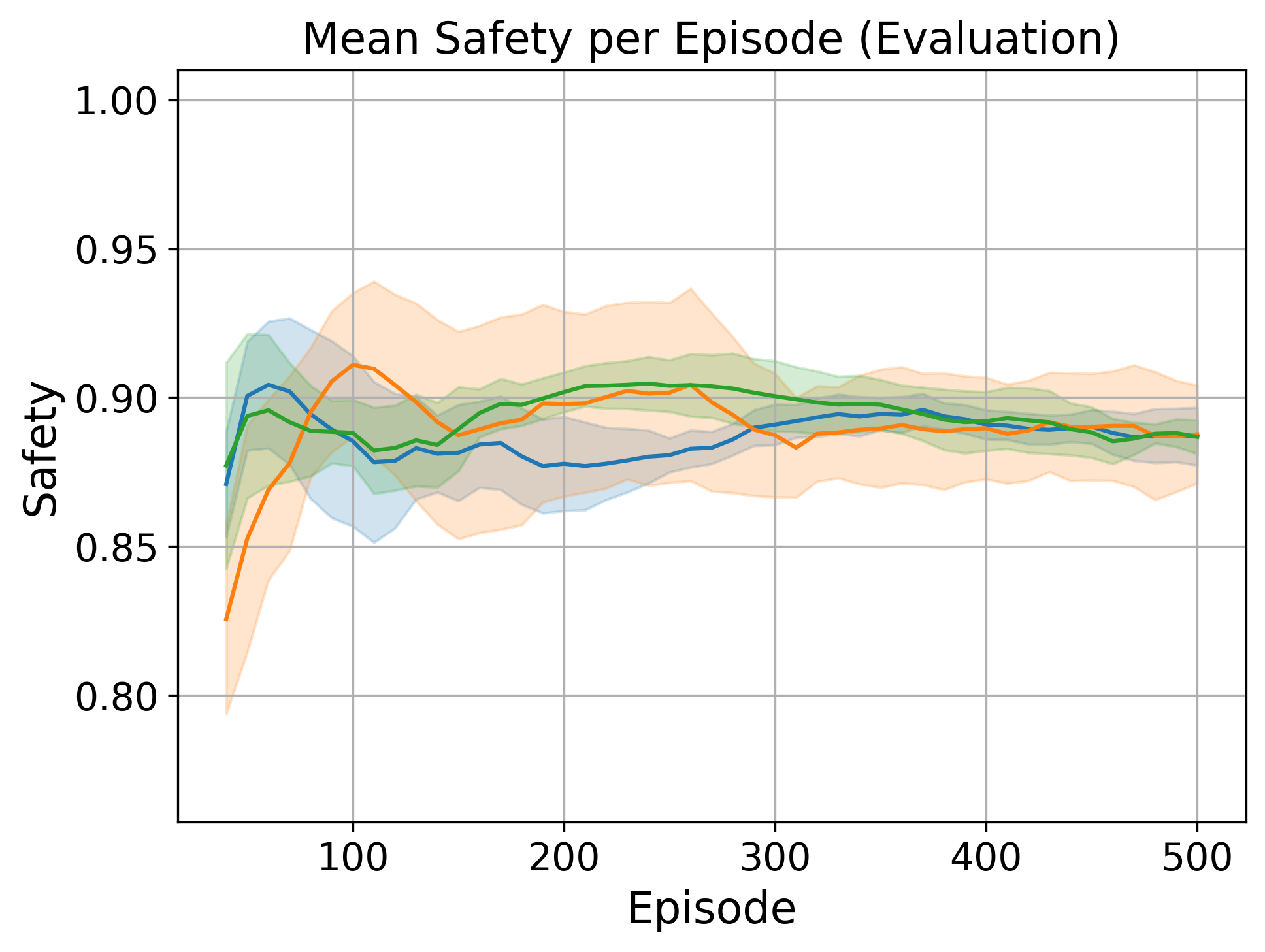}
        \vspace{0.3em}

        \textbf{(f)}
        \label{fig:cartsafe_dqn:sub6}
    \end{minipage}

    \vspace{10pt}

    % Third row: Off-policy, softmax
    \begin{minipage}[b]{0.32\textwidth}
        \centering
        \includegraphics[width=\textwidth]{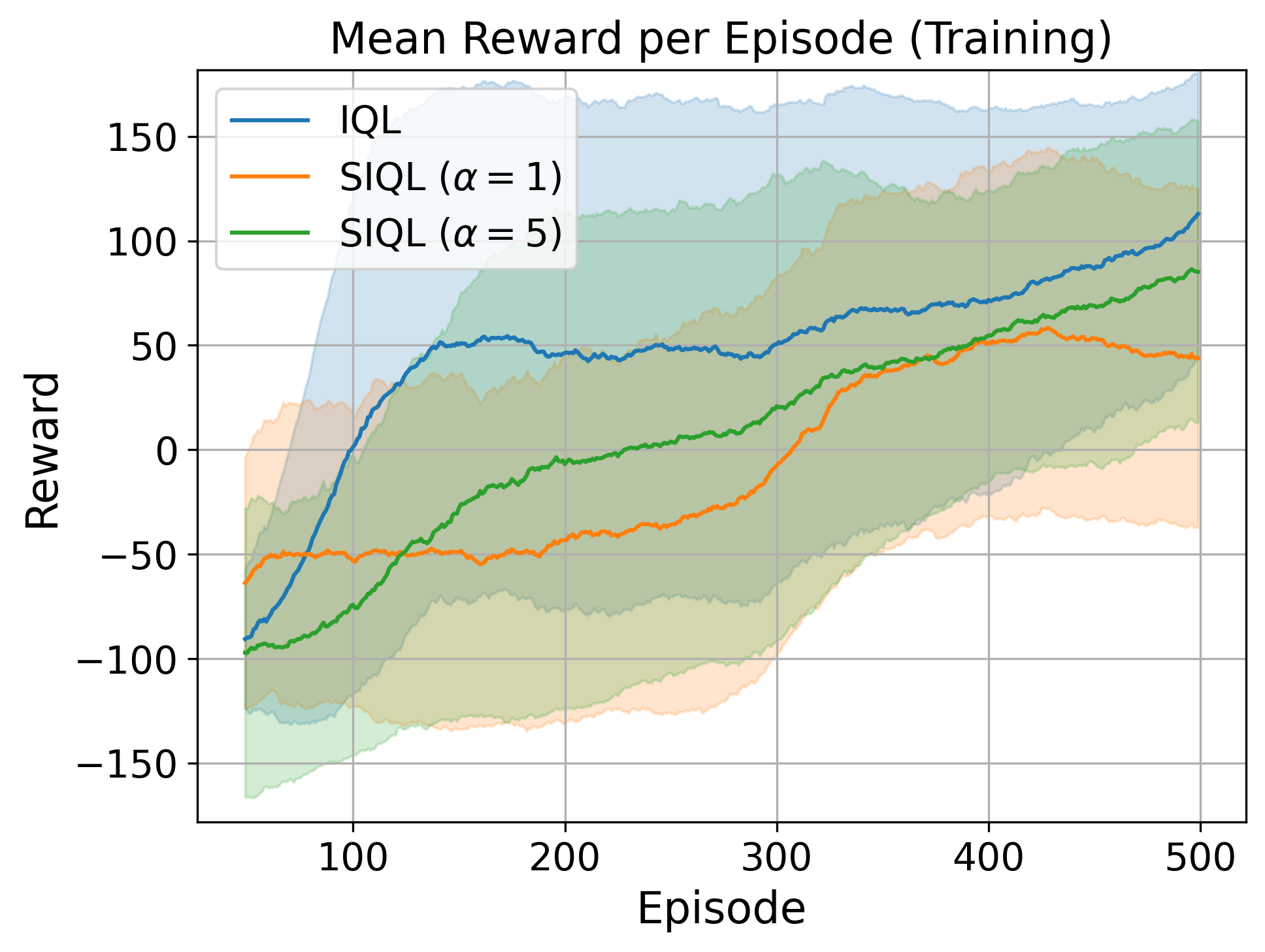}
        \vspace{0.3em}

        \textbf{(g)}
        \label{fig:cartsafe_dqn:sub7}
    \end{minipage}
    \hfill
    \begin{minipage}[b]{0.32\textwidth}
        \centering
        \includegraphics[width=\textwidth]{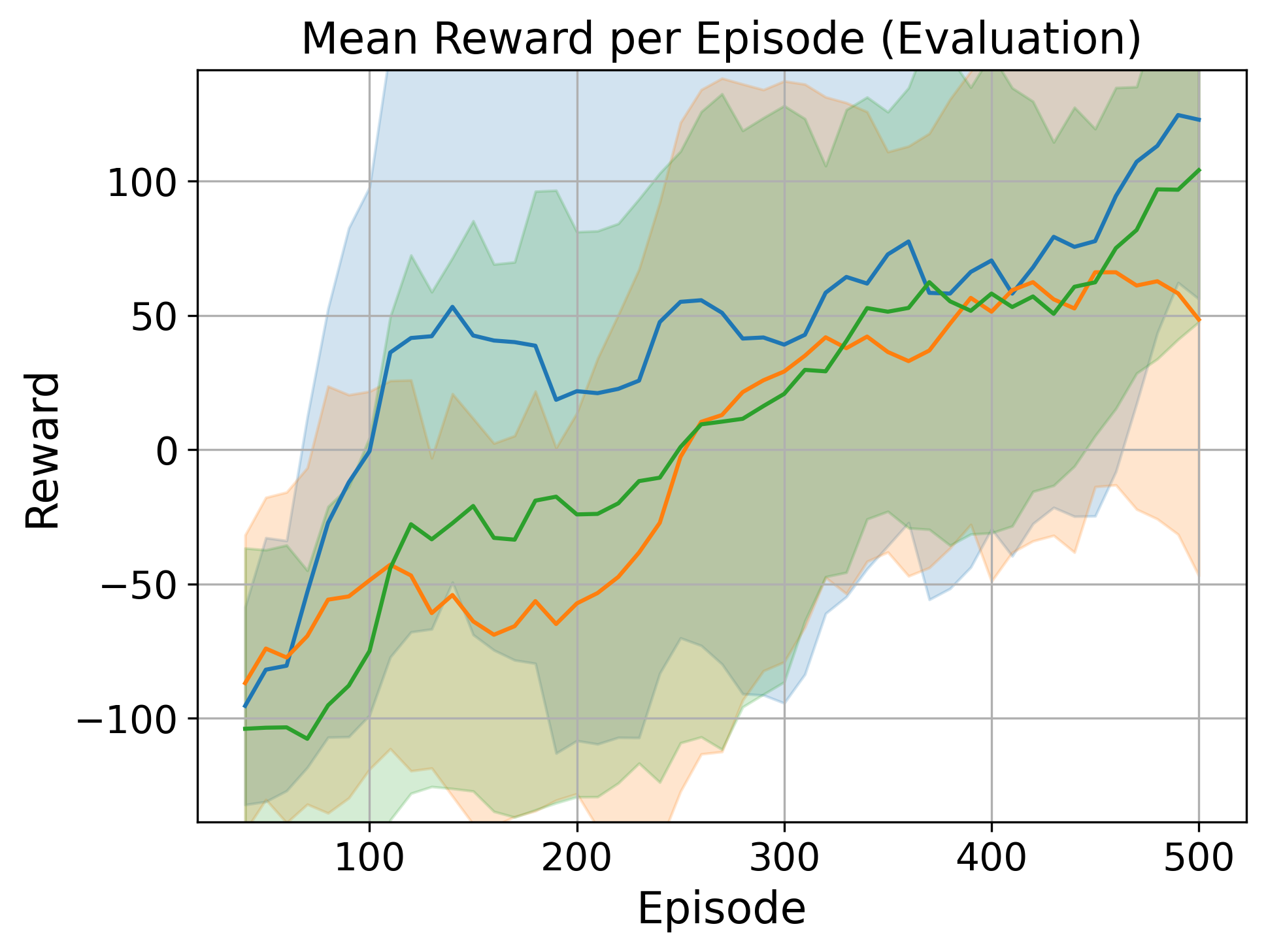}
        \vspace{0.3em}

        \textbf{(h)}
        \label{fig:cartsafe_dqn:sub8}
    \end{minipage}
    \hfill
    \begin{minipage}[b]{0.32\textwidth}
        \centering
        \includegraphics[width=\textwidth]{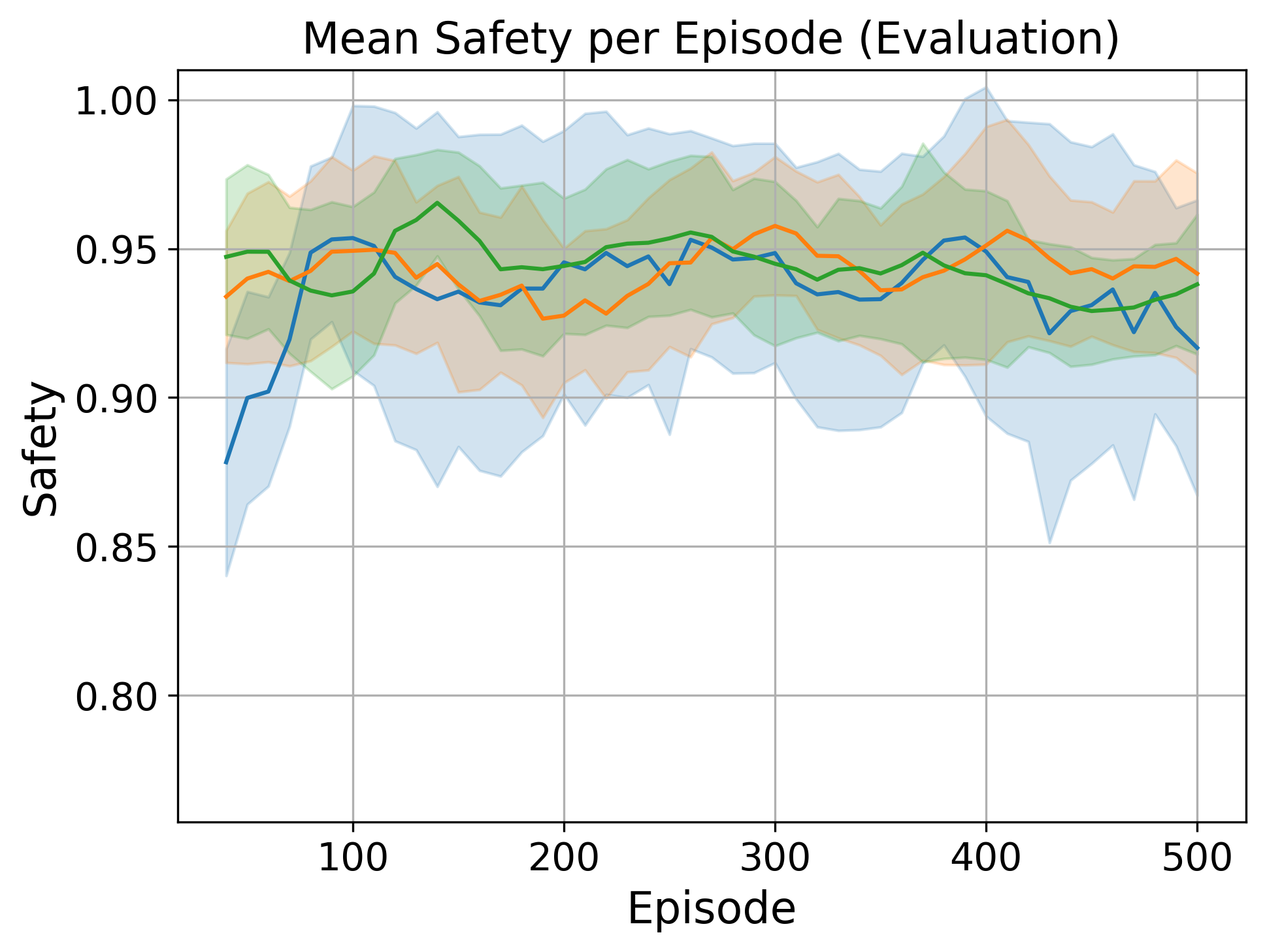}
        \vspace{0.3em}

        \textbf{(i)}
        \label{fig:cartsafe_dqn:sub9}
    \end{minipage}

    \vspace{10pt}

    % Fourth row: On-policy, softmax
    \begin{minipage}[b]{0.32\textwidth}
        \centering
        \includegraphics[width=\textwidth]{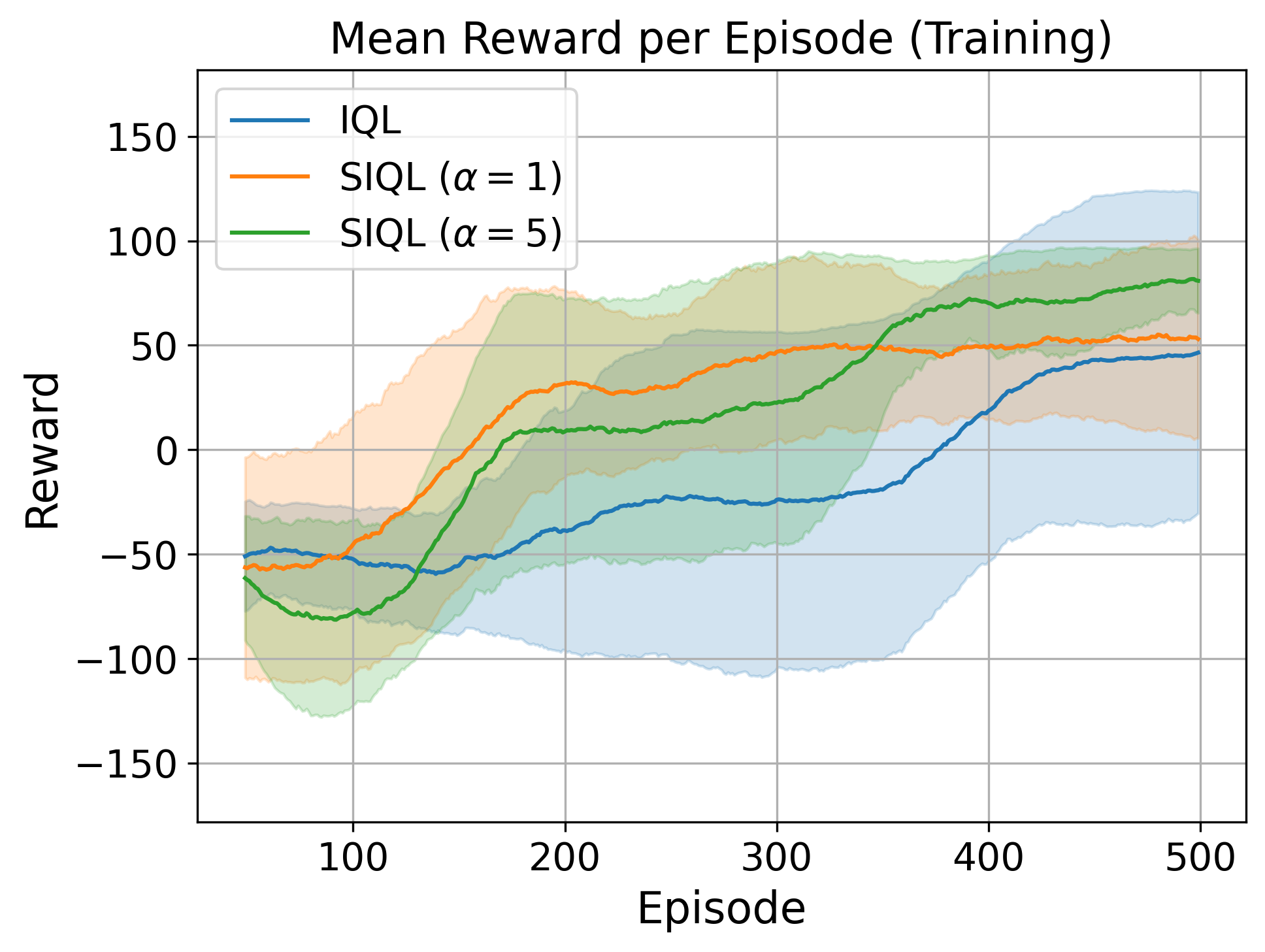}
        \vspace{0.3em}

        \textbf{(j)}
        \label{fig:cartsafe_dqn:sub10}
    \end{minipage}
    \hfill
    \begin{minipage}[b]{0.32\textwidth}
        \centering
        \includegraphics[width=\textwidth]{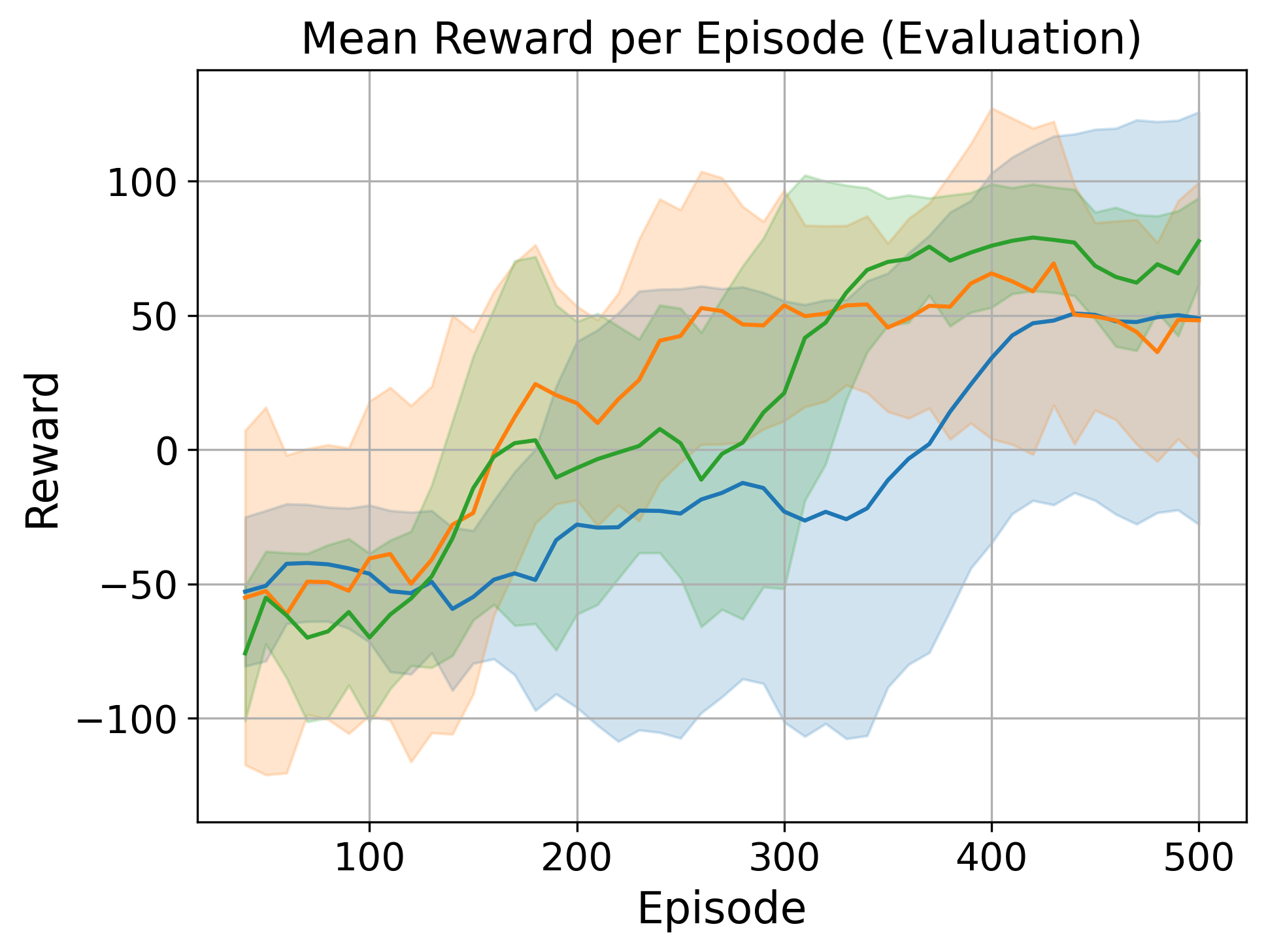}
        \vspace{0.3em}

        \textbf{(k)}
        \label{fig:cartsafe_dqn:sub11}
    \end{minipage}
    \hfill
    \begin{minipage}[b]{0.32\textwidth}
        \centering
        \includegraphics[width=\textwidth]{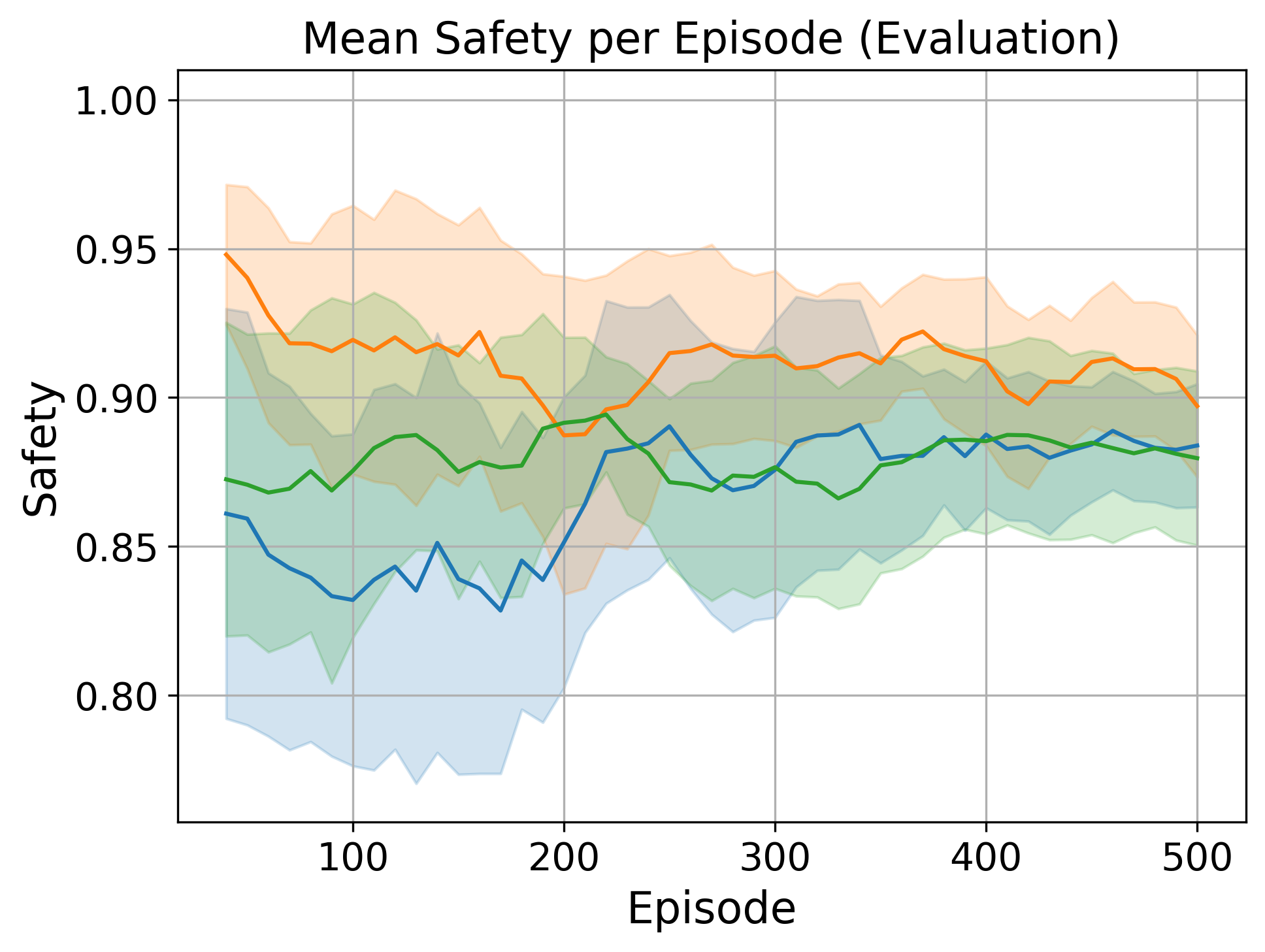}
        \vspace{0.3em}

        \textbf{(l)}
        \label{fig:cartsafe_dqn:sub12}
    \end{minipage}

    % Figure caption
    \caption[Single-agent Results for \textit{CartSafe} (PLTD agents).]{Training, evaluation and safety results for \textit{CartSafe} for on- and off-policy DQN-based agents with different exploration strategies: \textit{top row:} off-policy, $\epsilon$-greedy; \textit{top middle row:} on-policy, $\epsilon$-greedy; \textit{bottom middle row:} off-policy, softmax; and \textit{bottom row:} on-policy, softmax. The lines represent the mean and the shadow the standard deviation over 3 seeds.}
    \label{fig:cartsafe_dqn}
\end{figure*}

Before testing the two shielded multi-agent DQN-based algorithms (SIQL and SPSQL), it is imperative to first test whether or not PLTD (ProbLog-shielded DQN) works as well as vanilla DQN for environments with constraints, and how it compares to PLPG (ProbLog-shielded PPO). Additionally, it would be useful to verify that the custom PLPG code works for single agent settings before attempting multi-agent experiments.

\textit{CartSafe} is modification of OpenAI Gym's classical control problem \textit{CartPole} \citep{brockman1606.01540} with simple constraints inspired by \citet{jemaw2019}. The problem has the following elements: the environment consists of a rod attached to a box or cart at one end and is allowed to freely swing around the attachment point. This setup exists in a bounded 2-dimensional real space with a set width and an active physics simulation. At each time step, the agent has two actions: accelerate the cart to the left or to the right. The agent observes the cart's position, velocity, the angular velocity of the rod and the angle of the rod in radians. The goal is to actively balance the rod upright without leaving the bounds of the environment. The agent receives a reward at each time step based on the angle of the pole ($\theta_{pole}$) with respect to vertical: $r^t=1$ if $-15^\circ\leq\theta_{pole}\leq15^\circ$ and $r^t=-1$ otherwise. The environment terminates when a maximum number of steps ($t_{max}$) are reached or the cart moves off screen. The additional constraints that are added to this setup are such that the cart must not enter the region bounded by the leftmost and rightmost quarters of the environment -- thus the agent must learn to balance the rod near the center of the screen.

The purpose of this experiment is trifold: 
\begin{enumerate}
    \item To test whether the PLPG and ProbLog shield implementations adequately work on a common RL task with a continuous state space;
    \item To verify the ideas of PLTD introduced in Section 3.2 and its implementation to see whether it would help (or at least does not hurt) the ability of a DQN agent converging to a good and safe policy;
    \item To explore the possibilities in terms of exploration and on/off-policy strategies to see how they affect performance and safety of PLTD agents.
\end{enumerate}

\subsection{Experimental conditions}
The maximum length of each episode was set to $t_{max}=200$ with 500 training episodes. The agents are evaluated every 10 episodes. Each algorithm was re-run with 5 different seeds each and their aggregated results are discussed below. The PPO and DQN parameters were minimally hand-tuned and can be found in Table~\ref{table:hyperparams} in Appendix~\ref{appsec:hyperparams}. The PPO agents were shielded with safety penality coefficient $\alpha\in\{0.5,1.0,2.0\}$. The DQN-based algorithms were tested by varying the following aspects: on-policy vs.\ off-policy, $\epsilon$-greedy vs. softmax exploration (paired with greedy and softmax exploitation policies respectively) and whether they are shielded (SIQL with safety coefficient $\alpha\in\{1.0,5.0\}$) or not (IQL).

% \begin{figure}[h]
%     \centering
%     \fbox{
%     \includegraphics[width=0.3\linewidth]{Experiments/images/cartsafe/cartsafe_env.png}}
%     \caption[Starting Configuration for \textit{CartSafe}.]{Example starting configuration for \textit{CartSafe}. The `cart' is the black box that the agent controls, the brown rectangle is a `pole' that pivots around the purple attachment point. The regions to the right and left of the red vertical lines are the constrained regions ($\pm1$ unit away from the center, with a maximal width of the environment at 4.8 units).}
%     \label{fig:cartsafe_env}
% \end{figure}

\subsection{Shield construction}
A simple shield that aims to satisfy the constraints outlined in the task description can be constructed as follows:
\begin{lstlisting}[language=Prolog, caption={Shield for soft constraint satisfaction for \textit{CartSafe}.}, label={shield:cartsafe}]
% actions
action(0)::action(left);
action(1)::action(right).

% sensors
sensor_value(0)::sensor(cost).
sensor_value(1)::sensor(xpos).
sensor_value(2)::sensor(left).
sensor_value(3)::sensor(right).

% constraints
unsafe_next :- action(X), sensor(X), 
               sensor(cost), sensor(xpos).
safe_next :- \+unsafe_next.
\end{lstlisting}

The actions correspond to moving the cart to the left or right. The sensors here correspond to (in order) whether or not the current state is within a constrained region (Boolean value), the distance from the center on the $x$-axis normalized between 0 and 1 (such that it is 0 when the cart is in the middle, and 1 when it is at the right or left edges of the screen), and whether the cart is on the left side or right side of the screen (Boolean value). The constraint can be interpreted as \textit{it is unsafe if an action is taken towards some side X, the cart is on side X, the current state is in the constrained region, and the x-position is sufficiently large.}

This means that if the cart is within the permissible central region of the environment, the safety probability of the next action is computed as $(\texttt{sensor(cost)}=0)\implies(\texttt{safe\_next}=1)$, and the policy remains unchanged. If the cart is too far right, outside the permissible region, it is unsafe to move the cart towards the right. The \textit{right} action gets progressively more unsafe as the cart gets further to the right. The safe policy will proportionally favor the left action as a result. Analogous behavior occurs if the cart is too far left instead of right.

Note that this also means that the agent is \textit{already} within the constrained region (specifically, at the edge of the region) when the shield takes effect -- this can be easily solved by instead having a predicate that looks at whether the \textit{next} step will be impermissible given an action, rather than whether is \textit{currently} is in such a state. However, this should not significantly alter the behavior of the agent in this environment.

Note that $\texttt{sensor(xpos)}$ is the \textit{only} non-Boolean input to this shield. If it not placed in the body of the constraint in line 12, since \texttt{sensor(X)} and \texttt{sensor(cost)} are purely Boolean, the safe policy computation will result in a deterministic behavior in the constrained regions -- \textit{go left if the cart is in the right-hand constrained region with probability 1}, and vice-versa for the left-hand constrained region. Thus, the behavior of the agent act safely immediately and deterministically. The inclusion of $\texttt{sensor(xpos)}$ allows us to experiment with soft constraints -- here, it enforces the idea that it gets linearly worse when travel deeper into the constrained regions.

\subsection{Results: PLTD (Shielded DQN)}

Figure~\ref{fig:cartsafe_dqn} displays training, evaluation and safety results of a variation of PLTD agents. As a reminder, these agents vary in the following aspects: on-policy vs.\ off-policy, $\epsilon$-greedy vs. softmax exploration (paired with greedy and softmax exploitation policies respectively) and whether they are shielded (SIQL with safety coefficient $\alpha\in\{1,5\}$) or not (IQL).

\paragraph{\textit{Shielding}:} First, we may look at whether adding a shield significantly alters the behavior of the agent in this setting. The rightmost column shows the mean safety per episode computed with the CartSafe shield (Shield~\ref{shield:cartsafe}). We see that in general, adding a shield to an agent with either of the tested safety penalty values does not significantly affect its safety, \textit{ceteris paribus}. However, all safety values tend to be quite high, generally above 0.9 with not much improvement after the first $\approx100$ episodes. We also note that the shielded agents do not do \textit{worse} than the unshielded ones. This may hint at the safety constraints being relatively simple to accomplish given the environment goals and algorithms. Thus there is little to no trade-off between reward and safety in this task.

\paragraph{\textit{$\epsilon$-Greedy vs.\ Softmax}:} When looking at the difference between the exploration and exploitation strategies, we see quite a large difference in the variation of results when trained multiple times -- the $\epsilon$-greedy-based algorithms tend to more reliably converge to a good policy (converging around reward $\approx100$), whereas the softmax policies tend to have a much larger standard deviation with less smoothly increasing curves, to policies that generally attain less reward per episode.

\paragraph{\textit{On vs.\ Off-Policy}:} For the $\epsilon$-greedy runs, it is noticeable that the on-policy versions of the algorithms all tends to converge faster than the off-policy ones, reaching around reward~$\approx$~50 in 100 episodes during evaluation versus 200 episodes for off-policy DQN. We do not see the same trend for the softmax policies, it seems that that it does reduce the variance across seeds significantly. There does not seem to be a significant difference in safety between these two conditions. One outlier in this regard is that the off-policy agents with softmax policies are slightly but significantly more safe than all the others.

\paragraph{\textit{PLTD vs.\ PLPG}:}
Based on the experiments conducted and discussed, we see that PLPG and PLTD agents act somewhat differently in this setting. None of the PPO-based agents attain the same reward as the DQN-based ones. Adding shields to these agents has a stronger influence the PLPG agents than the PLTD ones -- this is not unexpected, as the PLTD agents attain safety information both from the safety penalty term and the safe policy gradient, whereas safety information for PLTD agents is only input through the safety penalty. The PLTD agents tend to do better in terms of rewards, but worse in terms of safety, which is on par with the baseline of the unshielded PLTD agents or the vanilla PPO agent.

\subsection{Results: PLPG (Shielded PPO)}
Figure~\ref{fig:cartsafe_ppo} shows training and evaluation results for four agents that have a PPO-based learning scheme. The first agent is vanilla (unshielded) PPO (the orange line) and the other three are PLPG agents (blue) with varying safety penalties, $\alpha\in\{0.5,1,2\}$.

We see that for the training reward graphs (showing the total rewards attained per agent per episode), there is a relatively high variation in rewards, but the evaluation graphs show a clear upward trend of the agents learning to balance the pole. The variation is likely due to (at least in part) the choice of hyperparameters. The PPO agent shows better performance based on reward, but with much higher variation, and it does not seem like the algorithms have yet converged after 500 episodes. This trend holds for the evaluations as well (occurring after every 10 episodes) -- the PPO agent have better evaluation rewards but with more variation. 

When it comes to evaluating safety, we do see a very significant difference -- the unshielded PPO agent has a final mean evaluation safety of around 0.88 (learning to balance the pole generally but not always in the center), but all other PLPG agents converge quickly to a nearly perfect safe policy within 10 training episodes across seeds. Increasing the safety penalty $\alpha$ does not seem to make a significant difference for any of the shielded algorithm results.

\section{List of Hyperparameters}\label{appsec:hyperparams}

\begin{table*}[!h]
\centering
\begin{tabular}{|c|c|c|c|c|c|c|}
\hline
\textbf{Algorithm} & \textbf{Parameter} & \textbf{CartSafe} & \textbf{Stag-Hunt} & \textbf{Centipede} & \textbf{EPGG} & \textbf{Markov Stag-Hunt} \\
\hline\hline
\multirow{9}{*}{\textbf{DQN}} & Epochs & 1 & x & 1 & x & x \\
\cline{2-7}
& $\gamma$ & 0.9 & x & 0.99 & x & x \\
\cline{2-7}
& Buffer size & 10000  & x & 512 & x & x \\
\cline{2-7}
& Batch size & 128 & x & 128& x & x \\
\cline{2-7}
& lr & 0.001 & x & 0.001 & x & x \\
\cline{2-7}
& $\epsilon$-decay & 0.99996 & x & 0.9972 & x & x \\
\cline{2-7}
& $\epsilon_{min}$ & 0.01 & x & 0.01 & x & x \\
\cline{2-7}
& $\alpha$ & $\{1.0,5.0\}$ & x & 1.0 & x & x \\
\cline{2-7}
& $\tau$ & 1.0 & x & 1.0 & x & x \\
\hline\hline
\multirow{9}{*}{\textbf{PPO}} & Epochs & 10  & 10 &  10 & 10 & 10 \\
\cline{2-7}
& $\gamma$ & 0.9 & 0.99 & 0.99 & 0.99 & 0.99 \\
\cline{2-7}
& Buffer size & 400 & 50 & 100 & 50 & 100 \\ % update timestep
\cline{2-7}
& $\epsilon$-clip & 0.1 & 0.1 & 0.15 & 0.1 & 0.1 \\
\cline{2-7}
& lr (actor) & 0.001 & 0.001 & 0.001 & 0.001 & 0.001 \\
\cline{2-7}
& lr (critic) & 0.001 & 0.001 & 0.001 & 0.001 & 0.001 \\
\cline{2-7}
& VF coef 0.5 & 0.5 & 0.5 & 0.5 & 0.5 & 0.5 \\
\cline{2-7}
& S coef 0.01 & 0.01  & 0.01 & 0.01 & 0.01 & 0.01 \\
\cline{2-7}
& $\alpha$ & $\{0.5,1.0,2.0\}$  & 1.0 & 1.0 & 1.0 & 1.0 \\
\hline
\end{tabular}
\caption{Hyperparameters for experiments. }
\label{table:hyperparams}
\end{table*}

Table~\ref{table:hyperparams} describes the hyperparameters used in the underlying RL algorithms for all independent MARL experiments. The hyperparameters are:
\begin{itemize}
    \item \textbf{Epochs:} The number of epochs that the neural networks are trained per training call.
    \item {$\mathbf{\gamma}$:} The discount factor.
    \item \textbf{Buffer size:} The number of latest steps (states and actions taken) contained in the history buffer.
    \item \textbf{Batch size:} (DQN) The batch size sampled from the buffer per training epoch.
    \item \textbf{lr/lr (actor)/lr (critic):} Learning rate for Q network (DQN), actor/critic networks (PPO). All networks were optimised via Adam \citep{kingma2014adam}, with all other hyperparameters left as in PyTorch's default implementation \citep{paszke2019pytorch}.
    \item \textbf{$\epsilon$-decay:} (DQN) $\epsilon$-decay factor $\delta$ for $\epsilon$-greedy exploration according to the following formula, where $t$ is the number of exploration steps taken so far:
    \begin{equation}
        \epsilon\leftarrow\delta^t
    \end{equation}
    \item \textbf{$\epsilon_{min}$:} (DQN) Minimum value of $\epsilon$ for $\epsilon$-greedy exploration.
    \item \textbf{$\tau$:} (DQN) Exploration factor for softmax exploration.
    \item \textbf{VF coef:} (PPO) Value-function loss coefficient.
    \item \textbf{S coef:} (PPO) Entropy loss coefficient.
    \item \textbf{$\epsilon$-clip}: (PPO) Clipped policy update ratio range. 
    \item \textbf{$\alpha$:} Safety penalty coefficient (cf.\ Definition~3.3 and \citet{yang2023safe}'s Definition 5.1).
\end{itemize}
Additionally, all DQN agents used neural networks with 2 hidden layers of 64 neurons each with a ReLU activation in between each layer, and all PPO agents used the same for actor and critic networks, except using the hyperbolic tangent instead of ReLU between layers.

% % \newpage
\section{Training Curves} \label{appsec:training_curves}
The training curves for \textit{2-player Extended Public Goods Game} (Figure~\ref{fig:epgg_results}) and
% \textit{Centipede} (Figure~\ref{fig:centipede_results}),
\textit{2-player Markov Stag-Hunt} (Figures~\ref{fig:msh_full_results} and \ref{fig:msh_part_results}) are given below. All remaining experiment training curves will be provided at the request of the reader.
\begin{figure*}[h!]
    \centering

    % Row 1: mu = 0.5
    \begin{minipage}[b]{0.3\textwidth}
        \centering
        \includegraphics[width=0.9\textwidth]{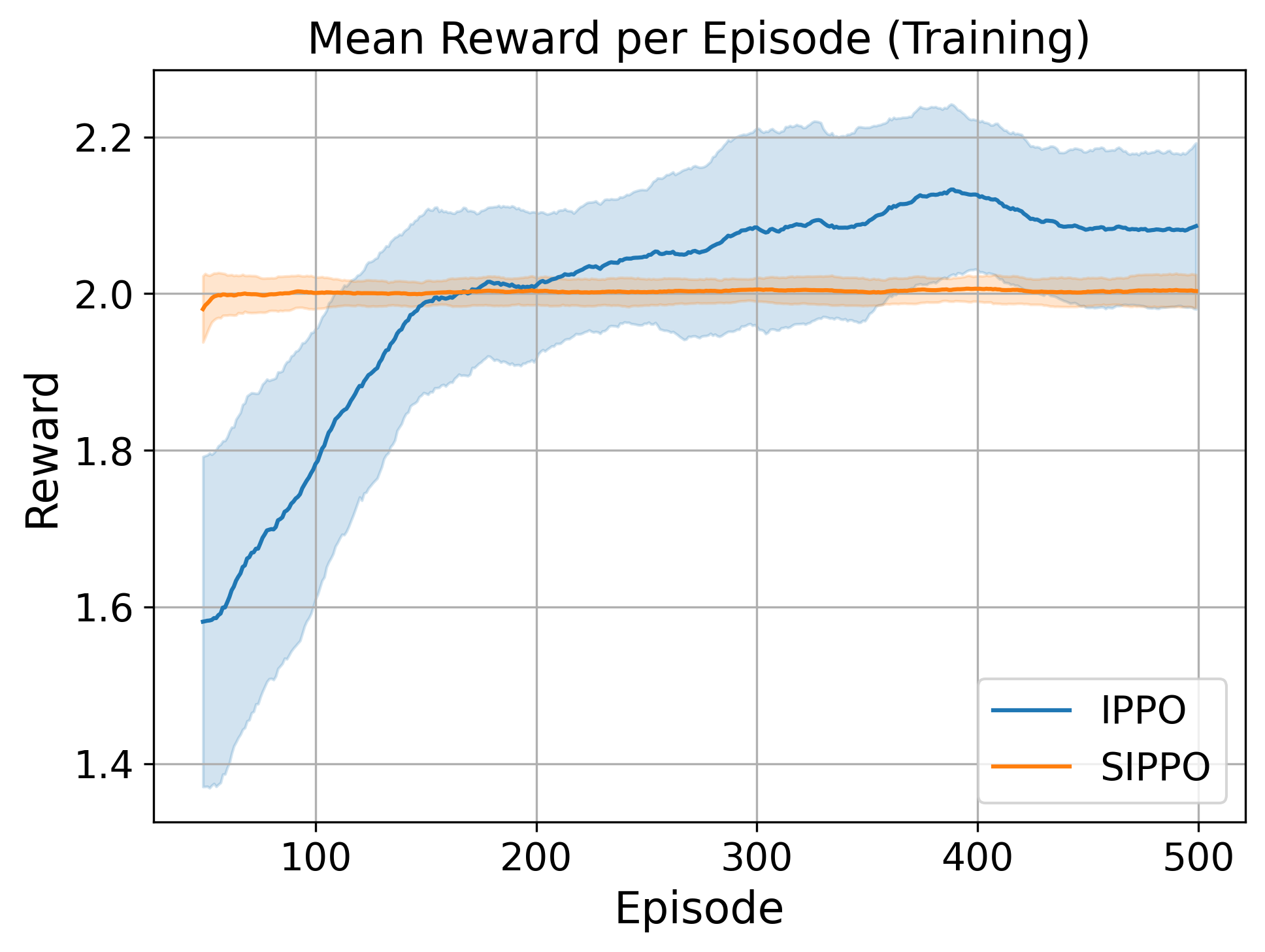}
        \vspace{0.3em}

        \textbf{(a)} Training ($\mu=0.5$)
        \label{fig:epgg_results:sub1}
    \end{minipage}
    \hfill
    \begin{minipage}[b]{0.3\textwidth}
        \centering
        \includegraphics[width=0.9\textwidth]{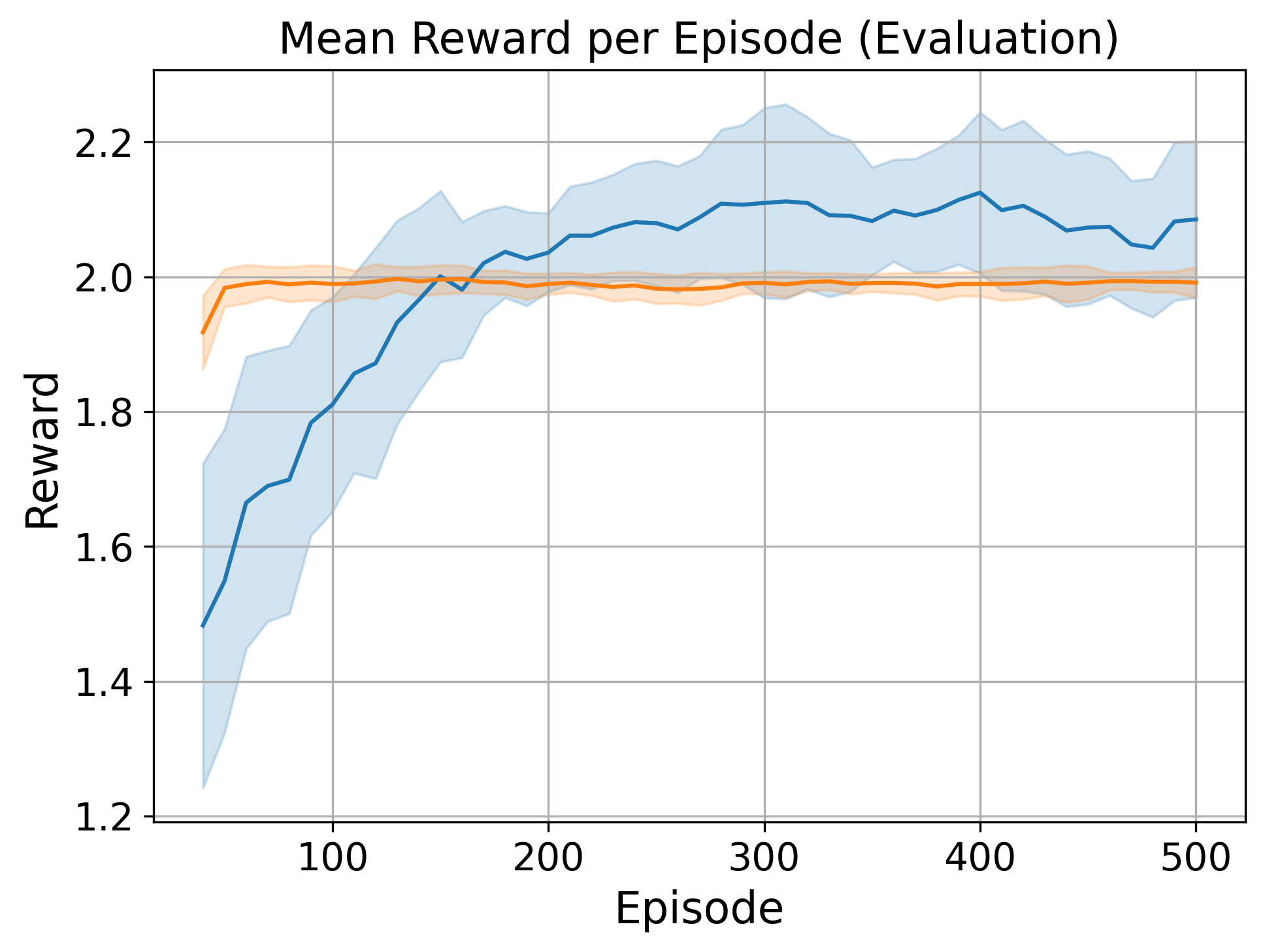}
        \vspace{0.3em}

        \textbf{(b)} Evaluation ($\mu=0.5$)
        \label{fig:epgg_results:sub2}
    \end{minipage}
    \hfill
    \begin{minipage}[b]{0.3\textwidth}
        \centering
        \includegraphics[width=0.9\textwidth]{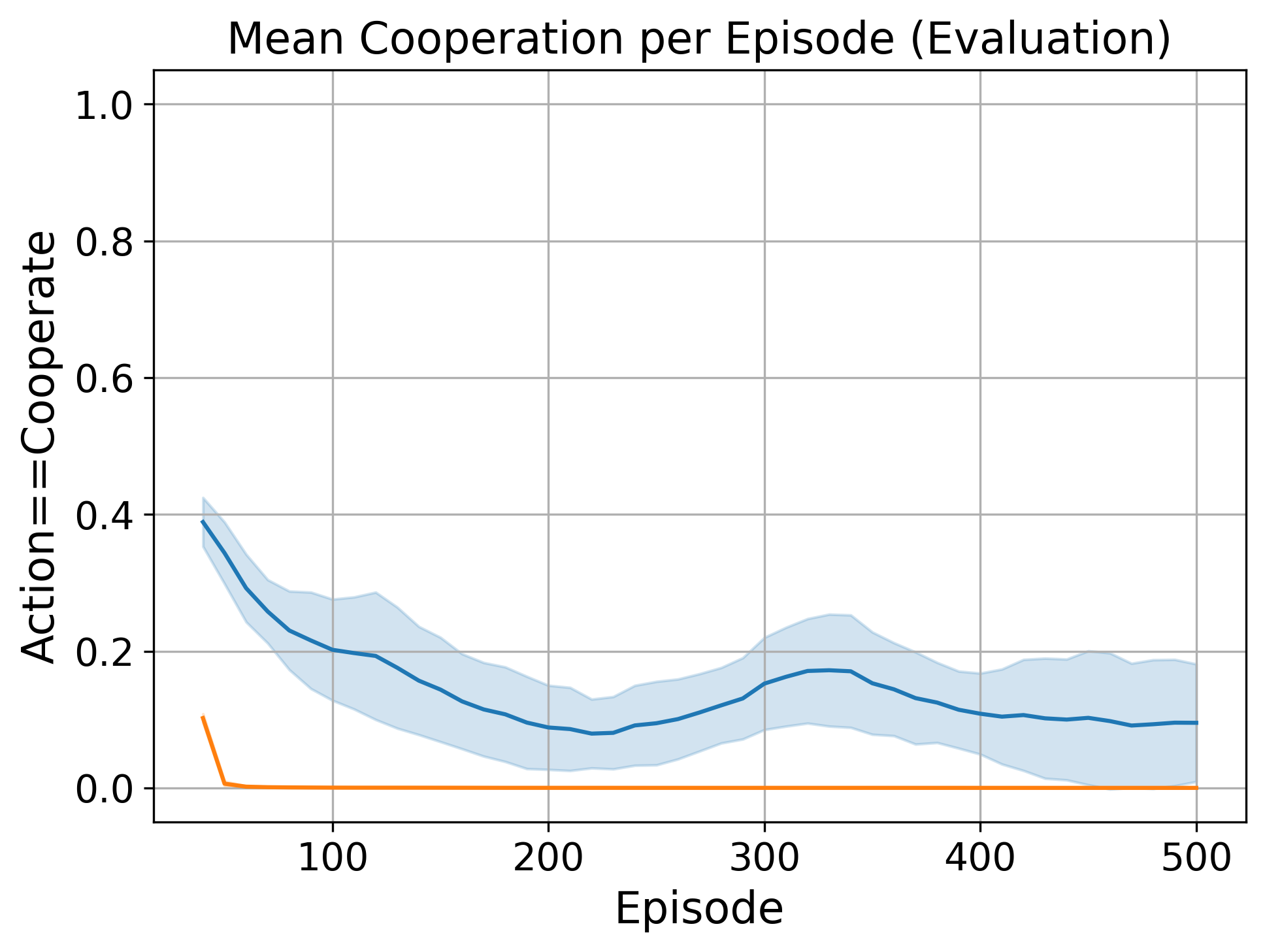}
        \vspace{0.3em}

        \textbf{(c)} Safety ($\mu=0.5$)
        \label{fig:epgg_results:sub3}
    \end{minipage}

    \vspace{5pt}

    % Row 2: mu = 1.0
    \begin{minipage}[b]{0.3\textwidth}
        \centering
        \includegraphics[width=0.9\textwidth]{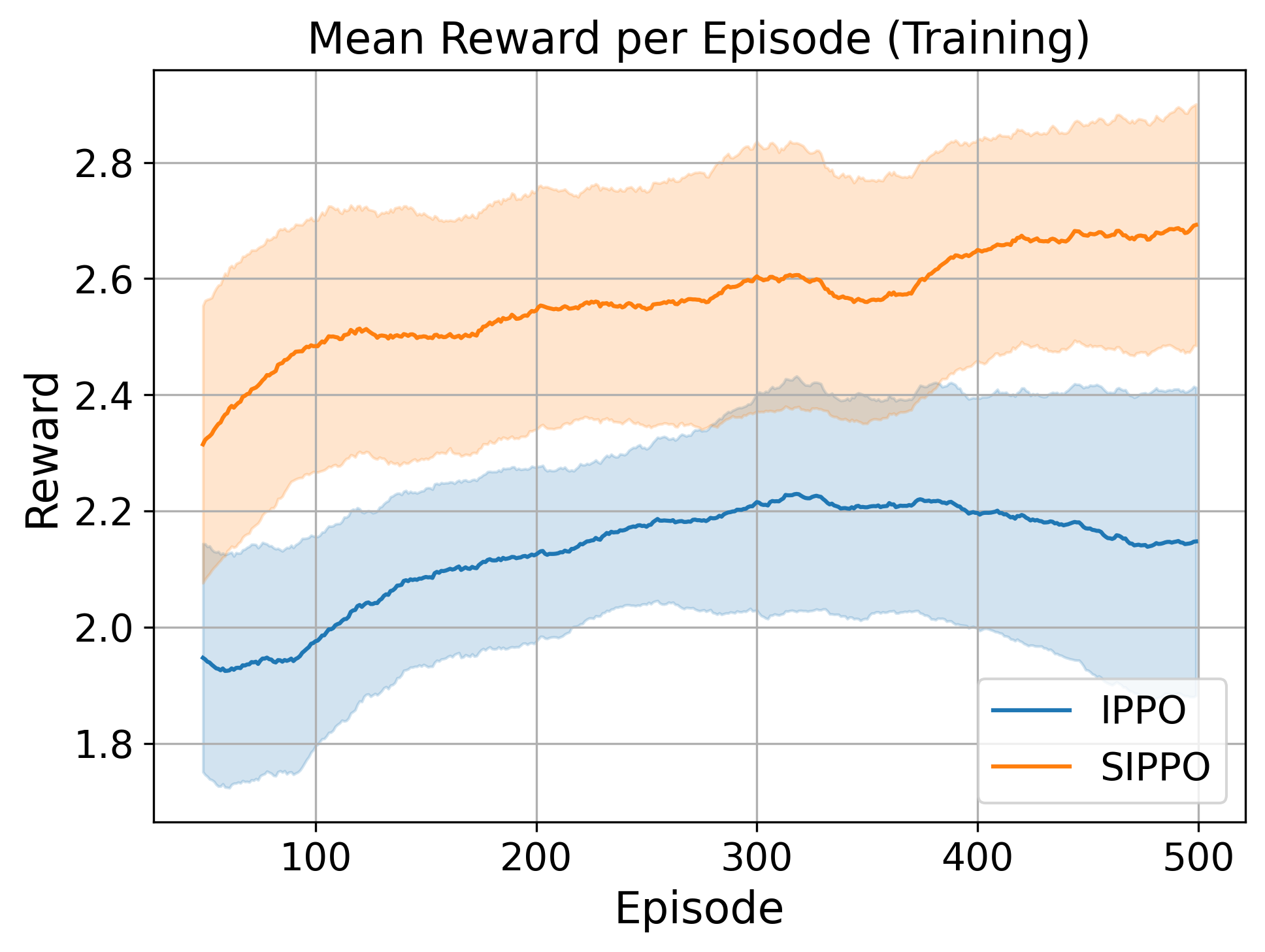}
        \vspace{0.3em}

        \textbf{(d)} Training ($\mu=1.0$)
        \label{fig:epgg_results:sub4}
    \end{minipage}
    \hfill
    \begin{minipage}[b]{0.3\textwidth}
        \centering
        \includegraphics[width=0.9\textwidth]{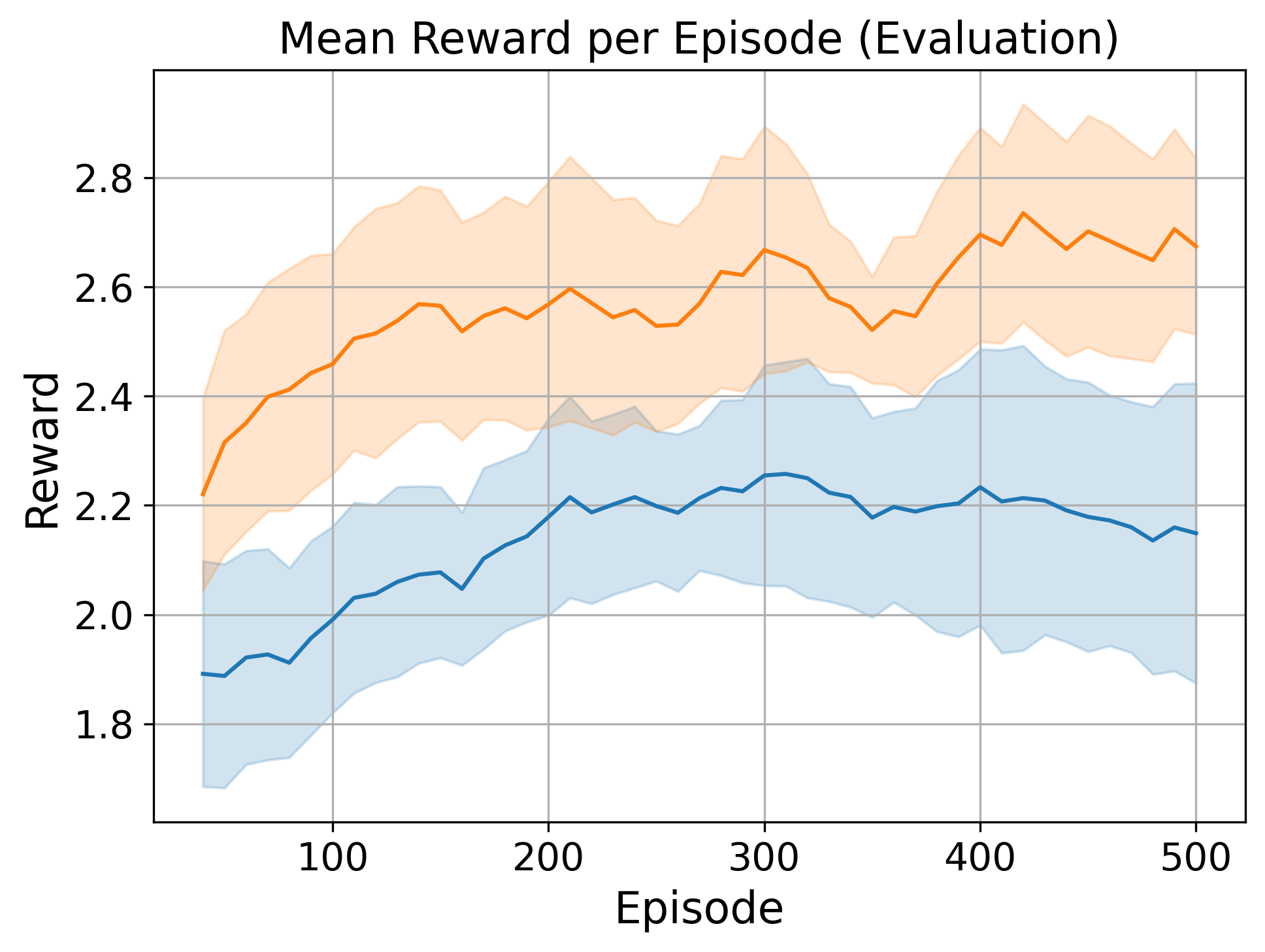}
        \vspace{0.3em}

        \textbf{(e)} Evaluation ($\mu=1.0$)
        \label{fig:epgg_results:sub5}
    \end{minipage}
    \hfill
    \begin{minipage}[b]{0.3\textwidth}
        \centering
        \includegraphics[width=0.9\textwidth]{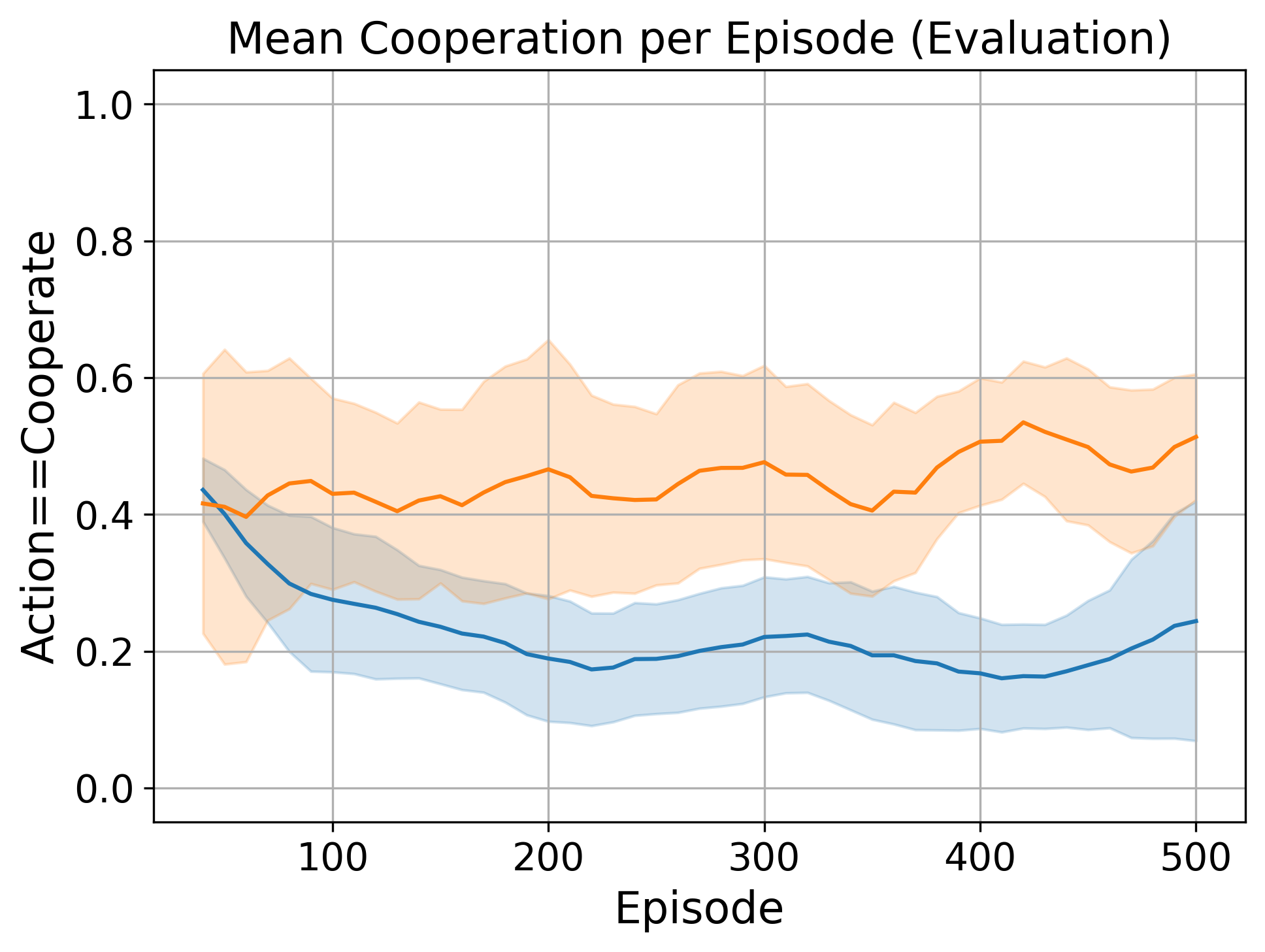}
        \vspace{0.3em}

        \textbf{(f)} Safety ($\mu=1.0$)
        \label{fig:epgg_results:sub6}
    \end{minipage}

    \vspace{5pt}

    % Row 3: mu = 1.5
    \begin{minipage}[b]{0.3\textwidth}
        \centering
        \includegraphics[width=0.9\textwidth]{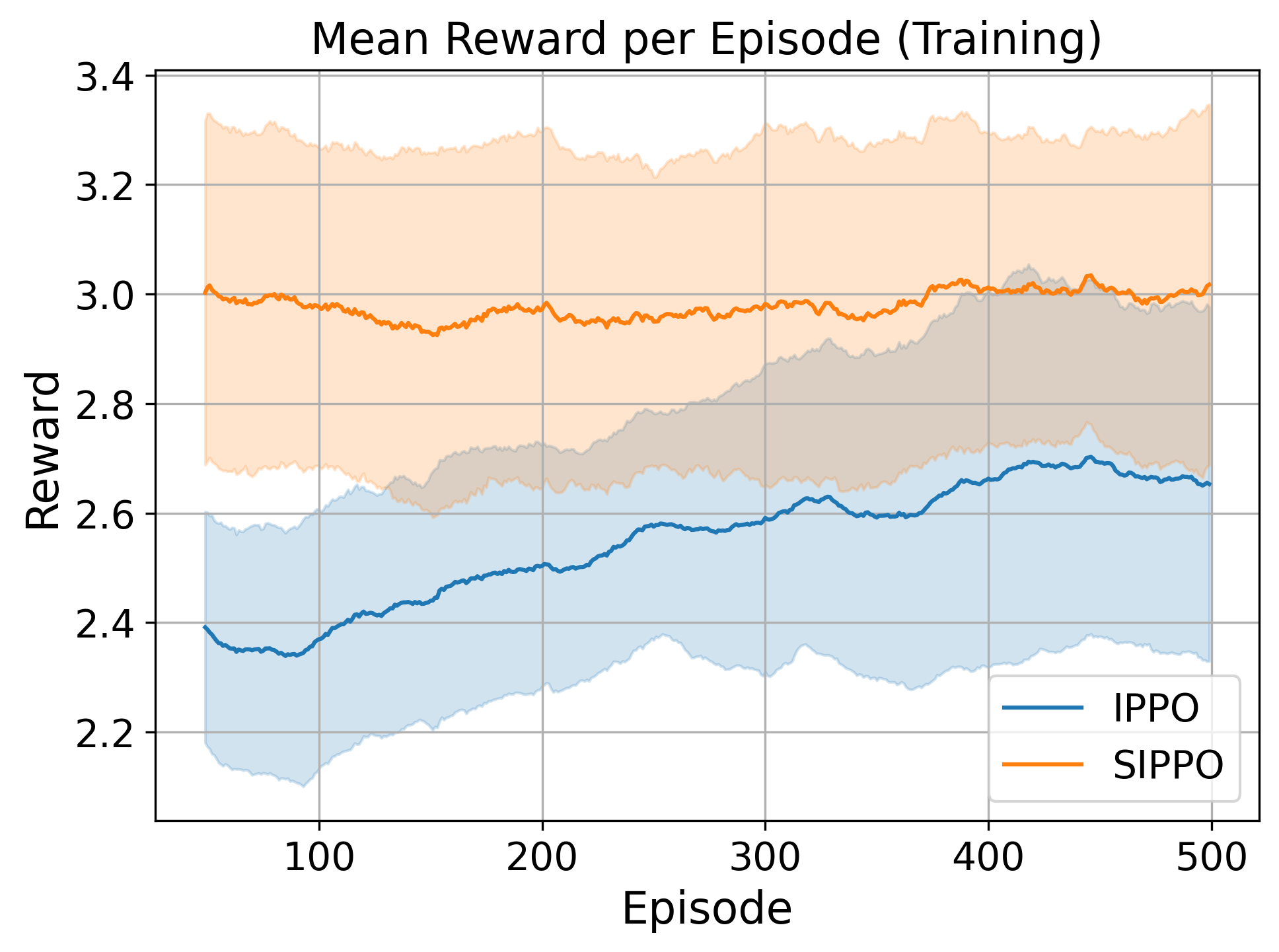}
        \vspace{0.3em}

        \textbf{(g)} Training ($\mu=1.5$)
        \label{fig:epgg_results:sub7}
    \end{minipage}
    \hfill
    \begin{minipage}[b]{0.3\textwidth}
        \centering
        \includegraphics[width=0.9\textwidth]{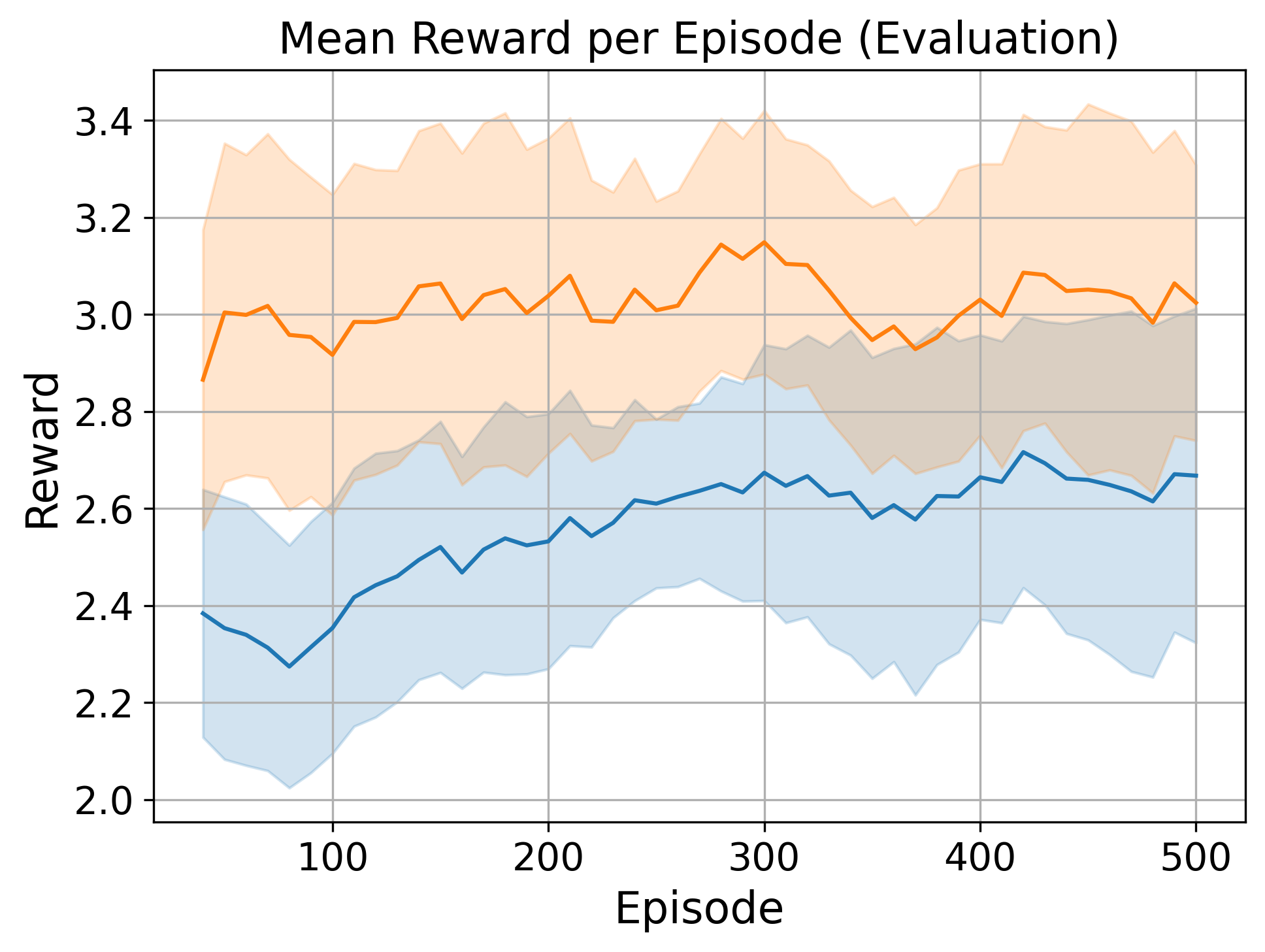}
        \vspace{0.3em}

        \textbf{(h)} Evaluation ($\mu=1.5$)
        \label{fig:epgg_results:sub8}
    \end{minipage}
    \hfill
    \begin{minipage}[b]{0.3\textwidth}
        \centering
        \includegraphics[width=0.9\textwidth]{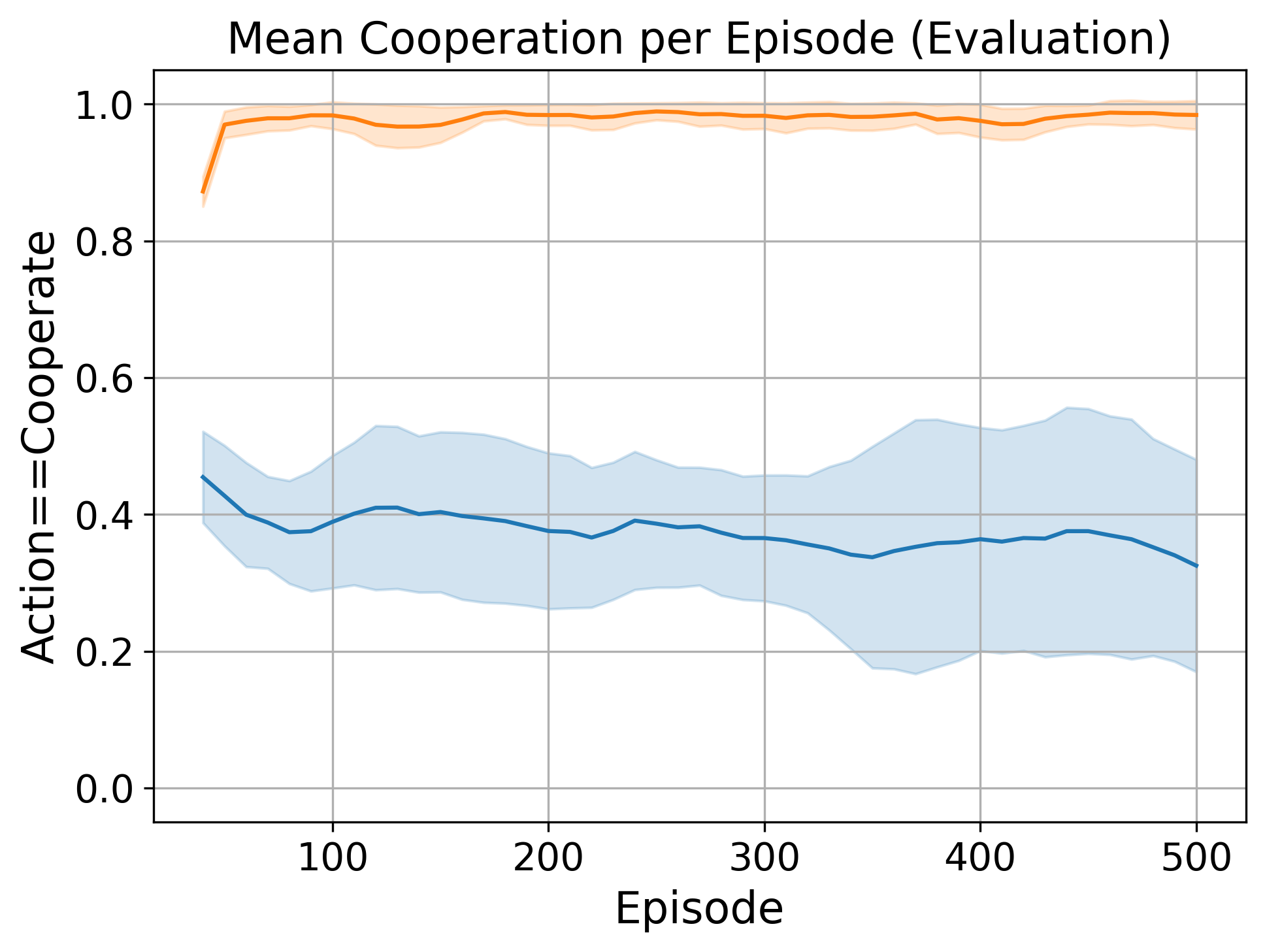}
        \vspace{0.3em}

        \textbf{(i)} Safety ($\mu=1.5$)
        \label{fig:epgg_results:sub9}
    \end{minipage}

    \vspace{5pt}

    % Row 4: mu = 2.5
    \begin{minipage}[b]{0.3\textwidth}
        \centering
        \includegraphics[width=0.9\textwidth]{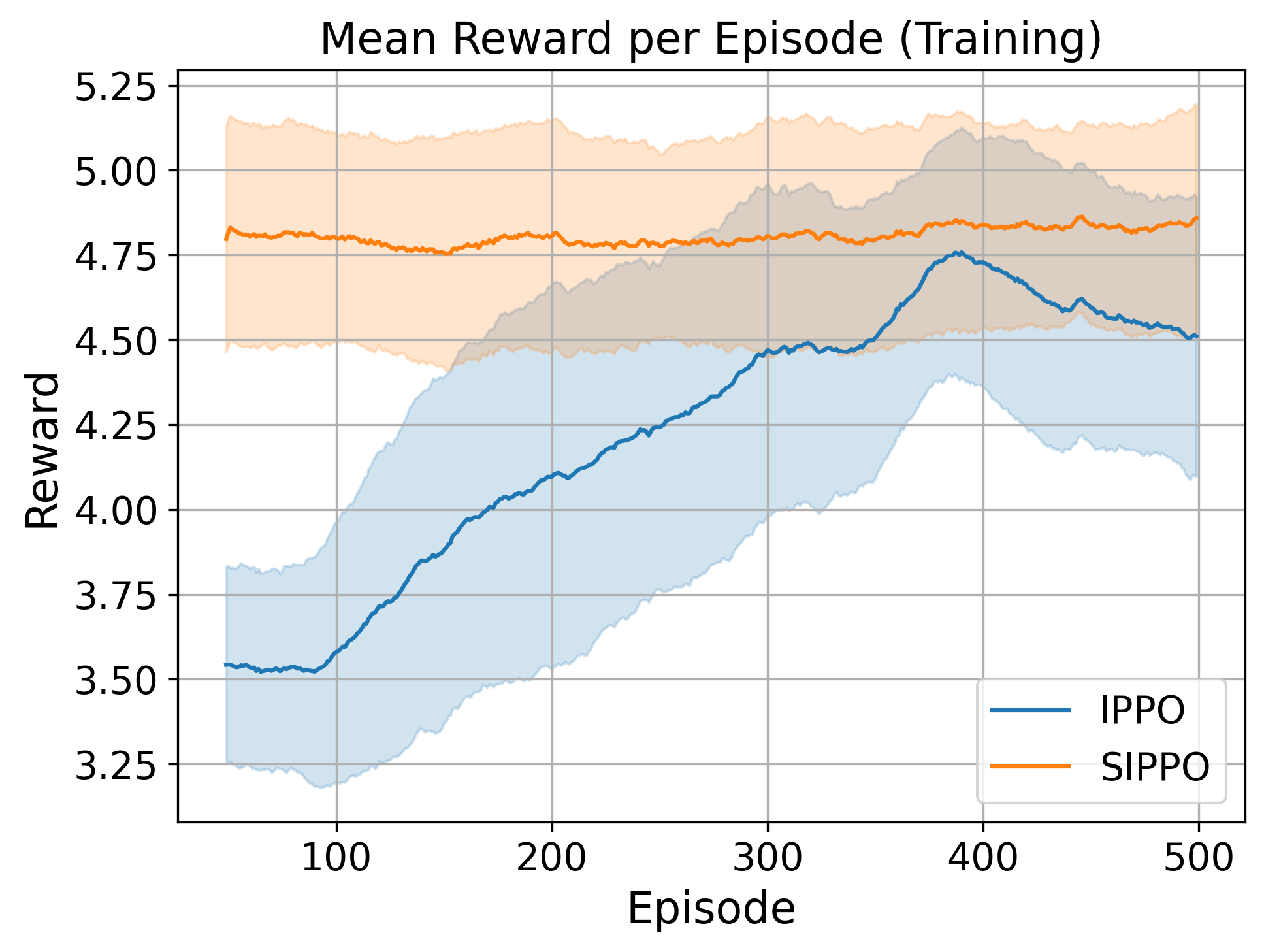}
        \vspace{0.3em}

        \textbf{(j)} Training ($\mu=2.5$)
        \label{fig:epgg_results:sub10}
    \end{minipage}
    \hfill
    \begin{minipage}[b]{0.3\textwidth}
        \centering
        \includegraphics[width=0.9\textwidth]{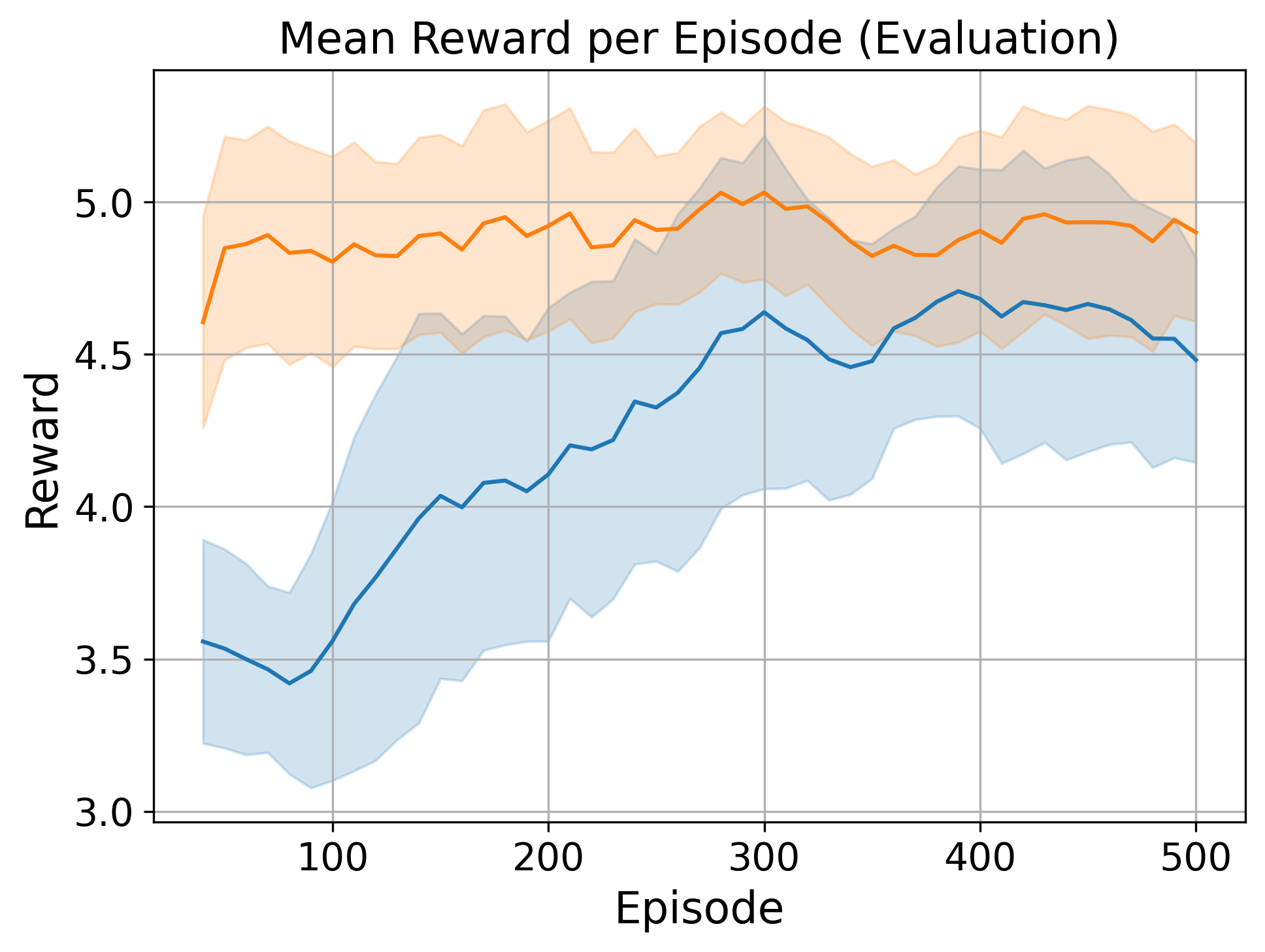}
        \vspace{0.3em}

        \textbf{(k)} Evaluation ($\mu=2.5$)
        \label{fig:epgg_results:sub11}
    \end{minipage}
    \hfill
    \begin{minipage}[b]{0.3\textwidth}
        \centering
        \includegraphics[width=0.9\textwidth]{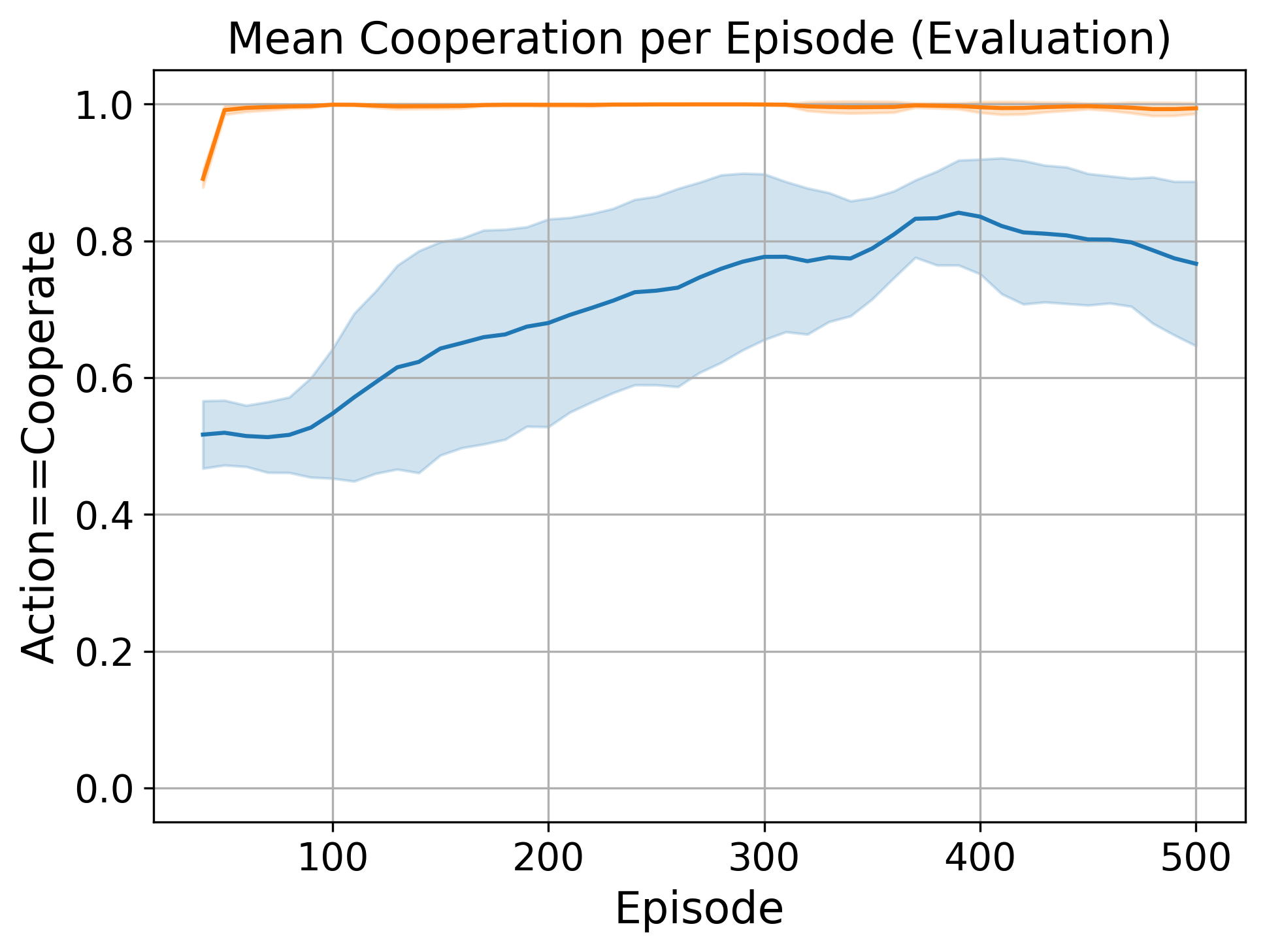}
        \vspace{0.3em}

        \textbf{(l)} Safety ($\mu=2.5$)
        \label{fig:epgg_results:sub12}
    \end{minipage}

    \vspace{5pt}

    % Row 5: mu = 5.0
    \begin{minipage}[b]{0.3\textwidth}
        \centering
        \includegraphics[width=0.9\textwidth]{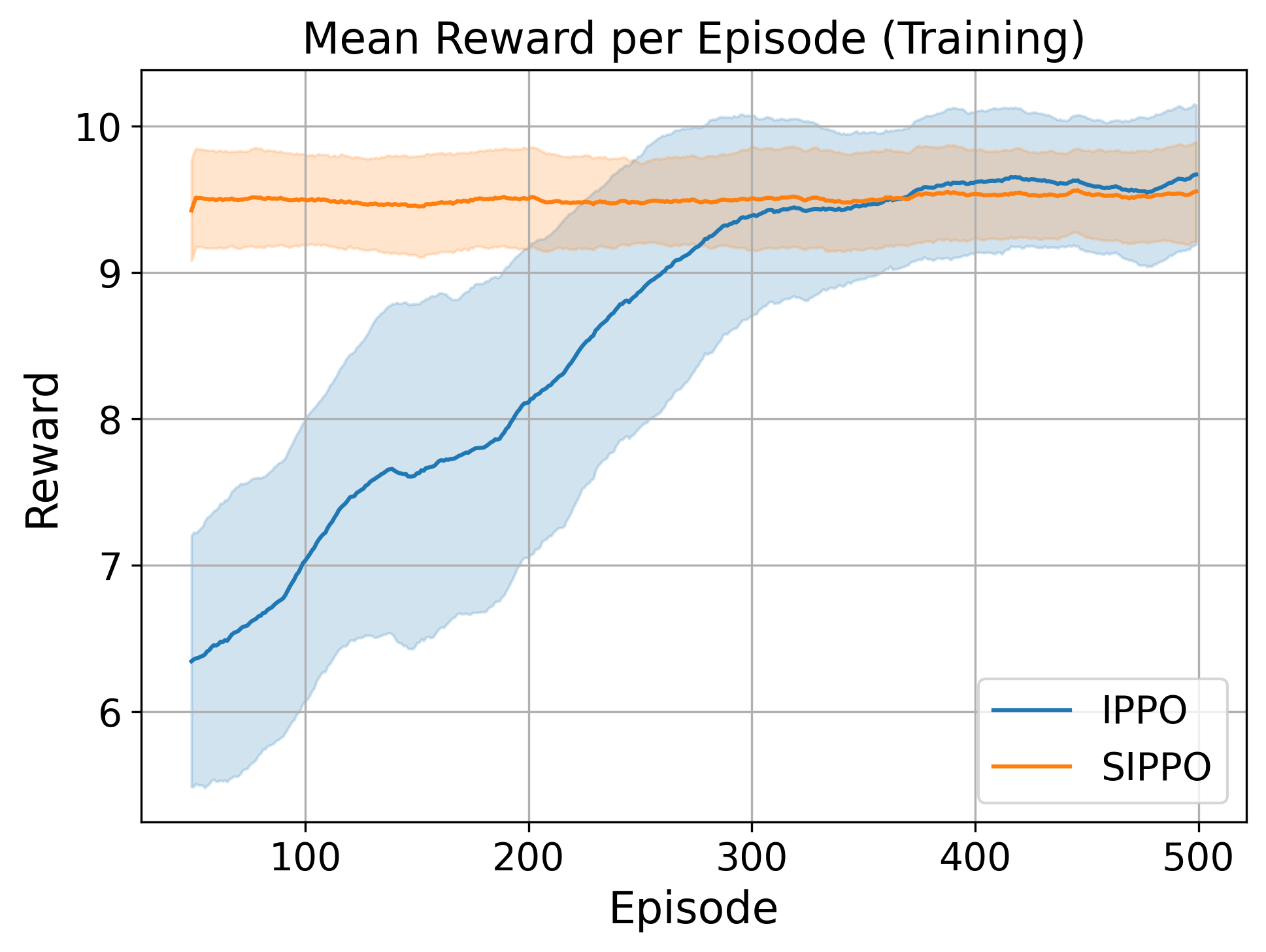}
        \vspace{0.3em}

        \textbf{(m)} Training ($\mu=5.0$)
        \label{fig:epgg_results:sub13}
    \end{minipage}
    \hfill
    \begin{minipage}[b]{0.3\textwidth}
        \centering
        \includegraphics[width=0.9\textwidth]{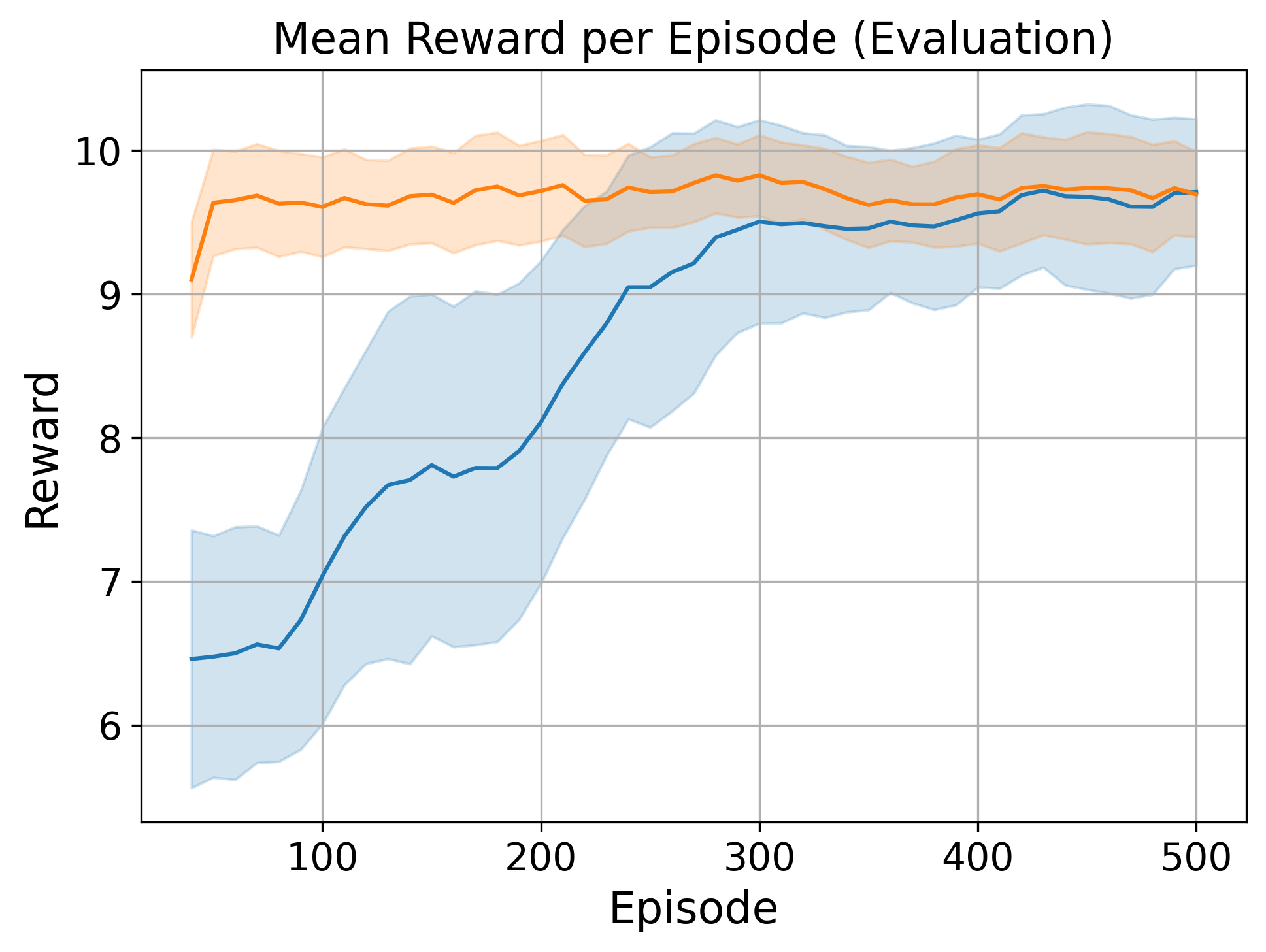}
        \vspace{0.3em}

        \textbf{(n)} Evaluation ($\mu=5.0$)
        \label{fig:epgg_results:sub14}
    \end{minipage}
    \hfill
    \begin{minipage}[b]{0.3\textwidth}
        \centering
        \includegraphics[width=0.9\textwidth]{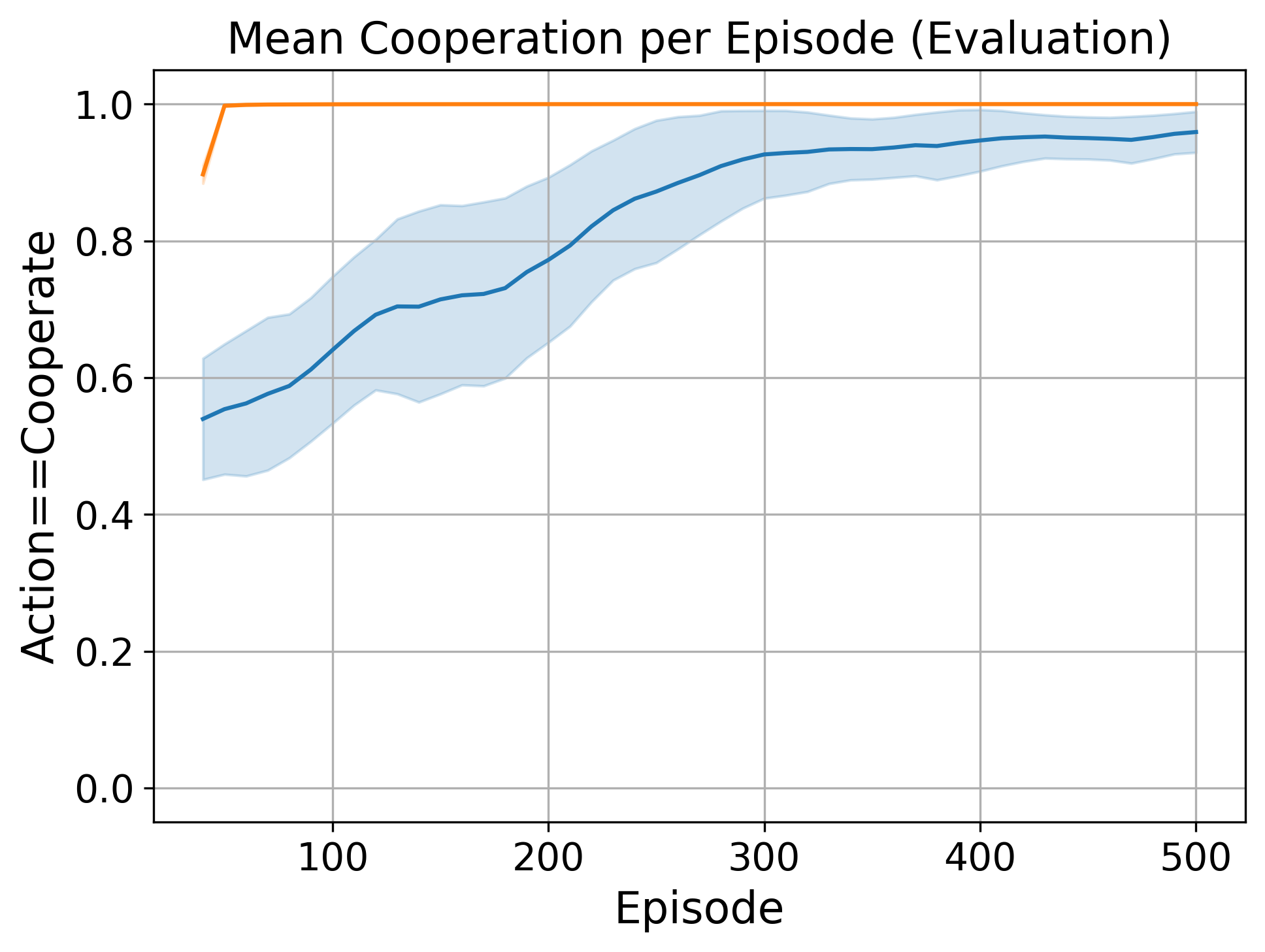}
        \vspace{0.3em}

        \textbf{(o)} Safety ($\mu=5.0$)
        \label{fig:epgg_results:sub15}
    \end{minipage}

    \caption[Results for 2-player \textit{Extended Public Goods Game}.]{Training, evaluation, and safety results for the two-player EPGG for PLPG-based agents. The multiplicative factor $f_t$ is sampled from $\mathcal{N}(\mu,1)$, where $\mu \in \{0.5, 1.0, 1.5, 2.5, 5.0\}$ for different experiments. The lines represent the mean and the shadow the standard deviation over 5 seeds. Results are smoothed using a rolling average with window size 50.}
    \label{fig:epgg_results}
\end{figure*}
\begin{figure*}[h!]
    \centering

    % First row
    \begin{minipage}[b]{0.32\textwidth}
        \centering
        \includegraphics[width=\textwidth]{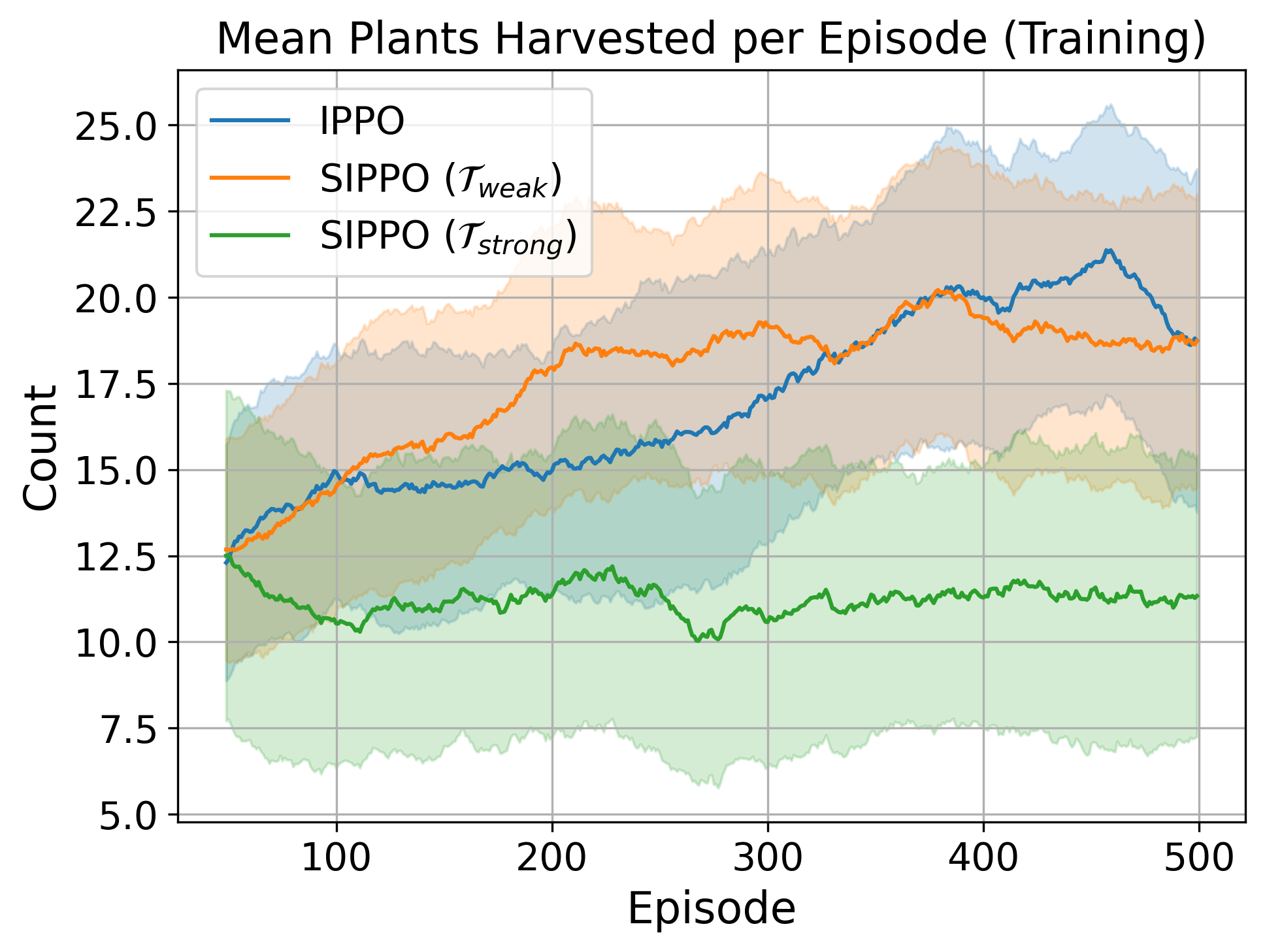}
        \vspace{0.3em}

        \textbf{(a)}
        \label{fig:msh_full_results:subfig1}
    \end{minipage}
    \hfill
    \begin{minipage}[b]{0.32\textwidth}
        \centering
        \includegraphics[width=\textwidth]{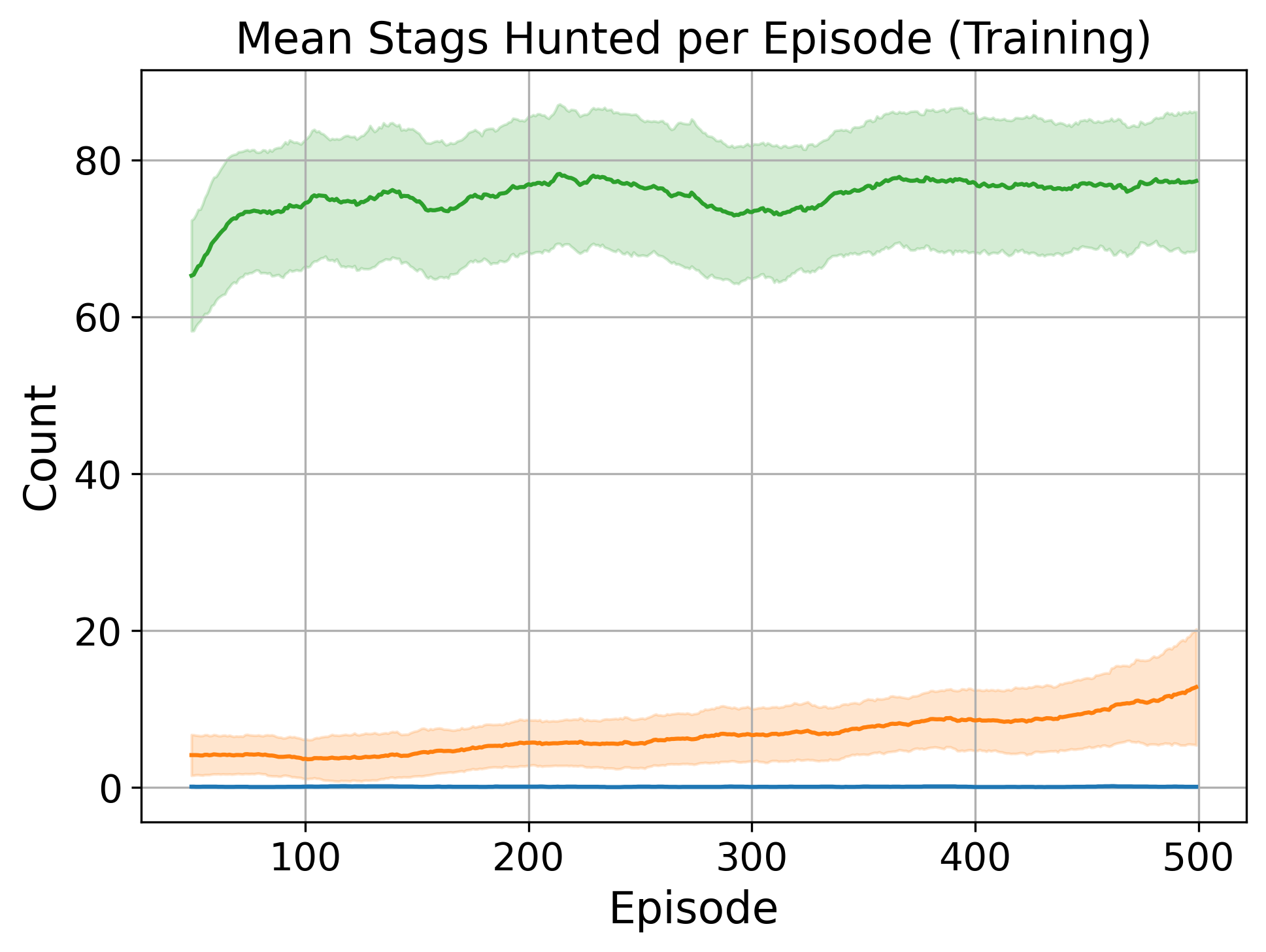}
        \vspace{0.3em}

        \textbf{(b)}
        \label{fig:msh_full_results:subfig2}
    \end{minipage}
    \hfill
    \begin{minipage}[b]{0.32\textwidth}
        \centering
        \includegraphics[width=\textwidth]{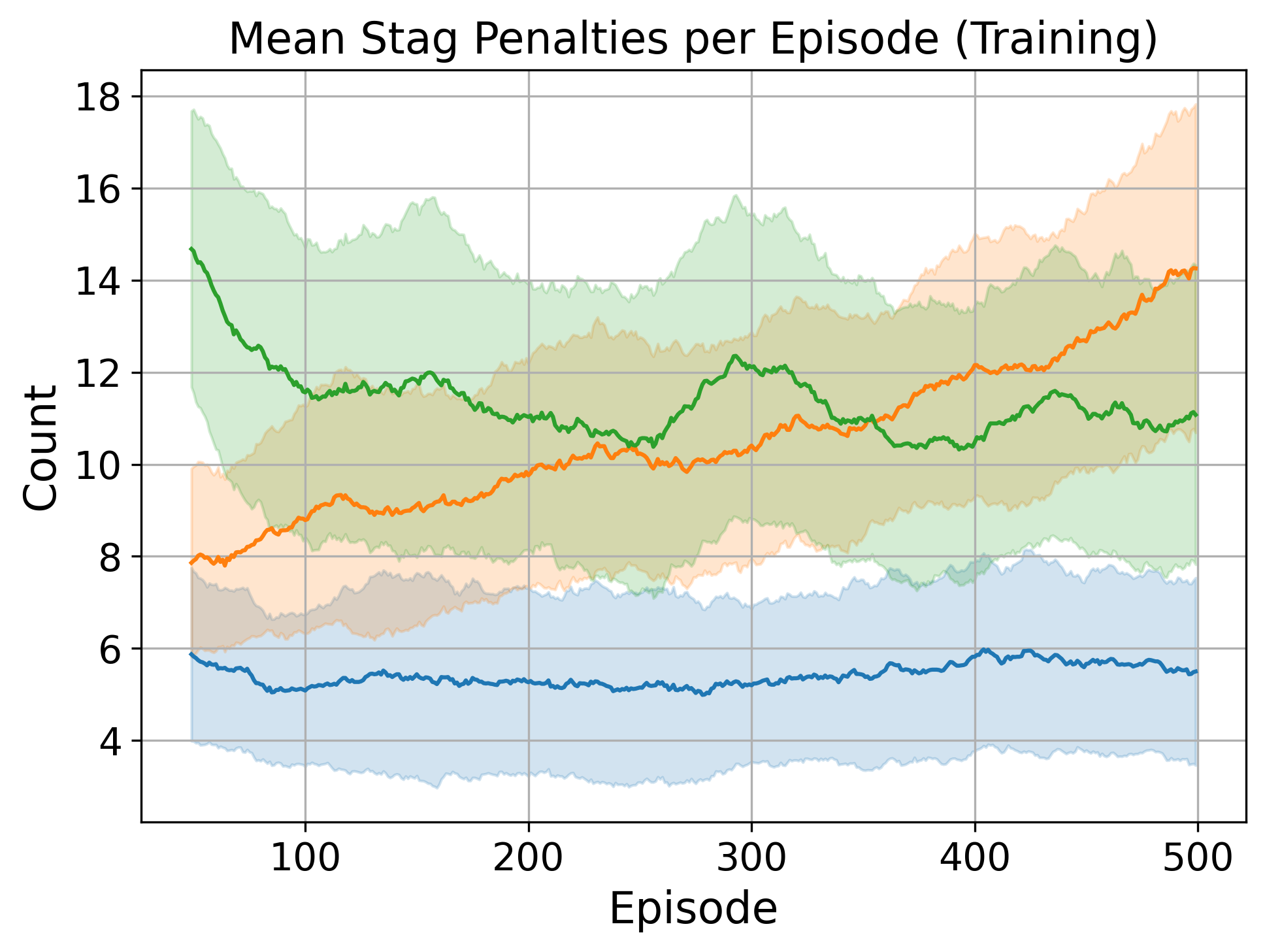}
        \vspace{0.3em}

        \textbf{(c)}
        \label{fig:msh_full_results:subfig3}
    \end{minipage}

    \vspace{10pt}

    % Second row
    \begin{minipage}[b]{0.32\textwidth}
        \centering
        \includegraphics[width=\textwidth]{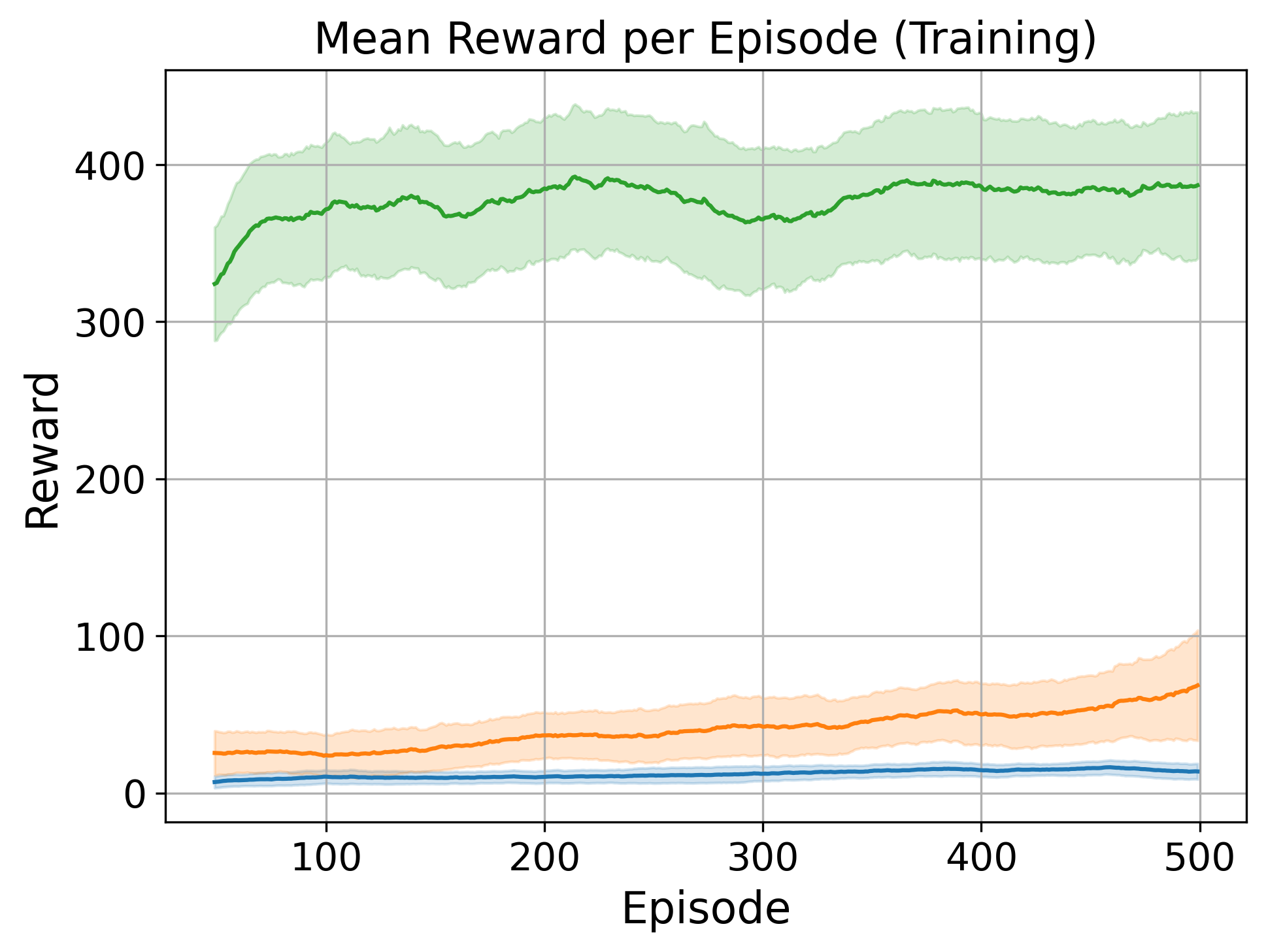}
        \vspace{0.3em}

        \textbf{(d)}
        \label{fig:msh_full_results:subfig4}
    \end{minipage}
    \hfill
    \begin{minipage}[b]{0.32\textwidth}
        \centering
        \includegraphics[width=\textwidth]{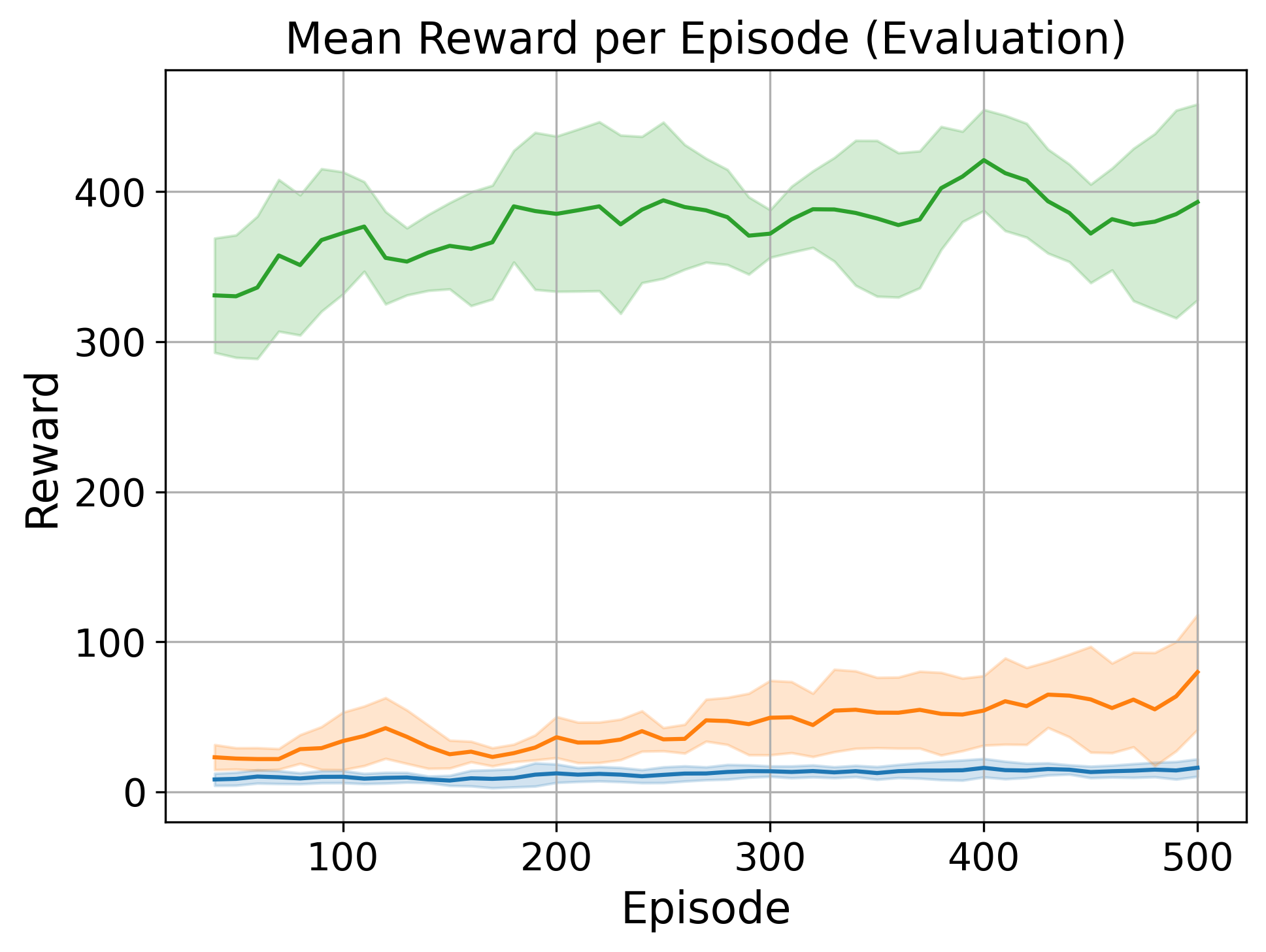}
        \vspace{0.3em}

        \textbf{(e)}
        \label{fig:msh_full_results:subfig5}
    \end{minipage}
    \hfill
    \begin{minipage}[b]{0.32\textwidth}
        \centering
        \includegraphics[width=\textwidth]{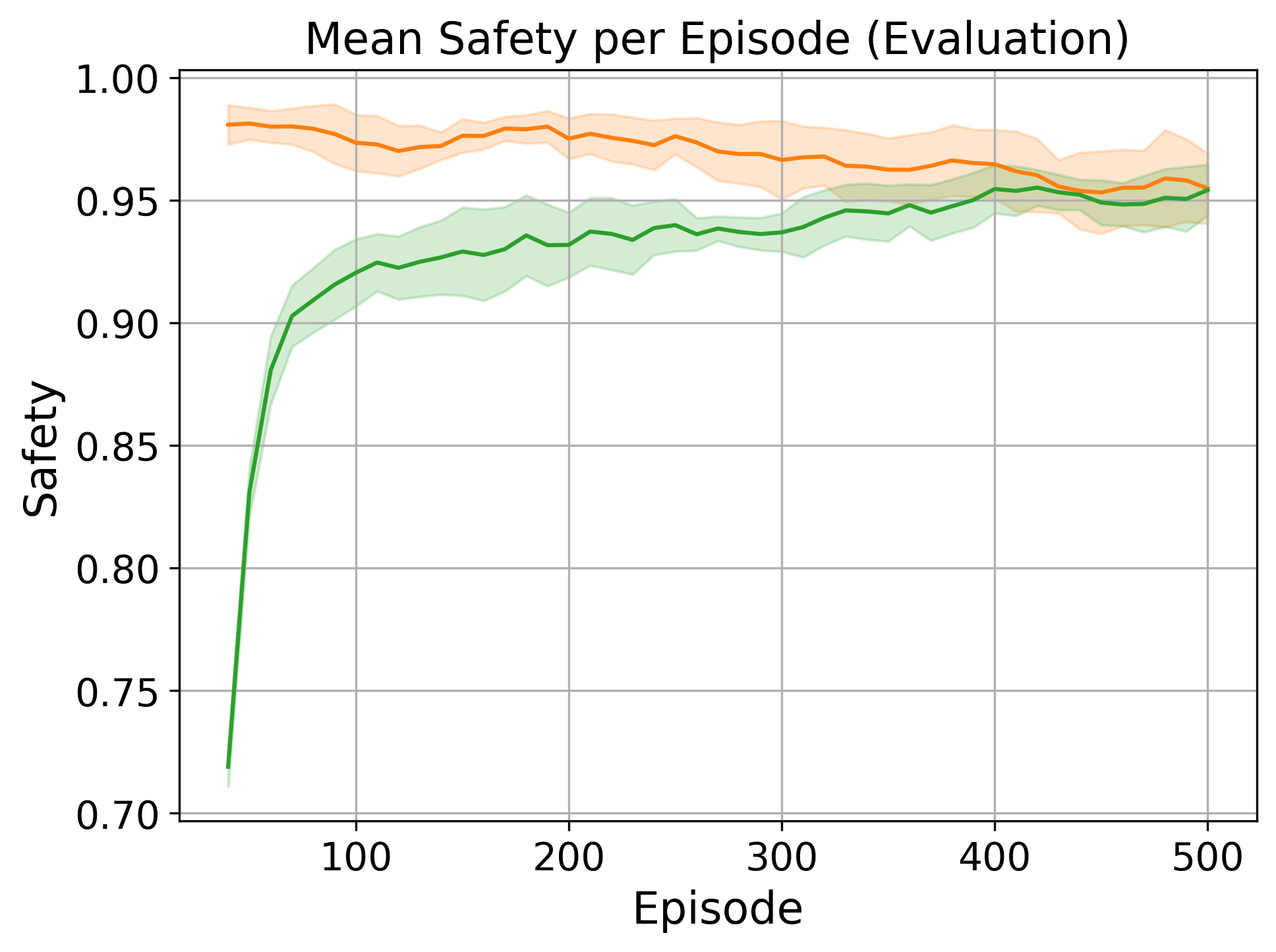}
        \vspace{0.3em}

        \textbf{(f)}
        \label{fig:msh_full_results:subfig6}
    \end{minipage}

    \caption[Results for \textit{Markov Stag-Hunt}.]{Behavioural and training results during training for \textit{Markov Stag-Hunt} for PLPG-based agents with no shield or two different shields, $\mathcal{T}_{weak}$ or $\mathcal{T}_{strong}$. The lines represent the mean and the shadow the standard deviation over 3 seeds. Results are smoothed using a rolling average with window size 50.}
    \label{fig:msh_full_results}
\end{figure*}
\begin{figure*}[h!]
    \centering

    % First row
    \begin{minipage}[b]{0.32\textwidth}
        \centering
        \includegraphics[width=\textwidth]{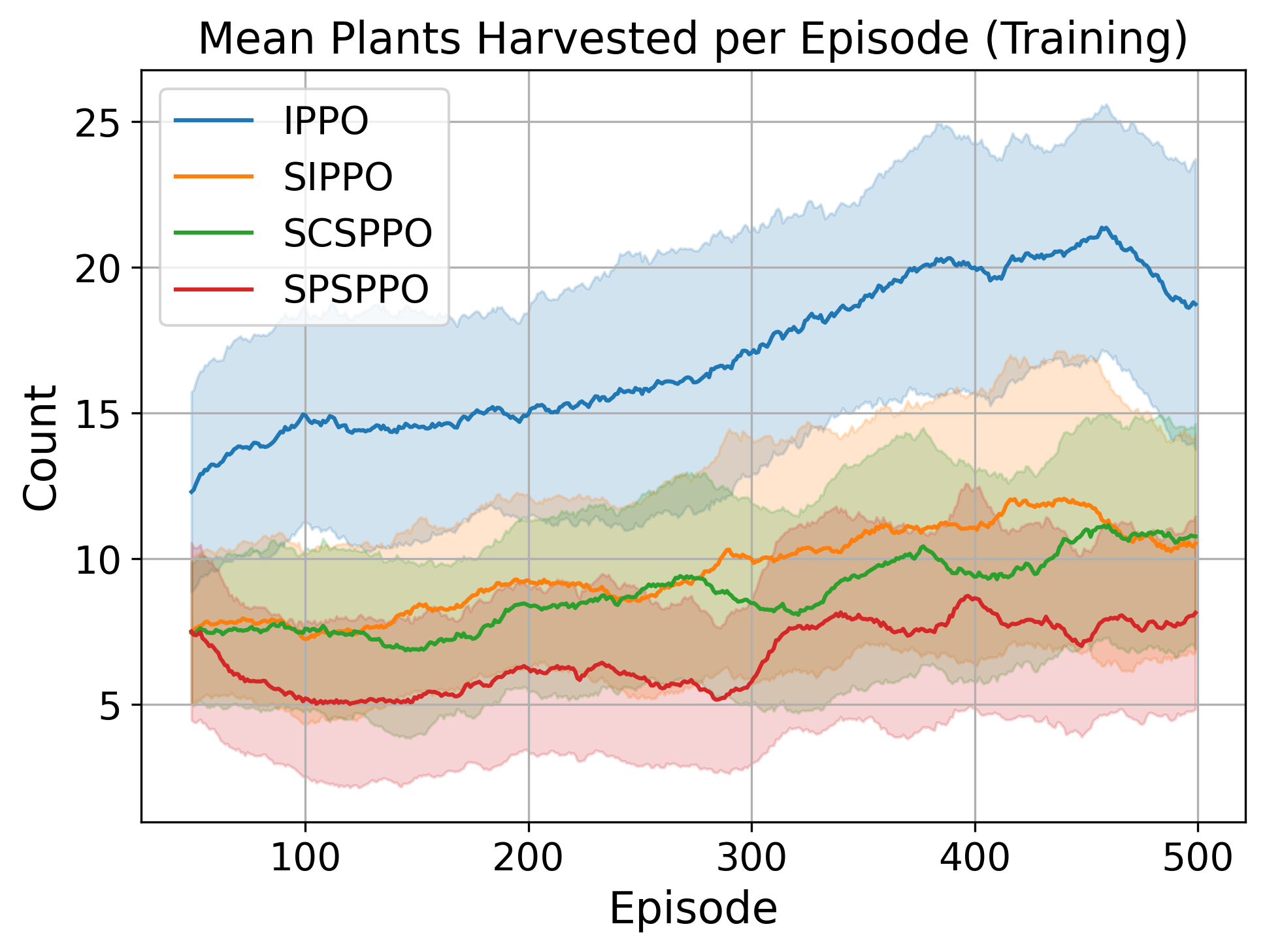}
        \vspace{0.3em}

        \textbf{(a)}
        \label{fig:msh_part_results:subfig1}
    \end{minipage}
    \hfill
    \begin{minipage}[b]{0.32\textwidth}
        \centering
        \includegraphics[width=\textwidth]{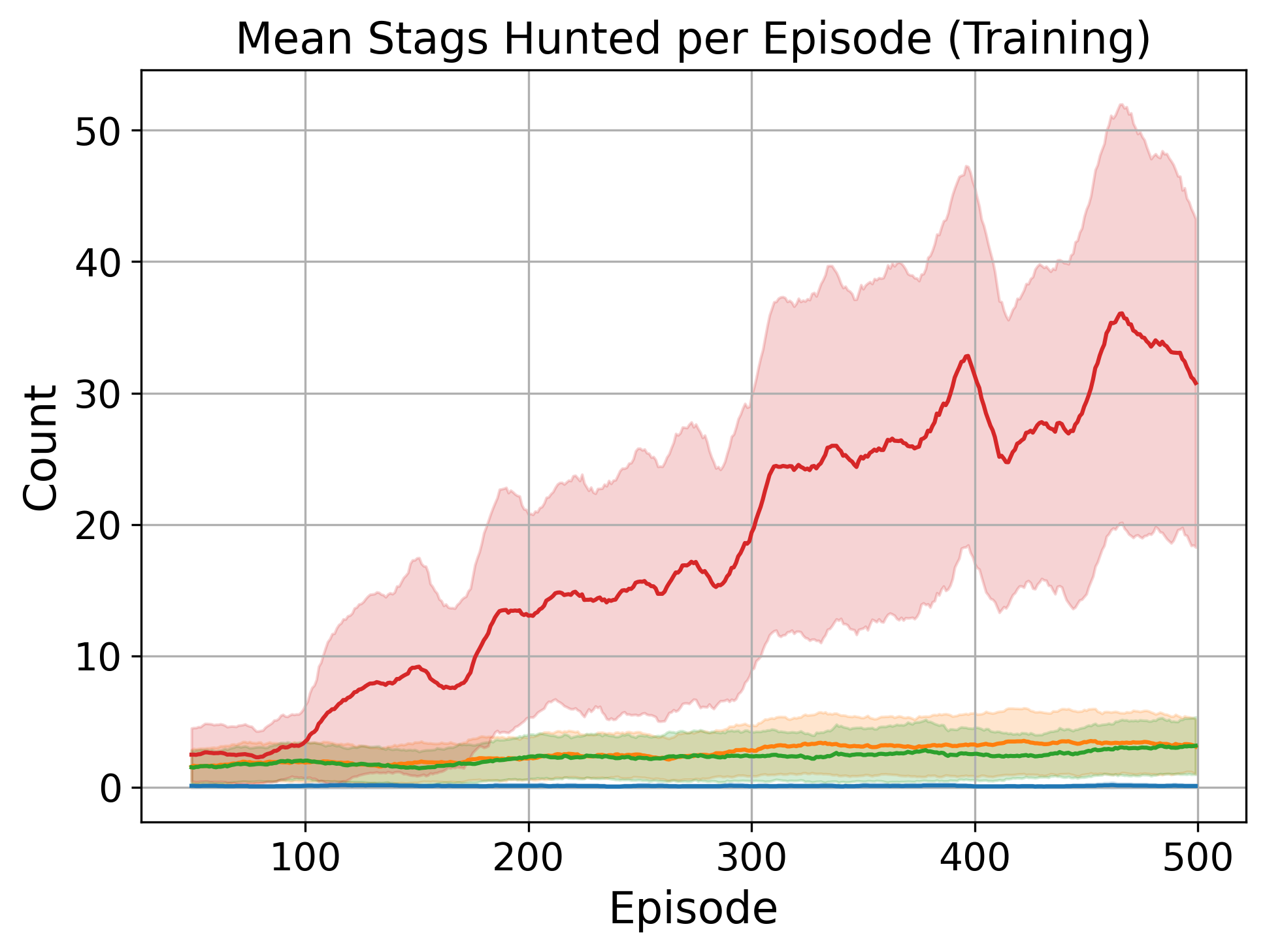}
        \vspace{0.3em}

        \textbf{(b)}
        \label{fig:msh_part_results:subfig2}
    \end{minipage}
    \hfill
    \begin{minipage}[b]{0.32\textwidth}
        \centering
        \includegraphics[width=\textwidth]{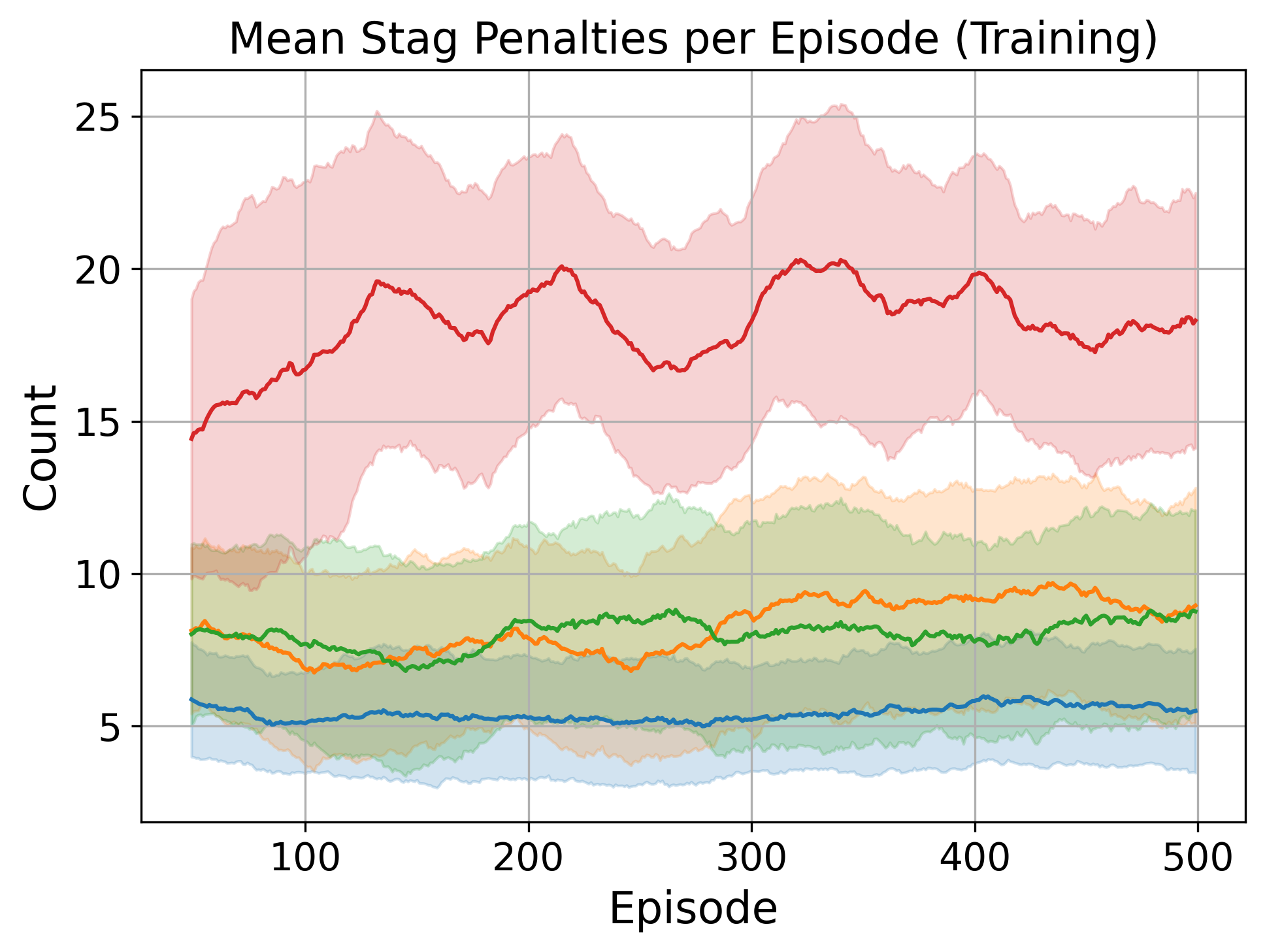}
        \vspace{0.3em}

        \textbf{(c)}
        \label{fig:msh_part_results:subfig3}
    \end{minipage}

    \vspace{10pt}

    % Second row
    \begin{minipage}[b]{0.32\textwidth}
        \centering
        \includegraphics[width=\textwidth]{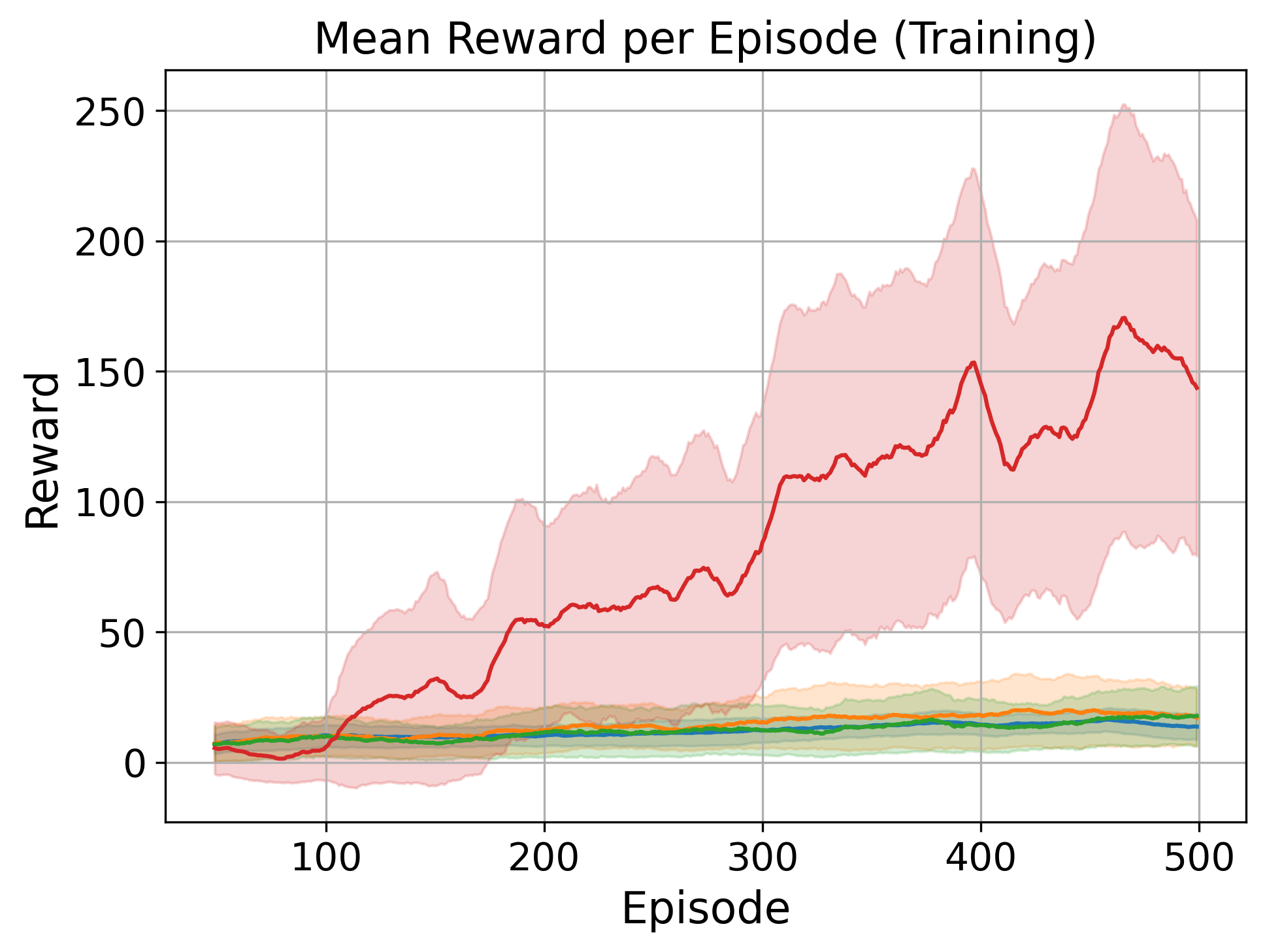}
        \vspace{0.3em}

        \textbf{(d)}
        \label{fig:msh_part_results:subfig4}
    \end{minipage}
    \hfill
    \begin{minipage}[b]{0.32\textwidth}
        \centering
        \includegraphics[width=\textwidth]{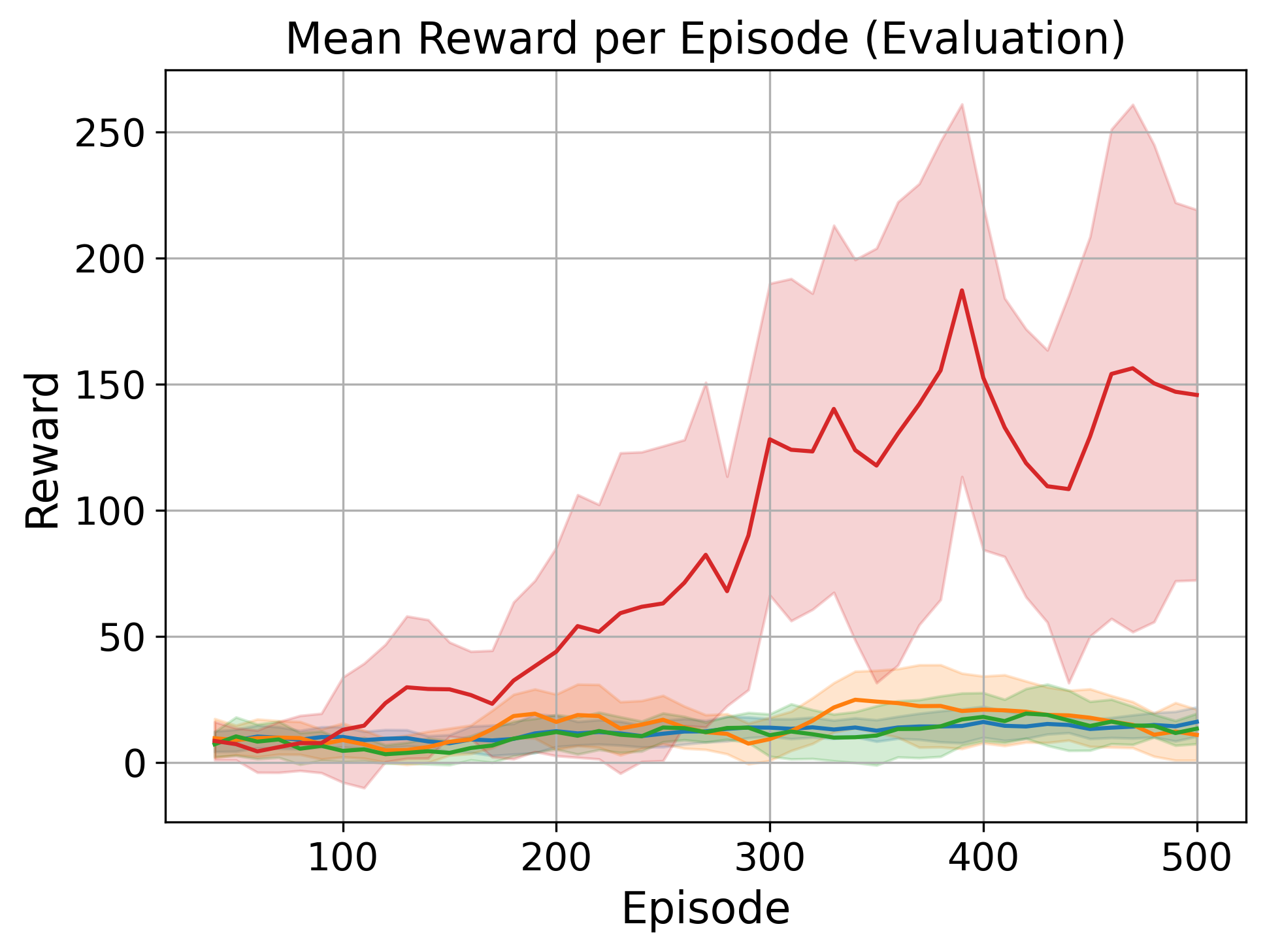}
        \vspace{0.3em}

        \textbf{(e)}
        \label{fig:msh_part_results:subfig5}
    \end{minipage}
    \hfill
    \begin{minipage}[b]{0.32\textwidth}
        \centering
        \includegraphics[width=\textwidth]{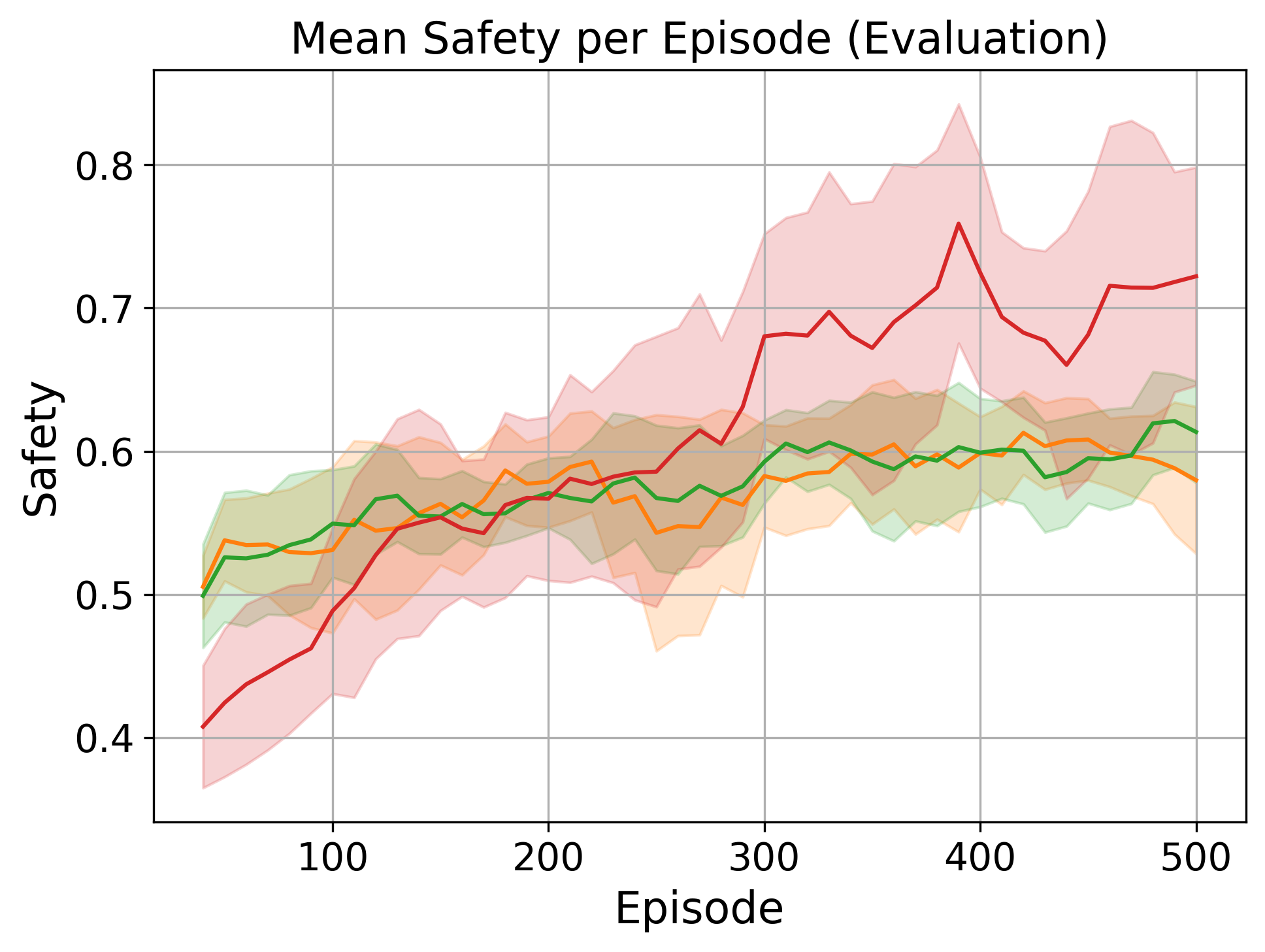}
        \vspace{0.3em}

        \textbf{(f)}
        \label{fig:msh_part_results:subfig6}
    \end{minipage}

    \caption[Results for Partially Shielded \textit{Markov Stag-Hunt}.]{Behavioural and training results for \textit{Markov Stag-Hunt} for PLPG-based agents with or without parameter sharing during training, where one of the agents is shielded with $\mathcal{T}_{strong}$. The solid lines represent the mean and the shadow the standard deviation over 3 seeds. Results are smoothed using a rolling average with window size 50.}
    \label{fig:msh_part_results}
\end{figure*}

\end{document}